\documentclass{article}
\newcommand{\MemInit}{\mathrm{MemInit}}
\PassOptionsToPackage{numbers}{natbib}
\usepackage[preprint]{neurips_2026}

\usepackage[utf8]{inputenc} 
\usepackage[T1]{fontenc}    
\usepackage{hyperref}       
\usepackage{url}            
\usepackage{booktabs}       
\usepackage{amsfonts}       
\usepackage{nicefrac}       
\usepackage{microtype}      
\usepackage{xcolor}         
\usepackage[utf8]{inputenc} 
\usepackage[T1]{fontenc}    
\usepackage{hyperref}       
\usepackage{url}            
\usepackage{booktabs}       
\usepackage{multirow}       
\usepackage{enumitem}       
\usepackage{amsfonts}       
\usepackage{mathtools}      
\usepackage{nicefrac}       
\usepackage{microtype}      
\usepackage{amsthm}         
\usepackage{xcolor}         
\usepackage{amssymb}        
\usepackage{algorithm}
\usepackage{algorithmicx}
\usepackage{algpseudocode}
\algrenewcommand\algorithmicrequire{\textbf{Input:}}
\algrenewcommand\algorithmicensure{\textbf{Output:}}
\usepackage{amsmath}

\usepackage{amsthm}

\DeclareMathOperator{\degen}{degen}
\theoremstyle{plain}
\newtheorem{theorem}{Theorem}
\newtheorem{proposition}{Proposition}
\newtheorem*{proposition*}{Proposition}
\newtheorem{lemma}{Lemma}
\newtheorem{corollary}{Corollary}
\theoremstyle{definition}

\theoremstyle{remark}
\newtheorem{remark}{Remark}
\newcommand{\cA}{\mathcal{A}}
\newcommand{\cX}{\mathcal{X}}
\newcommand{\cD}{\mathcal{D}}

\newcommand{\cN}{\mathcal{N}}
\newcommand{\cQ}{\mathcal{Q}}

\newcommand{\eps}{\epsilon}

\newcommand{\EE}{\mathbb{E}}

\newcommand{\cC}{\mathcal{C}}

\newcommand{\Reg}{\mathrm{Reg}}
\newcommand{\UCB}{\mathrm{UCB}}
\newcommand{\LCB}{\mathrm{LCB}}
\newcommand{\Score}{\mathrm{Score}}
\newcommand{\Compress}{\mathrm{Compress}}
\newcommand{\MemRead}{\mathrm{read}}

\title{Remember the Decision, Not the Description: A Rate-Distortion Framework for Agent Memory}

\author{%
\textbf{Mingxi Zou}\textsuperscript{1} \quad
\textbf{Zhihan Guo}\textsuperscript{2} \quad
\textbf{Langzhang Liang}\textsuperscript{1} \quad
\textbf{Zhuo Wang}\textsuperscript{1} \quad
\textbf{Qifan Wang}\textsuperscript{3} \\
\textbf{Qingsong Wen}\textsuperscript{4} \quad
\textbf{Irwin King}\textsuperscript{2} \quad
\textbf{Lizhen Qu}\textsuperscript{5,*} \quad
\textbf{Zenglin Xu}\textsuperscript{6,1,*} \\[0.5em]
\textsuperscript{1}Fudan University \\
\textsuperscript{2}The Chinese University of Hong Kong \\
\textsuperscript{3}Meta AI \\
\textsuperscript{4}AI Research Institute, Squirrel Ai Learning \\
\textsuperscript{5}Monash University \\
\textsuperscript{6}Shanghai Academy of AI for Science \\
\textsuperscript{*}Corresponding authors
}

\begin{document}

\maketitle

\begin{abstract}
Long-horizon language agents must operate under limited runtime memory, yet existing memory mechanisms often organize experience around descriptive criteria—relevance, salience, or summary quality. 
For an agent, however, memory is valuable not because it faithfully describes
the past, but because it preserves the distinctions between histories that
must remain separated under a fixed budget to support good decisions.
We cast this as a
decision-centric rate-distortion problem, measuring memory quality by the loss in
achievable decision quality induced by compression. 
This yields an exact forgetting
boundary for what can be safely forgotten, and a memory-distortion frontier characterizing the optimal tradeoff between memory budget and decision quality. 
Motivated by this decision-centric view of memory, we propose \textbf{DeMem}, an online memory learner that refines
its partition only when data certify that a shared state would induce decision
conflict, and prove near-minimax regret guarantees. 
On both controlled synthetic diagnostics and long-horizon conversational
benchmarks, DeMem yields consistent gains under the same runtime budget,
supporting the principle that memory should preserve the distinctions that
matter for decisions, not descriptions.
\end{abstract}
\vspace{-0.5em}

\section{Introduction}
\vspace{-0.5em}
Large language model (LLM) agents increasingly operate over long interaction
horizons. They must answer using earlier dialogue, plan from past observations,
and adapt behavior using feedback accumulated across trials and sessions
\citep{yao2023reactsynergizingreasoningacting,shinn2023reflexionlanguageagentsverbal}.
Yet long-horizon interaction remains bottlenecked by memory. Increasing the
context window alone does not guarantee reliable use of distant evidence
\citep{liu2023lostmiddlelanguagemodels}, and recent benchmarks continue to
expose failures in long-term conversational memory, incremental interaction,
test-time learning, and constraint-consistent recall
\citep{maharana-etal-2024-evaluating,wu2025longmemevalbenchmarkingchatassistants,
hu2026evaluating,li2026locomoplusbeyondfactualcognitivememory}.

These limitations have motivated a growing class of external and persistent
memory mechanisms for language models and agents. Retrieval-augmented methods
make past evidence available at generation time
\citep{lewis2021retrievalaugmentedgenerationknowledgeintensivenlp}, while
OS-style and explicit long-term memory systems maintain state beyond the
current context window
\citep{packer2024memgptllmsoperatingsystems,NEURIPS2023_ebd82705}.
More recent systems organize memory through graph-structured, self-organizing,
or production-oriented memory layers
\citep{gutiérrez2025hipporagneurobiologicallyinspiredlongterm,
chhikara2025mem0buildingproductionreadyai,
nan2025nemoriselforganizingagentmemory,latimer2025hindsight2020buildingagent}.

Existing agent memory systems typically decide what to store, retrieve, or
compress using \emph{descriptive} properties of experience. Retrieval-based
methods prioritize semantic relevance to the current query
\citep{lewis2021retrievalaugmentedgenerationknowledgeintensivenlp}, while
long-term and OS-style memory systems often rely on salience, recency, or
summary fidelity to maintain usable persistent state
\citep{packer2024memgptllmsoperatingsystems,NEURIPS2023_ebd82705}.
Graph-structured and self-organizing memory systems further organize experience
through entity links, temporal structure, or other relational cues
\citep{gutiérrez2025hipporagneurobiologicallyinspiredlongterm,
chhikara2025mem0buildingproductionreadyai,
nan2025nemoriselforganizingagentmemory}.
Such criteria are useful for organizing past experience, but they do not directly capture the downstream purpose of memory: supporting decisions under a strict runtime budget. 
This concern is not hypothetical. On the LoCoMo benchmark, we find that descriptive similarity is a
\emph{statistically significant but weak predictor of evidence compatibility}:
the Spearman correlation is only $\rho = 0.103$ (AUC $= 0.548$). Under a matched answer-time budget,
description-based retrieval recovers only $66\%$ of gold evidence,
compared to $83\%$ for DeMem(ours).
Among the queries where description-based memory fails, $85\%$ of
errors trace directly to evidence miss or dilution caused by this
mismatch (Appendix~\ref{app:mismatch-prevalence},
Tables~\ref{tab:mismatch-predictive}--\ref{tab:error-attribution}).
Thus, what must be remembered should be defined by its effect on
decision quality, not by its descriptive fidelity.

In this paper, we study budgeted agent memory through the lens of
\emph{decision-centric compression}. This view is consistent with recent calls
for decision-centric design in LLM systems \citep{Sun2026DecisionCentricDF},
but specializes that perspective to the memory bottleneck: at decision time,
the agent cannot rely on the full interaction history, and must instead access
the past through one of \(K\) internal memory states---a fixed number of memory slots available to the decision policy.  The budget is therefore placed on runtime decision states, rather than on offline storage or preprocessing.
Motivated by this view, we propose \textsc{DeMem}, a budgeted memory method with
\(K\) runtime slots. DeMem routes each history--query pair to a slot and refines
the slots only when shared memory would create a decision conflict. In this way,
it remembers distinctions that matter for action rather than description.

Our contributions answer three concrete questions raised by this view:
\begin{itemize}[
    leftmargin=*,
    topsep=0.25em,
    itemsep=0.25em,
    parsep=0pt,
    partopsep=0pt,
    after=\vspace{-0.3em}
]
    \item \textbf{When can histories be safely merged? --- A decision-centric theory of forgetting.}
    We formalize when histories can share memory without hurting downstream decisions, yielding an exact forgetting boundary separating
    safe merges from decision-conflicting ones, and a memory--distortion frontier
    describing the best achievable decision quality under a fixed memory budget.

    \item \textbf{How can decision-aware memory be learned online? --- A certified online memory learner for decision-aware refinement.}
    We introduce DeMem, an online memory learner that maintains \(K\) runtime memory states, refines them only upon certified decision conflict, and comes with near-minimax regret guarantees.

    \item \textbf{Does this principle improve practical agent memory? --- A practical \(K\)-slot agent memory system.}
    We instantiate DeMem with slot routing and certified refinement (the illustrative DeMem workflow in Figure~\ref{fig:worked-example}), and show consistent gains on synthetic diagnostics and long-horizon conversational memory benchmarks.
\end{itemize}
More broadly, our results suggest a simple principle for agent memory design: \emph{remember the decision, not the description}.

\vspace{-0.5em}
\section{Related Work}
\label{sec:related}
\vspace{-0.5em}

\paragraph{State abstraction and decision-preserving compression.}
Our work is theoretically closest to state abstraction and decision-preserving
compression in sequential decision making, including model minimization,
approximate aggregation, and policy abstraction\citep{Li2006TowardsAU,GIVAN2003163,ravindran2004approximate,abel2017nearoptimalbehaviorapproximate,abel2018lifelong,zhang2022unifiedpolicyabstractiontheory}. These works study when multiple states can be treated as equivalent for control, and characterize abstraction quality through value, model, or policy preservation. A closely related line develops behavioral metrics, including bisimulation-style distances for measuring when states are approximately interchangeable for decision making \citep{Bisimulation11Ferns}, while more recent work extends homomorphism ideas to continuous settings \citep{rezaeishoshtari2022continuousmdphomomorphismshomomorphic}. Our setting differs mainly in the object of compression. 
Prior work abstracts environment states or dynamics for planning, whereas we compress interaction histories into a limited set of runtime memory states and measure the resulting loss in decision quality.
\vspace{-0.5em}
\paragraph{Task-driven compression and bounded interaction.}
Our formulation is also related in spirit to task-driven information bottleneck approaches to estimation and control, which advocate preserving only the information needed for downstream decisions rather than for faithful reconstruction \citep{Task-Driven19Pacelli}, as well as recent decision-centric perspectives on designing LLM systems around downstream utility rather than intermediate descriptive objectives \citep{Sun2026DecisionCentricDF}. A second nearby line studies control under bounded internal state, including finite-state controllers for partially observable settings \citep{simão2023safepolicyimprovementpomdps} and, more recently, canonical abstractions induced by agent-bounded indistinguishability in bounded interaction \citep{nixon2026myhillnerodetheoremboundedinteraction}. We share the premise that capability constraints determine which distinctions matter, but provide concrete formal tools for the memory setting: an exact forgetting boundary, a memory–distortion frontier, and finite-sample regret guarantees for online refinement.
\vspace{-0.5em}
\paragraph{Memory systems and learned management for LLM agents.}

A large body of work augments language agents with external or persistent memory, including retrieval-based memory \citep{lewis2021retrievalaugmentedgenerationknowledgeintensivenlp}, long-term memory modules \citep{NEURIPS2023_ebd82705}, operating-system-style or harness-managed memory \citep{packer2024memgptllmsoperatingsystems,kang-etal-2025-memory,Rafique2026ClawVMHV}, graph-structured or agentically organized memory \citep{gutiérrez2025hipporagneurobiologicallyinspiredlongterm,xu2025amemagenticmemoryllm}, production-oriented memory systems \citep{chhikara2025mem0buildingproductionreadyai}, and broader harness or externalization layers that manage context, tool use, persistent state, and long-horizon execution \citep{pan2026naturallanguageagentharnesses,lee2026metaharnessendtoendoptimizationmodel,zhou2026externalizationllmagentsunified}. 
Another line studies learned or agentic memory control, including dynamic note organization, memory actions, RL-based long-/short-term memory management, atomic memory operations, self-evolving memory skills or architectures \citep{xu2025amemagenticmemoryllm,zhang2026memoryactionautonomouscontext,yu2026agenticmemorylearningunified,huo2026atommemlearnabledynamic,zhang2026memskilllearningevolvingmemory,zhang2025memevolvemetaevolutionagentmemory}, as well as memory--reasoning coupling, episodic memory, temporal memory selection, and generative latent memory for long-horizon or self-evolving agents \citep{zhou2026mem,shu2026remem,du2026memoryt,zhang2026memgen}. 
Recent surveys organize agent memory beyond the long-/short-term dichotomy, using taxonomies based on representation, function, and memory dynamics \citep{zhang2024surveymemorymechanismlarge,hu2026memoryageaiagents,du2026memoryautonomousllmagentsmechanisms}, while benchmarks show that robust long-term, incremental, and action-coupled memory remains difficult \citep{maharana-etal-2024-evaluating,wu2025longmemevalbenchmarkingchatassistants,hu2026evaluating,li2026locomoplusbeyondfactualcognitivememory,tan2025membenchcomprehensiveevaluationmemory,he2026memoryarenabenchmarkingagentmemory,shen2026mem2actbenchbenchmarkevaluatinglongterm}. 
These works advance memory mechanisms, training procedures, and evaluations; our focus is complementary, characterizing the decision loss induced by bounded memory states and when histories can be safely merged.

\vspace{-0.5em}
\section{Problem Setup and Fundamental Limits}
\label{sec:setup}
\vspace{-0.5em}
This section develops the decision-centric memory formulation from the basic
runtime bottleneck to its fundamental limits. We first formalize the \(K\)-state memory constraint, where the agent can
access history only through a bounded internal state. Given this bottleneck, we then define the decision distortion induced by
compression, which measures how much decision quality is lost when histories are
merged. This loss notion allows us to characterize when merging is safe: Sections~\ref{sec:forgetting}--\ref{sec:limits} derive an exact forgetting boundary and
covering/packing measures for how many decision-distinguishable situations must
remain separated. Finally, Section~\ref{sec:hardness} shows that computing the
optimal partition is computationally hard, motivating the online methods in
Section~\ref{sec:algo}.
\vspace{-0.3em}
\subsection{Memory-constrained contextual decision model}
\label{sec:model}
\vspace{-0.3em}
We study a contextual decision problem under an explicit \emph{runtime memory budget}, instantiated analytically as a contextual bandit.
Let $\mathcal{H}$ be a finite set of dialogue histories, $\mathcal{Q}$ a finite set of queries, and $\cA$ a finite set of downstream decisions with $|\cA|=A$.
A context is a pair $x=(h,q)\in\cX\subseteq\mathcal{H}\times\mathcal{Q}$.
In the conversational benchmarks of Section~\ref{sec:exp}, $h$ is the dialogue history, $q$ is the evaluation query, and the action is the agent's generated answer.
At each round $t=1,\dots,T$:
(i)~a context $X_t=(H_t,Q_t)\sim\cD$ is drawn i.i.d.\ from an unknown distribution $\cD$ over $\cX$;
(ii)~the learner chooses a decision $A_t\in\cA$;
(iii)~the learner observes a reward $R_t\in[0,1]$ with conditional mean
\begin{equation}
\label{eq:mean-reward}
\EE[R_t \mid X_t=x,\, A_t=a] \;=\; \mu(x,a),
\end{equation}
where $\mu:\cX\times\cA\to[0,1]$ is fixed but unknown.
The full-information oracle chooses
$a^\star(x)\in\arg\max_{a\in\cA}\mu(x,a)$
with value $\mu^\star(x)\coloneqq \max_{a}\mu(x,a)$, and we measure regret
\begin{equation}
\label{eq:regret}
\Reg(T)\;\coloneqq\; \sum_{t=1}^{T}\bigl(\mu^\star(X_t)-\mu(X_t,A_t)\bigr).
\end{equation}

\vspace{-1em}
This i.i.d.\ sampling assumption is over answer-time decision instances,
not over individual dialogue turns; we defer the modeling clarification to
Appendix~\ref{app:model-details}.

\vspace{-0.5em}
\paragraph{Strict runtime memory constraint.}
At answer time, the current query \(Q_t\) is observed directly and is not
counted against the memory budget. The bottleneck instead constrains access
to the history: before acting, the learner maps the history into one of
\(K\) runtime memory states through a query-aware encoder,
\begin{equation}
\label{eq:memory-policy}
M_t = g_t(H_t,Q_t)\in[K],
\qquad
A_t = \pi_t(M_t,Q_t),
\end{equation}
where \(g_t:\mathcal{H}\times\mathcal{Q}\to[K]\) is a deterministic encoder
and \(\pi_t:[K]\times\mathcal{Q}\to\cA\) is a deterministic decision rule.
Thus, different queries may select different decision-relevant aspects of
the same history, but the policy can act only from the current query and the
selected memory state, not from the full history directly.

The resulting limits are query-conditional: for fixed \(q\), we ask which
histories can share a memory state without degrading decision quality. We
define the static memory--distortion frontier over deterministic answer-time
selectors, the class induced by deterministic decoding; Appendix~\ref{app:deterministic}
separates this modeling choice from the exploration randomness used by the
online learner. Proposition~\ref{prop:bridge} makes the theory--implementation
correspondence precise.

\begin{figure*}[t]
\centering
\includegraphics[width=\textwidth]{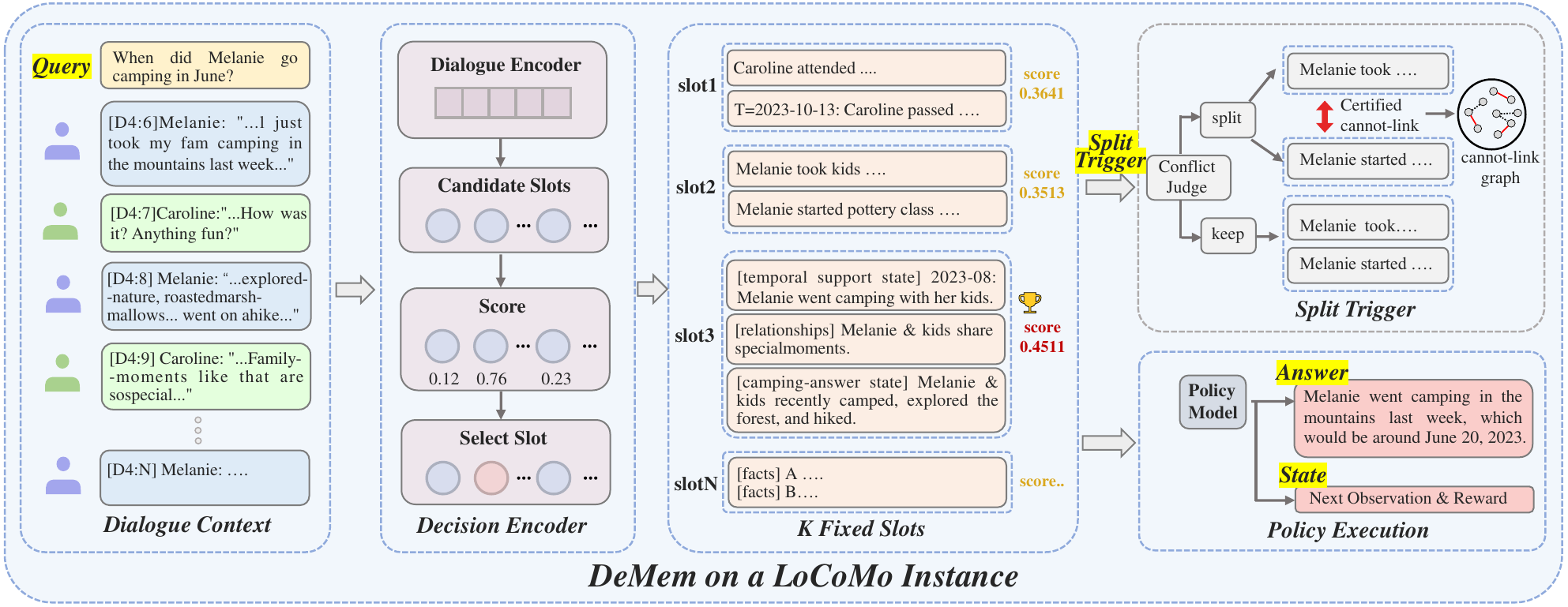}
\vspace{-1em}
\caption{DeMem routes histories into bounded slots and splits only on certified conflict.}
\label{fig:worked-example}
\end{figure*}

\vspace{-0.5em}
\subsection{Decision distortion under memory}
\label{sec:distortion}
\vspace{-0.5em}
Having fixed the \(K\)-state memory interface, we now quantify what is lost
when histories that may require different decisions are merged.
Fix \(q\in\mathcal Q\), and let
\(\mathcal X_q \coloneqq \{h\in\mathcal H : (h,q)\in\mathcal X\}\)
be the set of histories that may appear with query \(q\). For each
history--action pair \((h,a)\), write
\(\mu_q(h,a) \coloneqq \mu((h,q),a)\) and
\(\mu_q^\star(h) \coloneqq \max_{a\in\cA}\mu_q(h,a)\).
Thus, \(\mu_q^\star(h)\) is the value of the best action when the full
history \(h\) is available for answering query \(q\).

For any action \(a\), define its suboptimality gap at history \(h\) as
\begin{equation}
\label{eq:gap}
\Delta_q(h,a)
\coloneqq
\mu_q^\star(h)-\mu_q(h,a)
\in[0,1].
\end{equation}

A \(K\)-state memory encoder assigns each history to one of \(K\)
compressed memory states. For fixed \(q\), let
\(g_q:\mathcal X_q\to[K]\) denote the induced encoder
\(g_q(h)=g(h,q)\), and let \(\pi_q:[K]\to\cA\) denote the induced
decision rule \(\pi_q(m)=\pi(m,q)\). If history \(h\) is routed to state
\(m\), the resulting decision distortion is
\begin{equation}
\label{eq:distortion}
d_q(h,m;\pi_q)
\coloneqq
\mu_q^\star(h)-\mu_q\bigl(h,\pi_q(m)\bigr)
=
\Delta_q\bigl(h,\pi_q(m)\bigr).
\end{equation}
This measures how much decision quality is lost because the policy sees
only the compressed memory state \(m\), rather than the full history \(h\).

The best worst-case distortion achievable with \(K\) memory states for
query \(q\) is
\begin{equation}
\label{eq:epsK}
\eps^\star_\infty(K;q)
\coloneqq
\inf_{g_q:\mathcal X_q\to[K]}
\inf_{\pi_q:[K]\to\cA}
\sup_{h\in\mathcal X_q}
d_q\bigl(h,g_q(h);\pi_q\bigr).
\end{equation}

We call \(\eps^\star_\infty(K;q)\) the decision rate--distortion frontier
for query \(q\): the smallest possible worst-case decision loss incurred
by compressing histories into only \(K\) memory states. When the fixed
query is clear from context, we write \(\eps^\star_\infty(K)\).
\vspace{-0.5em}
\subsection{Exact forgetting boundary}
\label{sec:forgetting}
\vspace{-0.5em}
The distortion frontier tells us \emph{how much} decision quality is lost under a memory budget; we now characterize \emph{when} histories can be merged with
decision loss at most \(\eps\).

\begin{theorem}[Exact forgetting boundary]
\label{thm:forgetting}
Fix a query \(q\), $\eps\ge 0$ and a nonempty $C\subseteq\cX_q$.
The following are equivalent:
\begin{enumerate}[nosep]
    \item There exists $a\in\cA$ with $\Delta_q(h,a)\le\eps$ for all $h\in C$.
    \item There exists a one-state encoder on \(C\) and a deterministic decision
rule whose worst-case distortion over \(C\) is at most \(\eps\).
\end{enumerate}
\end{theorem}

\noindent A short proof appears in Appendix~\ref{app:proof-forgetting}.
Within a fixed query fiber, the criterion is purely decision-theoretic:
histories may share a memory state precisely when they admit a common
\(\eps\)-optimal action for that query---not when they are descriptively
similar or have close reward vectors. Across different queries, the same
abstract memory state may support different actions through the original
selector \(\pi(m,q)\).

Theorem~\ref{thm:forgetting} motivates a natural distance. Define the
\emph{pairwise decision distance}
\begin{equation}
\label{eq:ddec}
d_{\mathrm{dec}}^q(h,h')
\;\coloneqq\;
\min_{a\in\cA}\,\max\!\bigl\{\Delta_q(h,a),\;\Delta_q(h',a)\bigr\},
\end{equation}
the smallest \(\eps\) for which \(h\) and \(h'\) can share a state at
distortion \(\eps\). More generally, the \emph{cluster decision radius} of
\(C\subseteq\cX_q\) is
\[
\rho_{\mathrm{dec}}^q(C)
\;\coloneqq\;
\min_{a\in\cA}\,\max_{h\in C}\,\Delta_q(h,a),
\]
the smallest worst-case loss from assigning a single action to all histories
in \(C\). Because \(d_{\mathrm{dec}}^q\) need not satisfy the triangle
inequality, the complexity measures below are defined through cluster radii
directly.

\vspace{-0.3em}
\subsection{Decision covering and packing bounds}
\label{sec:limits}
\vspace{-0.3em}
Given the forgetting boundary and the decision distance, we can now determine
the minimum number of memory states necessary to constrain the distortion
within a target \(\eps\) for the fixed query \(q\). 
We introduce two combinatorial quantities. 
The decision covering number
\begin{equation}
\label{eq:covering}
\cN^{\mathrm{dec}}_{\mathrm{cov}}(\eps;q)
\;\coloneqq\;
\min\Bigl\{
K : \exists\, \cX_q=\bigsqcup_{m=1}^{K} C_m
\ \text{s.t.}\ 
\max_m \rho_{\mathrm{dec}}^q(C_m)\le \eps
\Bigr\}
\end{equation}
is the fewest states that achieve worst-case distortion \(\eps\) for query
\(q\). The decision packing number
\begin{equation}
\label{eq:packing}
\cN^{\mathrm{dec}}_{\mathrm{pack}}(\eps;q)
\;\coloneqq\;
\max\Bigl\{
|P| : P\subseteq\cX_q,\ 
\forall h\ne h'\in P,\ 
d_{\mathrm{dec}}^q(h,h')>\eps
\Bigr\}
\end{equation}
is the largest set of mutually decision-incompatible histories at scale
\(\eps\) for query \(q\).

\begin{theorem}[Decision covering and packing bounds]
\label{thm:cover-pack}
For every \(\eps>0\):

(1) if \(K\ge \cN^{\mathrm{dec}}_{\mathrm{cov}}(\eps;q)\), then
\(\eps^\star_\infty(K;q)\le \eps\);
(2) if \(\cN^{\mathrm{dec}}_{\mathrm{pack}}(2\eps;q)>K\), then
\(\eps^\star_\infty(K;q)>\eps\).

Moreover, any stochastic memory variable \(M\) supporting worst-case
distortion at most \(\eps\) over the fixed query fiber must satisfy
\[
I(Z;M) \ge \log \cN^{\mathrm{dec}}_{\mathrm{pack}}(2\eps;q),
\]
where \(Z\) is uniform over a maximum decision packing in \(\cX_q\).
\end{theorem}

\noindent Proofs are in Appendices~\ref{app:proof-covering}--\ref{app:info-lower-bound}.
The relevant notion of task entropy is the number of \emph{decision-distinguishable} situations that must remain separable, not generic representational complexity.
A distribution-dependent refinement replacing the worst-case frontier with average cluster radii is given in Appendix~\ref{app:avg-distortion} (Proposition~\ref{prop:avg-lb}); in Section~\ref{sec:exp} we directly measure its empirical counterpart.

\vspace{-0.5em}
\subsection{Computational hardness of optimal memory}
\label{sec:hardness}
\vspace{-0.5em}
The covering and packing bounds determine the minimum budget at any target distortion, but can this optimal partition be found efficiently?
\begin{theorem}[Computational hardness]
\label{thm:np-hard}
Given $\mu\in[0,1]^{N\times A}$, an integer $K$, and a rational
$0\le \eps < 1$, deciding whether $\eps^\star_\infty(K)\le\eps$ is
\textup{NP}-complete.
\end{theorem}

\noindent The proof is in Appendix~\ref{app:proof-hardness}; this NP-hardness motivates DeMem: instead of solving the global optimal
partition exactly, it learns a conservative decision-compatible partition
from data through certified refinements.

\vspace{-0.5em}
\section{Online Learning with Certified Memory Splits}
\label{sec:algo}
\vspace{-0.5em}
We now return to the online setting, where \(x=(h,q)\) denotes a full
answer-time context, while compatibility is assessed within each fixed-query
fiber. The frontier $\eps^\star_\infty(K)$ describes the best achievable compression
quality, but it depends on the unknown reward function $\mu$.
We therefore need an online learner that simultaneously
(i)~estimates which contexts are decision-compatible,
(ii)~respects the hard $K$-state budget,
and (iii)~refines memory only when the data justify doing so.
Our method, DeMem (\textbf{De}cision-aware \textbf{Mem}ory), is built
around one principle:
\begin{quote}
\centering
\emph{split a memory state only when the data certify that some pair of
contexts inside it cannot safely share a near-optimal action.}
\end{quote}
A full specification is given in Algorithm~\ref{alg:DeMem}
(Appendix~\ref{app:DeMem-algorithm}), while
Algorithm~\ref{alg:demem-overview} summarizes the epoch-level
execution flow and Figure~\ref{fig:worked-example} provides an
intuitive worked example.
We summarize the two key mechanisms below and then state the regret guarantees.

\vspace{-0.3em}
\subsection{Certified incompatibility and budgeted encoding}
\label{sec:certificates}
\vspace{-0.3em}
\paragraph{Incompatibility certificates.}
For each observed context--action pair \((x,a)\), let \(\hat\mu_t(x,a)\) be the
empirical mean, \(n_t(x,a)\) the count. From these we construct lower certificates
\(\underline d_t(x,x')\) for the pairwise decision distance satisfying,
with probability at least \(1-\delta\),
\begin{equation}
\label{eq:ddec-lower}
\underline d_t(x,x') \le d_{\mathrm{dec}}(x,x')
\qquad
\text{for all sampled pairs and all } t .
\end{equation}
Whenever \(\underline d_t(x,x')>\eps\), the data certify that \(x\) and
\(x'\) cannot share any memory state whose cluster radius is at most \(\eps\).
This yields a one-sided cannot-link relation.

Importantly, the converse is not assumed: the absence of pairwise
cannot-link edges inside a cluster does not by itself certify that the
cluster admits a single shared \(\eps\)-optimal action. For this reason,
the regret bound below is stated in terms of the realized cluster radius,
certified by \(\overline\rho_t(C)\), rather than in terms of pairwise graph
independence alone. The full construction and validity guarantees appear in
Appendix~\ref{app:certificates}. For one-pass LLM benchmarks, we instantiate this criterion with a
candidate-restricted answer-conflict test; implementation details and
calibration checks are provided in Appendix~\ref{app:practical-split-trigger}.
\vspace{-0.5em}
\paragraph{Certification exploration.}
To make cannot-link certificates statistically testable, DeMem
ensures a minimum amount of context-level exploration before relying on
cluster-level UCB. For a resolution parameter \(\gamma>0\), each observed
context--action pair is sampled until
\[n_t(x,a)\ge
B_t(\gamma)
\coloneqq
\left\lceil
\frac{8\log(4NAt^2/\delta)}{\gamma^2}
\right\rceil .
\]

Once this coverage condition holds, the algorithm falls back to the
cluster-level UCB rule. Appendix~\ref{app:cert-exploration} explains how
this step separates certifiable positive-margin conflicts from conflicts
that remain part of the realized compression error.
\vspace{-0.5em}
\paragraph{Budgeted encoder.}
At the beginning of each epoch, DeMem builds a cannot-link graph
over observed contexts and applies a polynomial-time greedy coloring
subroutine to obtain a feasible \(K\)-slot partition \(\mathcal P_e\)
(Algorithm~\ref{alg:greedy-partition},
Appendix~\ref{app:frozen-protocol}). Its certified realized compression
price is
\[
\eps^{\mathrm{cert}}_e
\coloneqq
\max_{C\in\mathcal P_e}\overline\rho_{t_e}(C),
\]
which upper-bounds the true cluster distortion on the high-probability
confidence event. The partition is then frozen for the epoch, and
DeMem runs UCB at the memory-state level. The full graph
construction and frozen-epoch protocol are deferred to
Appendix~\ref{app:frozen-protocol}.

\vspace{-0.3em}
\subsection{Regret guarantees}
\label{sec:regret}
\vspace{-0.3em}
Even if $\mu$ were known, a $K$-state memory induces an unavoidable
approximation error; learning adds a second, statistical term.
Let $N_T\coloneqq|\{X_1,\dots,X_T\}|\le\min\{N,T\}$ be the number of distinct
contexts observed.

\begin{theorem}[Upper bound for DeMem]
\label{thm:upper}
Let \(\bar\eps^{\mathrm{cert}}_T
\coloneqq
\max_{e\le\lceil\log_2 T\rceil}\eps^{\mathrm{cert}}_e\), where
\(\eps^{\mathrm{cert}}_e\) is the certified compression price of the
epoch-\(e\) partition defined in Section~\ref{sec:certificates}. Let
\(B_T(\gamma)\) denote the certification threshold \(B_t(\gamma)\) evaluated
at \(t=T\). With probability at least \(1-\delta\),
\begin{equation}
\label{eq:upper}
\Reg(T)
\;\le\;
T\cdot O\!\bigl(\bar\eps^{\mathrm{cert}}_T\bigr)
\;+\;
\tilde O\!\bigl(\sqrt{AKT}\bigr)
\;+\;
O\!\bigl(A N_T B_T(\gamma)\bigr).
\end{equation}
\end{theorem}

The three terms correspond to realized compression error, statistical
learning over \(K\) memory states, and context-level certification
exploration; the last term is \(\tilde O(A N_T/\gamma^2)\). This cost makes cannot-link certificates statistically valid, avoiding a pairwise-independence proxy for cluster compatibility. By
Lemma~\ref{lem:cert}, the certified compression prices upper-bound the true
cluster radii on the high-probability event, so the regret is charged to the
realized cluster-level distortion rather than to pairwise graph
independence. A proof sketch is in Appendix~\ref{app:upper-proof}.

\begin{theorem}[Minimax lower bound under $K$-state memory]
\label{thm:lower}
There exists a universal constant $c>0$ such that
\begin{equation}
\label{eq:lower}
\inf_{\mathrm{alg}}\;\sup_{\mu,\cD}\;
\EE[\Reg(T)]
\;\ge\;
c\,T\,\eps^\star_\infty(K)
\;+\;
c\,\sqrt{AKT},
\end{equation}
where the infimum ranges over algorithms whose answer-time decision rule is
deterministic and whose access to the history is mediated by a \(K\)-state
runtime memory.
\end{theorem}

\noindent A proof sketch is in Appendix~\ref{app:lower-proof}.
Together, Theorems~\ref{thm:upper} and~\ref{thm:lower} show that
DeMem matches the minimax statistical learning term up to
logarithmic factors. The remaining excess term is the realized compression
price
\[
\bar\eps^{\mathrm{cert}}_T-\eps^\star_\infty(K),
\]
together with the context-certification overhead. The greedy graph routine
affects how the partition is found, but the regret bound is charged to the
certified cluster radius of the returned partition. Block-stationary
extensions are stated in Appendix~\ref{app:nonstationary}.

\vspace{-0.5em}
\paragraph{From theory to slot-based agent memory.}
A $K$-slot memory system directly realizes the runtime bottleneck of
Section~\ref{sec:setup}: the router implements $g$, the selected slot identity
is $M_t$, and the policy conditions on the current query together with the
selected memory state. The complete DeMem procedure is provided in
Appendix~\ref{app:DeMem-algorithm}.
Proposition~\ref{prop:bridge} (Appendix~\ref{app:impl}) decomposes the
realized distortion into
\[
\underbrace{\eps^\star_\infty(K)}_{\text{compression}}
\;+\;
\underbrace{\eta_{\mathrm{route}}}_{\text{routing}}
\;+\;
\underbrace{\eta_{\mathrm{read}}}_{\text{realization}}.
\]
The complete DeMem procedure is given in
Appendix~\ref{app:DeMem-algorithm}, and the theory--implementation
correspondence is summarized in the alignment table of
Appendix~\ref{app:alignment-table}. Appendix~\ref{app:bridge-empirical}
empirically validates the decomposition above, while
Appendix~\ref{app:modularity} shows that the same decision-aware selection
criterion can improve other memory systems as a drop-in component.

\begin{algorithm}[t]
\caption{DeMem in one loop}
\label{alg:demem-overview}
\begin{algorithmic}[1]
\State Initialize a \(K\)-slot partition \(\mathcal P_0\) and context--action statistics
\For{each epoch \(e\)}
    \State \textbf{Act:} freeze \(\mathcal P_e\) and route each context,
    \(m_t \leftarrow g_{\mathcal P_e}(x_t)\)
    \State choose \(a_t\) by exploration or slot-level UCB,
    \(a_t \leftarrow \mathrm{Explore/UCB}(m_t,x_t)\); update statistics
    \State \textbf{Certify:} build cannot-link edges
    \(E_e=\{(x,x'):\underline d_e(x,x')>\eps\}\)
    \State \textbf{Refresh:} compute
    \(\mathcal P_{e+1}\leftarrow \mathrm{GreedyColor}_K(E_e)\)
\EndFor
\end{algorithmic}
\end{algorithm}
\begin{figure*}[t]
\centering
\begin{minipage}[t]{0.32\textwidth}
\centering
\includegraphics[width=\linewidth]{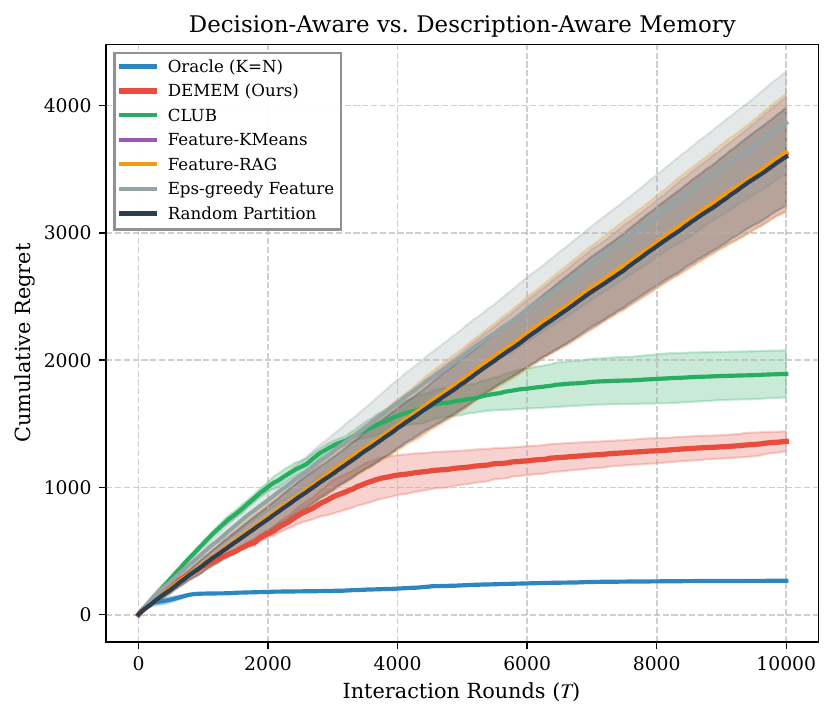}
\\[-1mm]
\small (a) Cumulative regret.
\end{minipage}\hfill
\begin{minipage}[t]{0.32\textwidth}
\centering
\includegraphics[width=\linewidth]{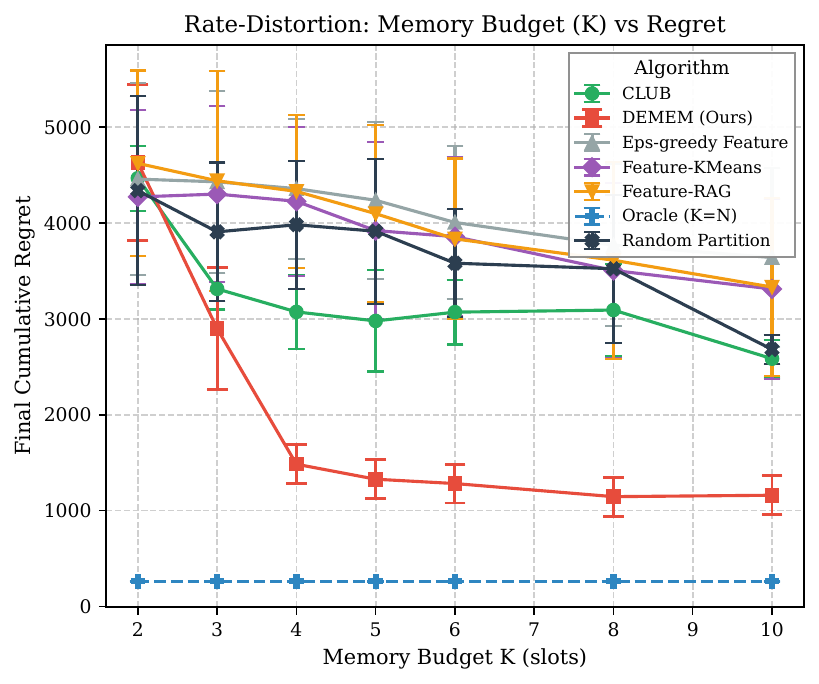}
\\[-1mm]
\small (b) Memory Budget vs. Regret.
\end{minipage}\hfill
\begin{minipage}[t]{0.32\textwidth}
\centering
\includegraphics[width=\linewidth]{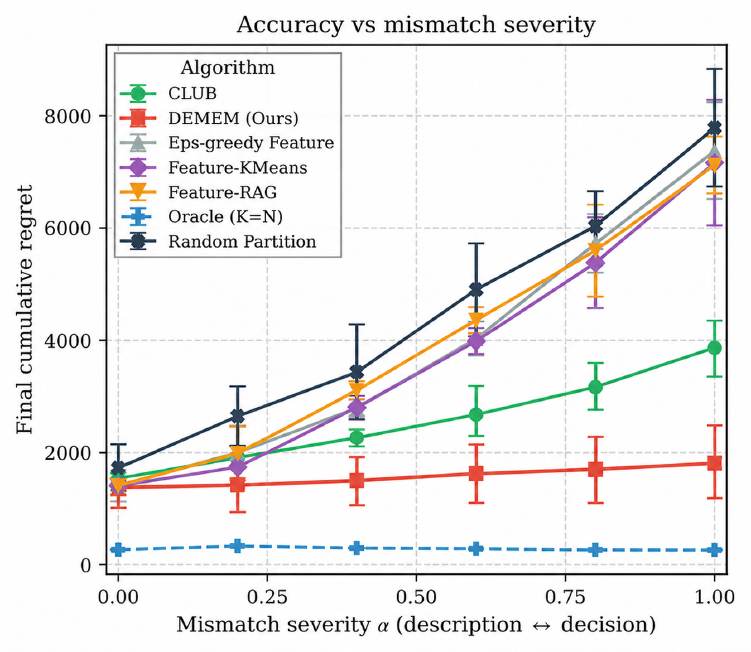}
\\[-1mm]
\small (c) Mismatch mechanism analysis.
\end{minipage}
\caption{Synthetic results.
(a) Cumulative regret on the Decoupled Bandit.
(b) Memory--distortion tradeoff under varying $K$.
(c) Performance under increasing description--decision mismatch.}
\label{fig:synthetic-results}
\end{figure*}

\vspace{-0.5em}
\section{Experiments}
\label{sec:exp}
\vspace{-0.5em}
\subsection{Experimental Setup}

\vspace{-0.3em}
\paragraph{Synthetic environment.}
We use a \emph{Decoupled Bandit} in which each context has a descriptive representation and a latent decision identity that are intentionally misaligned: descriptively similar contexts need not share the same optimal action.
Baselines include an unconstrained \textbf{Oracle}, \textbf{CLUB}-style online clustering~\citep{gentile2014onlineclusteringbandits}, description-aware methods (\textbf{Feature-KMeans}, \textbf{Feature-RAG}, $\epsilon$-greedy feature clustering), and a \textbf{Random Partition} control.
We report cumulative regret, an empirical memory--distortion curve (varying $K$), and a mismatch-severity sweep interpolating between aligned and misaligned regimes.
Full definitions are in Appendix~\ref{app:synth}.
\vspace{-0.5em}
\paragraph{benchmarks and baselines.}
We evaluate on \textbf{LoCoMo}~\citep{maharana-etal-2024-evaluating}, \textbf{LongMemEval}~\citep{wu2025longmemevalbenchmarkingchatassistants}, and \textbf{MemoryArena}~\citep{he2026memoryarenabenchmarkingagentmemory} against \textbf{FullContext}, \textbf{RAG}~\citep{lewis2021retrievalaugmentedgenerationknowledgeintensivenlp}, \textbf{LangMem}~\citep{langchain_langmem_2024}, \textbf{Mem0}~\citep{chhikara2025mem0buildingproductionreadyai}, \textbf{Zep}~\citep{rasmussen2025zeptemporalknowledgegraph}, \textbf{Nemori}~\citep{nan2025nemoriselforganizingagentmemory}, \textbf{EMem-G}~\citep{zhou2025simplestrongbaselinelongterm}, and \textbf{Mnemis}~\citep{tang2026mnemisdualrouteretrievalhierarchical}.
All methods share the same answering backbones (\texttt{gpt-4o-mini}, \texttt{gpt-4.1-mini}), deterministic decoding.
Answers are scored by two held-out LLM judges, \texttt{gpt-4o-mini}
and \texttt{gpt-4.1-mini}(binary on LoCoMo, graded on LongMemEval); reported numbers average across judges and instances.A human-agreement study on 150 stratified LoCoMo instances confirms
that the LLM judges align well with human majority vote
($\kappa = 0.79$, 91.3\% agreement), with no systematic bias toward
any method (Appendix~\ref{app:human-judge}).
Details on preprocessing, prompting, baseline instantiation, budget matching, and judge prompts are in Appendix~\ref{app:bench-details}.

\vspace{-0.3em}
\subsection{Synthetic Results}
\label{sec:synthetic}
\vspace{-0.3em}

Figure~\ref{fig:synthetic-results} summarizes the synthetic results.
DeMem achieves the lowest cumulative regret among all budgeted methods~(a), attains a more favorable memory--distortion tradeoff~(b), and widens its advantage as the mismatch between descriptive and decision similarity grows~(c).
Description-aware baselines merge contexts that require different actions, leading to persistent decision error; DeMem avoids this by refining memory only on certified decision conflict.
We additionally measure the induced partition distortion (Appendix~\ref{app:partition-distortion}, Figure~\ref{fig:partition-validation}): DeMem consistently learns the lowest-distortion partition, and downstream regret is broadly monotone in this quantity across methods and seeds, confirming the mechanism predicted by the theory.
Full environment definitions and baseline specifications are in Appendix~\ref{app:synth}.

\begin{table*}[t]
\centering
\caption{Results on LoCoMo. Mean \(\pm\) std over runs. }
\label{tab:locomo-main}
\small
\resizebox{\textwidth}{!}{%
\begin{tabular}{llccccc}
\toprule
\textbf{Backbone} & \textbf{Method}
& \textbf{Temporal}
& \textbf{Open Domain}
& \textbf{Multi-Hop}
& \textbf{Single-Hop}
& \textbf{Overall} \\
\midrule
\multirow{9}{*}{GPT-4o-mini}
& FullContext & $0.515_{\pm 0.050}$ & $0.428_{\pm 0.067}$ & $0.629_{\pm 0.057}$ & $0.791_{\pm 0.039}$ & $0.681_{\pm 0.053}$ \\
& LangMem     & $0.208_{\pm 0.090}$ & $0.423_{\pm 0.108}$ & $0.470_{\pm 0.096}$ & $0.565_{\pm 0.079}$ & $0.464_{\pm 0.093}$ \\
& Mem0        & $0.453_{\pm 0.072}$ & $0.365_{\pm 0.099}$ & $0.552_{\pm 0.080}$ & $0.644_{\pm 0.066}$ & $0.570_{\pm 0.077}$ \\
& RAG         & $0.572_{\pm 0.101}$ & $0.590_{\pm 0.126}$ & $0.543_{\pm 0.120}$ & $0.698_{\pm 0.108}$ & $0.637_{\pm 0.115}$ \\
& Zep         & $0.545_{\pm 0.076}$ & $0.336_{\pm 0.100}$ & $0.457_{\pm 0.088}$ & $0.592_{\pm 0.073}$ & $0.541_{\pm 0.082}$ \\
& Nemori      & $0.656_{\pm 0.057}$ & $0.401_{\pm 0.086}$ & $0.596_{\pm 0.068}$ & $0.787_{\pm 0.045}$ & $0.701_{\pm 0.060}$ \\
& EMem-G      & $0.717_{\pm 0.053}$ & $0.517_{\pm 0.078}$ & $0.702_{\pm 0.059}$ & $0.782_{\pm 0.042}$ & $0.737_{\pm 0.051}$ \\
& Mnemis      & $0.878_{\pm 0.037}$ & $0.793_{\pm 0.060}$ & $0.798_{\pm 0.054}$ & \textbf{0.938}$_{\pm 0.028}$ & $0.888_{\pm 0.042}$ \\
& \textbf{DeMem}
  & \textbf{0.919}$_{\pm 0.031}$
  & \textbf{0.868}$_{\pm 0.043}$
  & \textbf{0.847}$_{\pm 0.042}$
  & $0.932_{\pm 0.018}$
  & \textbf{0.911}$_{\pm 0.040}$ \\
\midrule
\multirow{9}{*}{GPT-4.1-mini}
& FullContext & $0.692_{\pm 0.057}$ & $0.511_{\pm 0.081}$ & $0.725_{\pm 0.053}$ & $0.831_{\pm 0.040}$ & $0.763_{\pm 0.058}$ \\
& LangMem     & $0.461_{\pm 0.087}$ & $0.531_{\pm 0.080}$ & $0.659_{\pm 0.068}$ & $0.800_{\pm 0.046}$ & $0.687_{\pm 0.063}$ \\
& Mem0        & $0.514_{\pm 0.077}$ & $0.428_{\pm 0.093}$ & $0.629_{\pm 0.066}$ & $0.669_{\pm 0.062}$ & $0.614_{\pm 0.071}$ \\
& RAG         & $0.710_{\pm 0.110}$ & $0.634_{\pm 0.117}$ & $0.593_{\pm 0.103}$ & $0.727_{\pm 0.113}$ & $0.693_{\pm 0.106}$ \\
& Zep         & $0.554_{\pm 0.076}$ & $0.383_{\pm 0.097}$ & $0.495_{\pm 0.084}$ & $0.623_{\pm 0.069}$ & $0.570_{\pm 0.073}$ \\
& Nemori      & $0.731_{\pm 0.056}$ & $0.454_{\pm 0.086}$ & $0.706_{\pm 0.059}$ & $0.811_{\pm 0.042}$ & $0.753_{\pm 0.051}$ \\
& EMem-G      & $0.757_{\pm 0.047}$ & $0.660_{\pm 0.064}$ & $0.747_{\pm 0.052}$ & $0.872_{\pm 0.031}$ & $0.812_{\pm 0.040}$ \\
& Mnemis      & $0.891_{\pm 0.031}$ & $0.858_{\pm 0.034}$ & $0.805_{\pm 0.046}$ & \textbf{0.940}$_{\pm 0.025}$ & $0.906_{\pm 0.039}$ \\
& \textbf{DeMem}
  & \textbf{0.921}$_{\pm 0.039}$
  & \textbf{0.908}$_{\pm 0.040}$
  & \textbf{0.831}$_{\pm 0.048}$
  & $0.935_{\pm 0.037}$
  & \textbf{0.920}$_{\pm 0.038}$ \\
\bottomrule
\end{tabular}
}
\end{table*}

\vspace{-0.3em}
\subsection{Main Results on LoCoMo and LongMemEval}
\vspace{-0.3em}
Table~\ref{tab:locomo-main} reports results on \textbf{LoCoMo}.
DeMem achieves the highest mean overall score on both backbones.
The pattern is especially clear on Temporal and Open Domain questions, where
preserving distinctions across distant interactions is important.
On LongMemEval (Table~\ref{tab:longmemeval-main}, Appendix~\ref{app:longmemeval-table}),
 DeMem again obtains the best mean overall score on both backbones,
with the largest gains on categories requiring cross-session integration.
DeMem also lies on a favorable accuracy--resource frontier
(Appendix~\ref{app:budget-latency}, Figure~\ref{fig:acc-budget-latency}),
indicating that the gains come from more efficient use of the fixed memory
budget rather than from additional memory or compute.
Results under Llama-3.1-70B (Appendix~\ref{app:llama-results}) confirm that
the advantage is not tied to a specific backbone family and generalizes
across both proprietary and open-weight models.
Mechanism audits on real benchmarks show the same pattern: DeMem
forms purer, less conflicted memory states and yields partitions closer to
oracle support than non-decision-aware alternatives
(Appendices~\ref{app:locomo-audit} and~\ref{app:partition-intervention}). Results on the agentic MemoryArena benchmark~\citep{he2026memoryarenabenchmarkingagentmemory}
show the same pattern: DeMem achieves the highest mean average success rate among the compared methods
(Appendix~\ref{app:memoryarena}).
\vspace{-0.3em}
\subsection{Ablation and Robustness}
\label{sec:ablation}
\vspace{-0.3em}
\begin{figure*}[t]
\centering
\begin{minipage}[t]{0.49\textwidth}
\centering
\includegraphics[width=\linewidth]{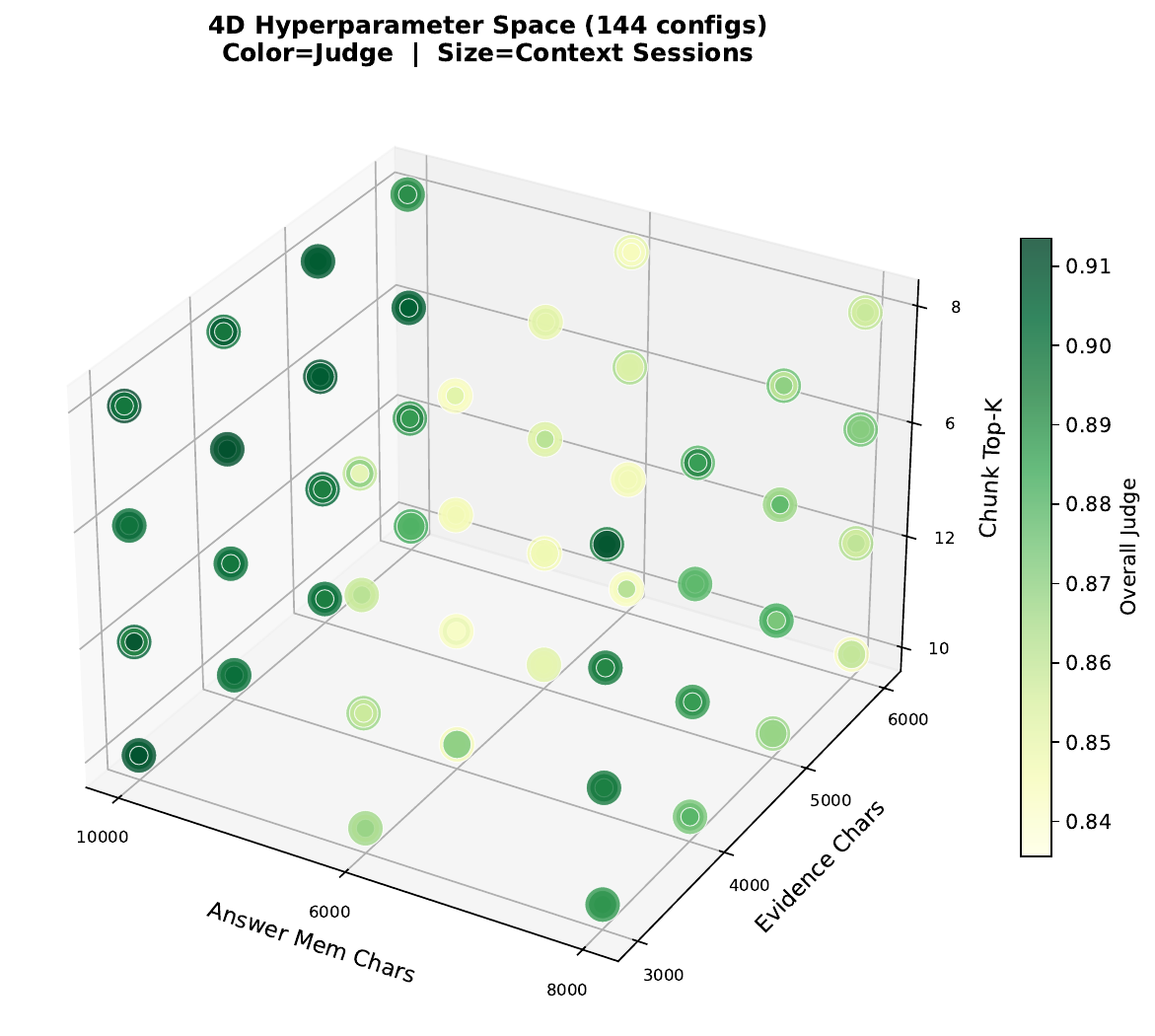}
\\[-1mm]
\small (a) Parameter sensitivity (144 configs).
\end{minipage}\hfill
\begin{minipage}[t]{0.49\textwidth}
\centering
\includegraphics[width=\linewidth]{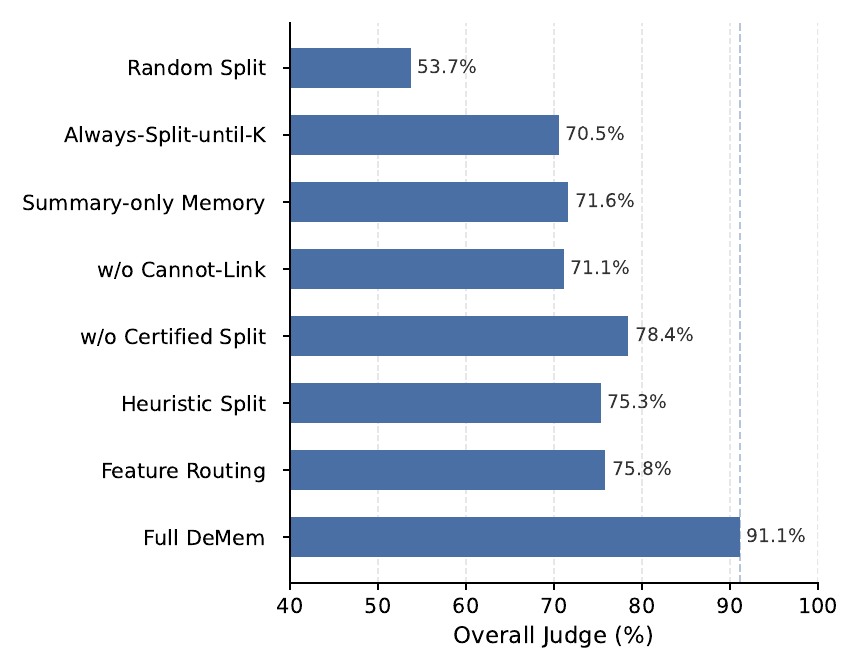}
\\[-1mm]
\small (b) Component ablations.
\end{minipage}

\vspace{1mm}

\begin{minipage}[t]{0.49\textwidth}
\centering
\includegraphics[width=\linewidth]{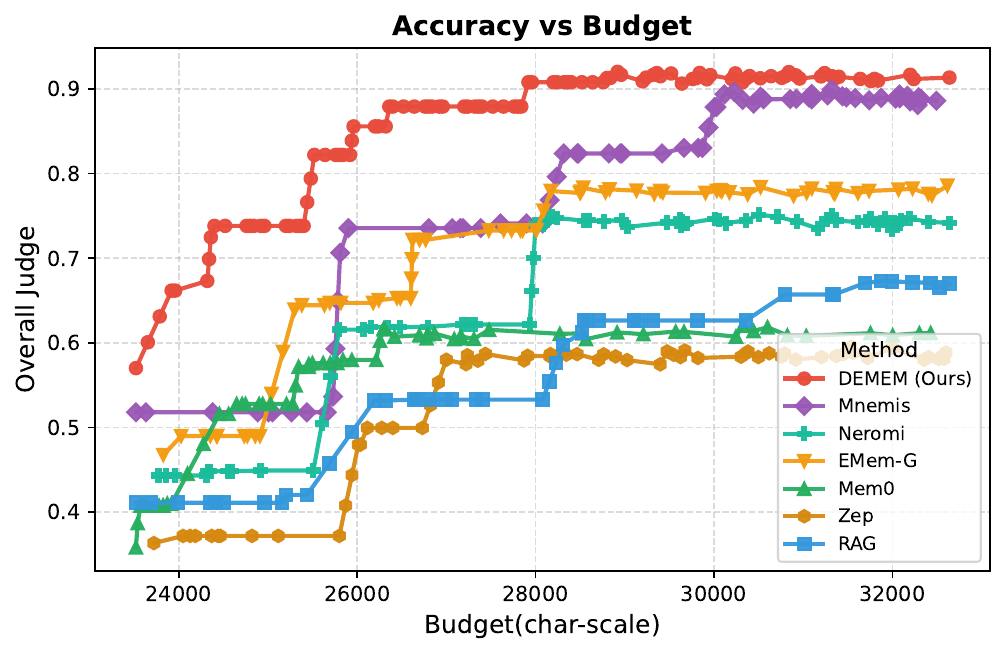}
\\[-1mm]
\small (c) Accuracy vs.\ memory budget.
\end{minipage}\hfill
\begin{minipage}[t]{0.49\textwidth}
\centering
\includegraphics[width=\linewidth]{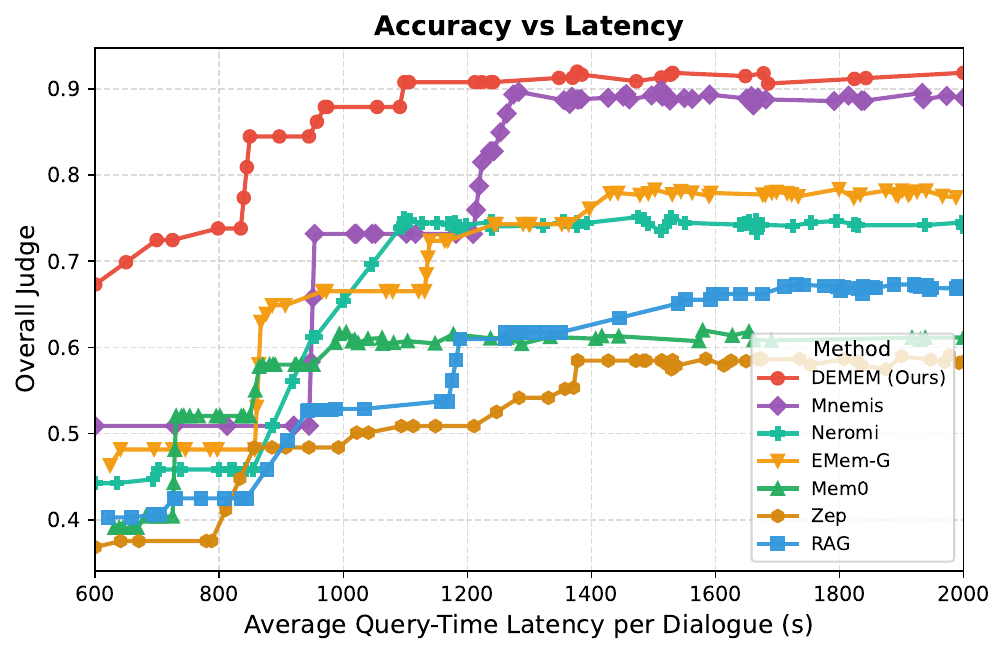}
\\[-1mm]
\small (d) Accuracy vs.\ latency.
\end{minipage}
\caption{Ablation and robustness. (a) Parameter sensitivity over 144 configurations. (b) Component ablations(defined in App~\ref{app:ablation-details}). (c) Accuracy vs. memory budget. (d) Accuracy vs. latency.}
\label{fig:ablation-param}
\label{fig:acc-budget-latency}
\end{figure*}

We ablate splitting, routing, and memory representation
(variant definitions in Appendix~\ref{app:ablation-details}).
Figure~\ref{fig:ablation-param}(b) shows that all ablations degrade
performance. Removing certified splitting is damaging: heuristic,
aggressive, and random split variants all underperform the certified
split rule. Certified splits are also selective and accurate, firing on
only \(4.6\%\) of LoCoMo routing events while reaching \(85\%\) precision
against gold annotations (Appendix~\ref{app:split-audit}). Feature-based
routing and description-only summaries further reduce performance,
supporting the need to preserve decision-relevant distinctions throughout
the memory pipeline.

Additional diagnostics support the same mechanism. On an annotated LoCoMo
subset, the observed answer-level distortion reaches about \(81\%\) of the
predicted bridge upper bound, with routing as the largest implementation
contributor (Appendix~\ref{app:bridge-empirical}). A theory-faithful
DeMem-Core variant retaining only slot structure, routing, and
certified refinement still reaches 90.8 overall
(Appendix~\ref{app:core-locomo}). Under corrupted LoCoMo feedback and
against feedback-trained or memory-action baselines,
DeMem remains strongest overall
(Appendices~\ref{app:reward-noise} and
\ref{app:feedback-aware-baselines},
Table~\ref{tab:locomo_qwen_feedbackaware}).
Overall, the split-precision, ablation, and DeMem-Core results are consistent with decision-aware state refinement being a key contributor to DeMem’s gains.

\vspace{-0.5em}
\section{Conclusion}
\vspace{-0.5em}
We introduced a decision-centric framework for budgeted agent memory, viewing
runtime memory as a compression bottleneck whose quality is measured by
downstream decision loss. This formulation yields an exact forgetting boundary,
a memory--distortion frontier, and complexity measures for the
decision-relevant distinctions that must be retained.
Building on this formulation, DeMem refines memory only under certified
decision conflict and enjoys regret guarantees that separate unavoidable
compression loss from statistical learning. Empirically, DeMem improves
performance under fixed memory budgets, supporting decision-centric memory
organization.
Overall, agent memory should preserve the distinctions that matter for action,
not merely description. At the same time, DeMem is best viewed as a budgeted decision layer that
complements richer descriptive memory stores; implementation considerations
and broader impacts are discussed in Appendices~\ref{app:limitations}
and~\ref{sec:impact}.

\bibliographystyle{plainnat}
\bibliography{reference}

\begin{thebibliography}{53}
\providecommand{\natexlab}[1]{#1}
\providecommand{\url}[1]{\texttt{#1}}
\expandafter\ifx\csname urlstyle\endcsname\relax
  \providecommand{\doi}[1]{doi: #1}\else
  \providecommand{\doi}{doi: \begingroup \urlstyle{rm}\Url}\fi

\bibitem[Abel et~al.(2017)Abel, Hershkowitz, and Littman]{abel2017nearoptimalbehaviorapproximate}
David Abel, D.~Ellis Hershkowitz, and Michael~L. Littman.
\newblock Near optimal behavior via approximate state abstraction, 2017.
\newblock URL \url{https://arxiv.org/abs/1701.04113}.

\bibitem[Abel et~al.(2018)Abel, Arumugam, Lehnert, and Littman]{abel2018lifelong}
David Abel, Dilip Arumugam, Lucas Lehnert, and Michael Littman.
\newblock State abstractions for lifelong reinforcement learning.
\newblock In \emph{Proceedings of the 35th International Conference on Machine Learning}, 2018.

\bibitem[Chhikara et~al.(2025)Chhikara, Khant, Aryan, Singh, and Yadav]{chhikara2025mem0buildingproductionreadyai}
Prateek Chhikara, Dev Khant, Saket Aryan, Taranjeet Singh, and Deshraj Yadav.
\newblock Mem0: Building production-ready ai agents with scalable long-term memory, 2025.
\newblock URL \url{https://arxiv.org/abs/2504.19413}.

\bibitem[Du(2026)]{du2026memoryautonomousllmagentsmechanisms}
Pengfei Du.
\newblock Memory for autonomous {LLM} agents:mechanisms, evaluation, and emerging frontiers, 2026.
\newblock URL \url{https://arxiv.org/abs/2603.07670}.

\bibitem[Du et~al.(2026)Du, Wang, Xiang, Wang, Huang, XUE, Liang, Zeng, Mi, Bai, Shang, Pan, Jiang, and Wong]{du2026memoryt}
Yiming Du, Baojun Wang, Yifan Xiang, Zhaowei Wang, Wenyu Huang, Boyang XUE, Bin Liang, Xingshan Zeng, Fei Mi, Haoli Bai, Lifeng Shang, Jeff~Z. Pan, Yuxin Jiang, and Kam-Fai Wong.
\newblock Memory-t1: Reinforcement learning for temporal reasoning in multi-session agents.
\newblock In \emph{The Fourteenth International Conference on Learning Representations}, 2026.
\newblock URL \url{https://openreview.net/forum?id=vQf2YR2Kpd}.

\bibitem[Ferns et~al.(2011)Ferns, Panangaden, and Precup]{Bisimulation11Ferns}
Norm Ferns, Prakash Panangaden, and Doina Precup.
\newblock Bisimulation metrics for continuous markov decision processes.
\newblock \emph{SIAM Journal on Computing}, 40\penalty0 (6):\penalty0 1662--1714, 2011.
\newblock \doi{10.1137/10080484X}.
\newblock URL \url{https://doi.org/10.1137/10080484X}.

\bibitem[Gentile et~al.(2014)Gentile, Li, and Zappella]{gentile2014onlineclusteringbandits}
Claudio Gentile, Shuai Li, and Giovanni Zappella.
\newblock Online clustering of bandits, 2014.
\newblock URL \url{https://arxiv.org/abs/1401.8257}.

\bibitem[Givan et~al.(2003)Givan, Dean, and Greig]{GIVAN2003163}
Robert Givan, Thomas Dean, and Matthew Greig.
\newblock Equivalence notions and model minimization in markov decision processes.
\newblock \emph{Artificial Intelligence}, 147\penalty0 (1):\penalty0 163--223, 2003.
\newblock ISSN 0004-3702.
\newblock \doi{https://doi.org/10.1016/S0004-3702(02)00376-4}.
\newblock URL \url{https://www.sciencedirect.com/science/article/pii/S0004370202003764}.
\newblock Planning with Uncertainty and Incomplete Information.

\bibitem[Gutiérrez et~al.(2025)Gutiérrez, Shu, Gu, Yasunaga, and Su]{gutiérrez2025hipporagneurobiologicallyinspiredlongterm}
Bernal~Jiménez Gutiérrez, Yiheng Shu, Yu~Gu, Michihiro Yasunaga, and Yu~Su.
\newblock Hipporag: Neurobiologically inspired long-term memory for large language models, 2025.
\newblock URL \url{https://arxiv.org/abs/2405.14831}.

\bibitem[He et~al.(2026)He, Wang, Zhi, Hu, Chen, Yin, Chen, Wu, Ouyang, Wang, Pei, McAuley, Choi, and Pentland]{he2026memoryarenabenchmarkingagentmemory}
Zexue He, Yu~Wang, Churan Zhi, Yuanzhe Hu, Tzu-Ping Chen, Lang Yin, Ze~Chen, Tong~Arthur Wu, Siru Ouyang, Zihan Wang, Jiaxin Pei, Julian McAuley, Yejin Choi, and Alex Pentland.
\newblock Memoryarena: Benchmarking agent memory in interdependent multi-session agentic tasks, 2026.
\newblock URL \url{https://arxiv.org/abs/2602.16313}.

\bibitem[Hu et~al.(2026{\natexlab{a}})Hu, Wang, and McAuley]{hu2026evaluating}
Yuanzhe Hu, Yu~Wang, and Julian McAuley.
\newblock Evaluating memory in {LLM} agents via incremental multi-turn interactions.
\newblock In \emph{The Fourteenth International Conference on Learning Representations}, 2026{\natexlab{a}}.
\newblock URL \url{https://openreview.net/forum?id=DT7JyQC3MR}.

\bibitem[Hu et~al.(2026{\natexlab{b}})Hu, Liu, Yue, Zhang, Liu, Zhu, Lin, Guo, Dou, Xi, Jin, Tan, Yin, Liu, Zhang, Sun, Zhu, Sun, Peng, Cheng, Fan, Guo, Yu, Zhou, Hu, Huo, Wang, Niu, Wang, Yin, Hu, Liao, Li, Wang, Zhou, Liu, Cheng, Zhang, Gui, Pan, Zhang, Torr, Dou, Wen, Huang, Jiang, and Yan]{hu2026memoryageaiagents}
Yuyang Hu, Shichun Liu, Yanwei Yue, Guibin Zhang, Boyang Liu, Fangyi Zhu, Jiahang Lin, Honglin Guo, Shihan Dou, Zhiheng Xi, Senjie Jin, Jiejun Tan, Yanbin Yin, Jiongnan Liu, Zeyu Zhang, Zhongxiang Sun, Yutao Zhu, Hao Sun, Boci Peng, Zhenrong Cheng, Xuanbo Fan, Jiaxin Guo, Xinlei Yu, Zhenhong Zhou, Zewen Hu, Jiahao Huo, Junhao Wang, Yuwei Niu, Yu~Wang, Zhenfei Yin, Xiaobin Hu, Yue Liao, Qiankun Li, Kun Wang, Wangchunshu Zhou, Yixin Liu, Dawei Cheng, Qi~Zhang, Tao Gui, Shirui Pan, Yan Zhang, Philip Torr, Zhicheng Dou, Ji-Rong Wen, Xuanjing Huang, Yu-Gang Jiang, and Shuicheng Yan.
\newblock Memory in the age of ai agents, 2026{\natexlab{b}}.
\newblock URL \url{https://arxiv.org/abs/2512.13564}.

\bibitem[Huo et~al.(2026)Huo, Lu, Zhang, Chen, and Lin]{huo2026atommemlearnabledynamic}
Yupeng Huo, Yaxi Lu, Zhong Zhang, Haotian Chen, and Yankai Lin.
\newblock Atommem : Learnable dynamic agentic memory with atomic memory operation, 2026.
\newblock URL \url{https://arxiv.org/abs/2601.08323}.

\bibitem[Kang et~al.(2025)Kang, Ji, Zhao, and Bai]{kang-etal-2025-memory}
Jiazheng Kang, Mingming Ji, Zhe Zhao, and Ting Bai.
\newblock Memory {OS} of {AI} agent.
\newblock In Christos Christodoulopoulos, Tanmoy Chakraborty, Carolyn Rose, and Violet Peng, editors, \emph{Proceedings of the 2025 Conference on Empirical Methods in Natural Language Processing}, pages 25961--25970, Suzhou, China, November 2025. Association for Computational Linguistics.
\newblock ISBN 979-8-89176-332-6.
\newblock \doi{10.18653/v1/2025.emnlp-main.1318}.
\newblock URL \url{https://aclanthology.org/2025.emnlp-main.1318/}.

\bibitem[{LangChain AI}(2024)]{langchain_langmem_2024}
{LangChain AI}.
\newblock Langmem: Long-term memory sdk for {LLM} agents, 2024.
\newblock URL \url{https://github.com/langchain-ai/langmem}.

\bibitem[Latimer et~al.(2025)Latimer, Boschi, Neeser, Bartholomew, Srivastava, Wang, and Ramakrishnan]{latimer2025hindsight2020buildingagent}
Chris Latimer, Nicoló Boschi, Andrew Neeser, Chris Bartholomew, Gaurav Srivastava, Xuan Wang, and Naren Ramakrishnan.
\newblock Hindsight is 20/20: Building agent memory that retains, recalls, and reflects, 2025.
\newblock URL \url{https://arxiv.org/abs/2512.12818}.

\bibitem[Lee et~al.(2026)Lee, Nair, Zhang, Lee, Khattab, and Finn]{lee2026metaharnessendtoendoptimizationmodel}
Yoonho Lee, Roshen Nair, Qizheng Zhang, Kangwook Lee, Omar Khattab, and Chelsea Finn.
\newblock Meta-harness: End-to-end optimization of model harnesses, 2026.
\newblock URL \url{https://arxiv.org/abs/2603.28052}.

\bibitem[Lewis et~al.(2021)Lewis, Perez, Piktus, Petroni, Karpukhin, Goyal, Küttler, Lewis, tau Yih, Rocktäschel, Riedel, and Kiela]{lewis2021retrievalaugmentedgenerationknowledgeintensivenlp}
Patrick Lewis, Ethan Perez, Aleksandra Piktus, Fabio Petroni, Vladimir Karpukhin, Naman Goyal, Heinrich Küttler, Mike Lewis, Wen tau Yih, Tim Rocktäschel, Sebastian Riedel, and Douwe Kiela.
\newblock Retrieval-augmented generation for knowledge-intensive nlp tasks, 2021.
\newblock URL \url{https://arxiv.org/abs/2005.11401}.

\bibitem[Li et~al.(2006)Li, Walsh, and Littman]{Li2006TowardsAU}
Lihong Li, Thomas~J. Walsh, and Michael~L. Littman.
\newblock Towards a unified theory of state abstraction for mdps.
\newblock In \emph{AI\&M}, 2006.
\newblock URL \url{https://api.semanticscholar.org/CorpusID:245037}.

\bibitem[Li et~al.(2026)Li, Guo, Zhang, Xu, Huang, Liu, Xu, Xu, and Liu]{li2026locomoplusbeyondfactualcognitivememory}
Yifei Li, Weidong Guo, Lingling Zhang, Rongman Xu, Muye Huang, Hui Liu, Lijiao Xu, Yu~Xu, and Jun Liu.
\newblock Locomo-plus: Beyond-factual cognitive memory evaluation framework for {LLM} agents, 2026.
\newblock URL \url{https://arxiv.org/abs/2602.10715}.

\bibitem[Liu et~al.(2023)Liu, Lin, Hewitt, Paranjape, Bevilacqua, Petroni, and Liang]{liu2023lostmiddlelanguagemodels}
Nelson~F. Liu, Kevin Lin, John Hewitt, Ashwin Paranjape, Michele Bevilacqua, Fabio Petroni, and Percy Liang.
\newblock Lost in the middle: How language models use long contexts, 2023.
\newblock URL \url{https://arxiv.org/abs/2307.03172}.

\bibitem[Maharana et~al.(2024)Maharana, Lee, Tulyakov, Bansal, Barbieri, and Fang]{maharana-etal-2024-evaluating}
Adyasha Maharana, Dong-Ho Lee, Sergey Tulyakov, Mohit Bansal, Francesco Barbieri, and Yuwei Fang.
\newblock Evaluating very long-term conversational memory of {LLM} agents.
\newblock In Lun-Wei Ku, Andre Martins, and Vivek Srikumar, editors, \emph{Proceedings of the 62nd Annual Meeting of the Association for Computational Linguistics (Volume 1: Long Papers)}, pages 13851--13870. Association for Computational Linguistics, August 2024.
\newblock \doi{10.18653/v1/2024.acl-long.747}.
\newblock URL \url{https://aclanthology.org/2024.acl-long.747/}.

\bibitem[Matula and Beck(1983)]{matula1983smallest}
David~W. Matula and Leland~L. Beck.
\newblock Smallest-last ordering and clustering and graph coloring algorithms.
\newblock \emph{J. ACM}, 30\penalty0 (3):\penalty0 417–427, July 1983.
\newblock ISSN 0004-5411.
\newblock \doi{10.1145/2402.322385}.
\newblock URL \url{https://doi.org/10.1145/2402.322385}.

\bibitem[Nan et~al.(2025)Nan, Ma, Wu, and Chen]{nan2025nemoriselforganizingagentmemory}
Jiayan Nan, Wenquan Ma, Wenlong Wu, and Yize Chen.
\newblock Nemori: Self-organizing agent memory inspired by cognitive science, 2025.
\newblock URL \url{https://arxiv.org/abs/2508.03341}.

\bibitem[Nixon(2026)]{nixon2026myhillnerodetheoremboundedinteraction}
Anthony~T. Nixon.
\newblock The myhill-nerode theorem for bounded interaction: Canonical abstractions via agent-bounded indistinguishability, 2026.
\newblock URL \url{https://arxiv.org/abs/2603.21399}.

\bibitem[Pacelli and Majumdar(2019)]{Task-Driven19Pacelli}
Vincent Pacelli and Anirudha Majumdar.
\newblock Task-driven estimation and control via information bottlenecks.
\newblock In \emph{2019 International Conference on Robotics and Automation (ICRA)}, pages 2061--2067, 2019.
\newblock \doi{10.1109/ICRA.2019.8794213}.

\bibitem[Packer et~al.(2024)Packer, Wooders, Lin, Fang, Patil, Stoica, and Gonzalez]{packer2024memgptllmsoperatingsystems}
Charles Packer, Sarah Wooders, Kevin Lin, Vivian Fang, Shishir~G. Patil, Ion Stoica, and Joseph~E. Gonzalez.
\newblock Memgpt: Towards {LLM}s as operating systems, 2024.
\newblock URL \url{https://arxiv.org/abs/2310.08560}.

\bibitem[Pan et~al.(2026)Pan, Zou, Guo, Ni, and Zheng]{pan2026naturallanguageagentharnesses}
Linyue Pan, Lexiao Zou, Shuo Guo, Jingchen Ni, and Hai-Tao Zheng.
\newblock Natural-language agent harnesses, 2026.
\newblock URL \url{https://arxiv.org/abs/2603.25723}.

\bibitem[Rafique and Bindschaedler(2026)]{Rafique2026ClawVMHV}
Mofasshara Rafique and Laurent Bindschaedler.
\newblock Clawvm: Harness-managed virtual memory for stateful tool-using {LLM} agents.
\newblock 2026.
\newblock URL \url{https://api.semanticscholar.org/CorpusID:287432731}.

\bibitem[Rasmussen et~al.(2025)Rasmussen, Paliychuk, Beauvais, Ryan, and Chalef]{rasmussen2025zeptemporalknowledgegraph}
Preston Rasmussen, Pavlo Paliychuk, Travis Beauvais, Jack Ryan, and Daniel Chalef.
\newblock Zep: A temporal knowledge graph architecture for agent memory, 2025.
\newblock URL \url{https://arxiv.org/abs/2501.13956}.

\bibitem[Ravindran and Barto(2004)]{ravindran2004approximate}
Balaraman Ravindran and Andrew~G Barto.
\newblock Approximate homomorphisms: A framework for non-exact minimization in markov decision processes.
\newblock 2004.

\bibitem[Rezaei-Shoshtari et~al.(2022)Rezaei-Shoshtari, Zhao, Panangaden, Meger, and Precup]{rezaeishoshtari2022continuousmdphomomorphismshomomorphic}
Sahand Rezaei-Shoshtari, Rosie Zhao, Prakash Panangaden, David Meger, and Doina Precup.
\newblock Continuous mdp homomorphisms and homomorphic policy gradient, 2022.
\newblock URL \url{https://arxiv.org/abs/2209.07364}.

\bibitem[Shen et~al.(2026)Shen, Li, Zhou, and Hu]{shen2026mem2actbenchbenchmarkevaluatinglongterm}
Yiting Shen, Kun Li, Wei Zhou, and Songlin Hu.
\newblock Mem2actbench: A benchmark for evaluating long-term memory utilization in task-oriented autonomous agents, 2026.
\newblock URL \url{https://arxiv.org/abs/2601.19935}.

\bibitem[Shinn et~al.(2023)Shinn, Cassano, Berman, Gopinath, Narasimhan, and Yao]{shinn2023reflexionlanguageagentsverbal}
Noah Shinn, Federico Cassano, Edward Berman, Ashwin Gopinath, Karthik Narasimhan, and Shunyu Yao.
\newblock Reflexion: Language agents with verbal reinforcement learning, 2023.
\newblock URL \url{https://arxiv.org/abs/2303.11366}.

\bibitem[Shu et~al.(2026)Shu, Jonnalagedda, Gao, Guti{\'e}rrez, Qi, Das, Sun, and Su]{shu2026remem}
Yiheng Shu, Saisri~Padmaja Jonnalagedda, Xiang Gao, Bernal~Jim{\'e}nez Guti{\'e}rrez, Weijian Qi, Kamalika Das, Huan Sun, and Yu~Su.
\newblock {REM}em: Reasoning with episodic memory in language agent.
\newblock In \emph{The Fourteenth International Conference on Learning Representations}, 2026.
\newblock URL \url{https://openreview.net/forum?id=fugnQxbvMm}.

\bibitem[Simão et~al.(2023)Simão, Suilen, and Jansen]{simão2023safepolicyimprovementpomdps}
Thiago~D. Simão, Marnix Suilen, and Nils Jansen.
\newblock Safe policy improvement for pomdps via finite-state controllers, 2023.
\newblock URL \url{https://arxiv.org/abs/2301.04939}.

\bibitem[Sun(2026)]{Sun2026DecisionCentricDF}
Wei Sun.
\newblock Decision-centric design for {LLM} systems.
\newblock 2026.
\newblock URL \url{https://api.semanticscholar.org/CorpusID:287021773}.

\bibitem[Tan et~al.(2025)Tan, Zhang, Ma, Chen, Dai, and Dong]{tan2025membenchcomprehensiveevaluationmemory}
Haoran Tan, Zeyu Zhang, Chen Ma, Xu~Chen, Quanyu Dai, and Zhenhua Dong.
\newblock Membench: Towards more comprehensive evaluation on the memory of {LLM}-based agents, 2025.
\newblock URL \url{https://arxiv.org/abs/2506.21605}.

\bibitem[Tang et~al.(2026)Tang, Yu, Xiao, Wen, Li, Zhou, Wang, Wang, Huang, Deng, Sun, and Zhang]{tang2026mnemisdualrouteretrievalhierarchical}
Zihao Tang, Xin Yu, Ziyu Xiao, Zengxuan Wen, Zelin Li, Jiaxi Zhou, Hualei Wang, Haohua Wang, Haizhen Huang, Weiwei Deng, Feng Sun, and Qi~Zhang.
\newblock Mnemis: Dual-route retrieval on hierarchical graphs for long-term {LLM} memory, 2026.
\newblock URL \url{https://arxiv.org/abs/2602.15313}.

\bibitem[Wang et~al.(2023)Wang, Dong, Cheng, Liu, Yan, Gao, and Wei]{NEURIPS2023_ebd82705}
Weizhi Wang, Li~Dong, Hao Cheng, Xiaodong Liu, Xifeng Yan, Jianfeng Gao, and Furu Wei.
\newblock Augmenting language models with long-term memory.
\newblock In A.~Oh, T.~Naumann, A.~Globerson, K.~Saenko, M.~Hardt, and S.~Levine, editors, \emph{Advances in Neural Information Processing Systems}, volume~36, pages 74530--74543. Curran Associates, Inc., 2023.
\newblock URL \url{https://proceedings.neurips.cc/paper_files/paper/2023/file/ebd82705f44793b6f9ade5a669d0f0bf-Paper-Conference.pdf}.

\bibitem[Wu et~al.(2025)Wu, Wang, Yu, Zhang, Chang, and Yu]{wu2025longmemevalbenchmarkingchatassistants}
Di~Wu, Hongwei Wang, Wenhao Yu, Yuwei Zhang, Kai-Wei Chang, and Dong Yu.
\newblock Longmemeval: Benchmarking chat assistants on long-term interactive memory, 2025.
\newblock URL \url{https://arxiv.org/abs/2410.10813}.

\bibitem[Xu et~al.(2025)Xu, Liang, Mei, Gao, Tan, and Zhang]{xu2025amemagenticmemoryllm}
Wujiang Xu, Zujie Liang, Kai Mei, Hang Gao, Juntao Tan, and Yongfeng Zhang.
\newblock A-mem: Agentic memory for {LLM} agents, 2025.
\newblock URL \url{https://arxiv.org/abs/2502.12110}.

\bibitem[Yao et~al.(2023)Yao, Zhao, Yu, Du, Shafran, Narasimhan, and Cao]{yao2023reactsynergizingreasoningacting}
Shunyu Yao, Jeffrey Zhao, Dian Yu, Nan Du, Izhak Shafran, Karthik Narasimhan, and Yuan Cao.
\newblock React: Synergizing reasoning and acting in language models, 2023.
\newblock URL \url{https://arxiv.org/abs/2210.03629}.

\bibitem[Yu et~al.(2026)Yu, Yao, Xie, Tan, Feng, Li, and Wu]{yu2026agenticmemorylearningunified}
Yi~Yu, Liuyi Yao, Yuexiang Xie, Qingquan Tan, Jiaqi Feng, Yaliang Li, and Libing Wu.
\newblock Agentic memory: Learning unified long-term and short-term memory management for large language model agents, 2026.
\newblock URL \url{https://arxiv.org/abs/2601.01885}.

\bibitem[Zhang et~al.(2025)Zhang, Ren, Zhan, Zhou, Wang, Zhu, Zhou, and Yan]{zhang2025memevolvemetaevolutionagentmemory}
Guibin Zhang, Haotian Ren, Chong Zhan, Zhenhong Zhou, Junhao Wang, He~Zhu, Wangchunshu Zhou, and Shuicheng Yan.
\newblock Memevolve: Meta-evolution of agent memory systems, 2025.
\newblock URL \url{https://arxiv.org/abs/2512.18746}.

\bibitem[Zhang et~al.(2026{\natexlab{a}})Zhang, Fu, and YAN]{zhang2026memgen}
Guibin Zhang, Muxin Fu, and Shuicheng YAN.
\newblock Memgen: Weaving generative latent memory for self-evolving agents.
\newblock In \emph{The Fourteenth International Conference on Learning Representations}, 2026{\natexlab{a}}.
\newblock URL \url{https://openreview.net/forum?id=vI56m4Iu4e}.

\bibitem[Zhang et~al.(2026{\natexlab{b}})Zhang, Long, Bao, Feng, Zhang, Yue, and Wang]{zhang2026memskilllearningevolvingmemory}
Haozhen Zhang, Quanyu Long, Jianzhu Bao, Tao Feng, Weizhi Zhang, Haodong Yue, and Wenya Wang.
\newblock Memskill: Learning and evolving memory skills for self-evolving agents, 2026{\natexlab{b}}.
\newblock URL \url{https://arxiv.org/abs/2602.02474}.

\bibitem[Zhang et~al.(2022)Zhang, Tang, Hao, and Zheng]{zhang2022unifiedpolicyabstractiontheory}
Min Zhang, Hongyao Tang, Jianye Hao, and Yan Zheng.
\newblock Towards a unified policy abstraction theory and representation learning approach in markov decision processes, 2022.
\newblock URL \url{https://arxiv.org/abs/2209.07696}.

\bibitem[Zhang et~al.(2026{\natexlab{c}})Zhang, Shu, Ma, Lin, Wu, and Sang]{zhang2026memoryactionautonomouscontext}
Yuxiang Zhang, Jiangming Shu, Ye~Ma, Xueyuan Lin, Shangxi Wu, and Jitao Sang.
\newblock Memory as action: Autonomous context curation for long-horizon agentic tasks, 2026{\natexlab{c}}.
\newblock URL \url{https://arxiv.org/abs/2510.12635}.

\bibitem[Zhang et~al.(2024)Zhang, Bo, Ma, Li, Chen, Dai, Zhu, Dong, and Wen]{zhang2024surveymemorymechanismlarge}
Zeyu Zhang, Xiaohe Bo, Chen Ma, Rui Li, Xu~Chen, Quanyu Dai, Jieming Zhu, Zhenhua Dong, and Ji-Rong Wen.
\newblock A survey on the memory mechanism of large language model based agents, 2024.
\newblock URL \url{https://arxiv.org/abs/2404.13501}.

\bibitem[Zhou et~al.(2026{\natexlab{a}})Zhou, Chai, Chen, Guo, Shan, Song, Xu, Yang, Yu, Zhang, Zheng, Zhu, Zheng, Zhang, Lou, Zhang, Fu, Wang, Liu, Lin, and Zhang]{zhou2026externalizationllmagentsunified}
Chenyu Zhou, Huacan Chai, Wenteng Chen, Zihan Guo, Rong Shan, Yuanyi Song, Tianyi Xu, Yingxuan Yang, Aofan Yu, Weiming Zhang, Congming Zheng, Jiachen Zhu, Zeyu Zheng, Zhuosheng Zhang, Xingyu Lou, Changwang Zhang, Zhihui Fu, Jun Wang, Weiwen Liu, Jianghao Lin, and Weinan Zhang.
\newblock Externalization in {LLM} agents: A unified review of memory, skills, protocols and harness engineering, 2026{\natexlab{a}}.
\newblock URL \url{https://arxiv.org/abs/2604.08224}.

\bibitem[Zhou and Han(2025)]{zhou2025simplestrongbaselinelongterm}
Sizhe Zhou and Jiawei Han.
\newblock A simple yet strong baseline for long-term conversational memory of {LLM} agents, 2025.
\newblock URL \url{https://arxiv.org/abs/2511.17208}.

\bibitem[Zhou et~al.(2026{\natexlab{b}})Zhou, Qu, Wu, Kim, Prakash, Rus, Low, and Liang]{zhou2026mem}
Zijian Zhou, Ao~Qu, Zhaoxuan Wu, Sunghwan Kim, Alok Prakash, Daniela Rus, Bryan Kian~Hsiang Low, and Paul~Pu Liang.
\newblock {MEM}1: Learning to synergize memory and reasoning for efficient long-horizon agents.
\newblock In \emph{The Fourteenth International Conference on Learning Representations}, 2026{\natexlab{b}}.
\newblock URL \url{https://openreview.net/forum?id=XY8AaxDSLb}.

\end{thebibliography}


\newpage

\appendix
\section{Additional Preliminaries and Technical Reductions}
\label{app:theory}
\subsection{Modeling details for answer-time decision instances}
\label{app:model-details}

The i.i.d.\ assumption in Section~\ref{sec:model} is made over
answer-time decision instances \(X_t=(H_t,Q_t)\), rather than over the
individual turns inside a dialogue. A history \(H_t\) may itself contain
temporally dependent, multi-session interactions. The contextual-bandit
abstraction treats such histories as sampled decision contexts in order to
isolate the runtime memory bottleneck faced by the agent at answer time.

The abstraction is made at the level of answer-time decision instances rather
than individual dialogue turns. Each instance may contain temporally dependent,
multi-session history, while the agent acts using the current query and a
bounded runtime memory state summarizing the relevant past. This isolates the
answer-time memory bottleneck analyzed in the main text.
\subsection{Concentration}
\label{app:concentration}
We repeatedly use the following standard fact:
if $\hat\mu$ is the empirical mean of $n$ i.i.d.\ rewards in $[0,1]$ with mean $\mu$, then for any $\delta\in(0,1)$,
\begin{equation}
\label{eq:hoeffding}
\Pr\!\left(
|\hat\mu-\mu|
\ge
\sqrt{\frac{\log(2/\delta)}{2n}}
\right)
\le \delta.
\end{equation}

\subsection{Deterministic decision rules}
\label{app:deterministic}

The memory-distortion frontier in the main text is defined over
deterministic encoders and deterministic decision rules. This restriction is
intentional. The language-agent implementation uses deterministic decoding:
given a query and a selected memory slot, the answer string is fixed.
Randomness used by the online learner for exploration is part of the
learning protocol and is not included in the static memory-compression
frontier.

\begin{proposition}[Deterministic reduction under deterministic decoding]
\label{prop:deterministic}
Fix a query \(q\), an answering backbone \(F\) with deterministic decoding,
a prompt template \(\mathcal P\), and a \(K\)-slot memory system with slot
contents \(\{S_1,\dots,S_K\}\). For each memory state \(m\), define
\[
a_m(q) \coloneqq F\bigl(\mathcal P(q,S_m)\bigr).
\]
Then the induced answer-time decision rule is a deterministic selector
\(m\mapsto a_m(q)\). Consequently, the distortion objectives in the main
text optimize over deterministic per-state actions.
\end{proposition}

\begin{proof}
Deterministic decoding makes \(F\) a function rather than a stochastic map.
For fixed \(q\) and fixed slot content \(S_m\), the prompt
\(\mathcal P(q,S_m)\) therefore produces a unique answer string
\(a_m(q)\). Once the encoder selects \(m=g(h,q)\), the downstream action is
fixed. Hence the induced policy is deterministic at answer time.
\end{proof}

\begin{remark}
The deterministic frontier matches the answer-time interface used in our agent
implementation. With deterministic decoding, a fixed query and selected memory
slot induce a fixed generated answer. Randomness used for online exploration is
part of the learning protocol, whereas the static memory--distortion frontier
characterizes the deployed deterministic answer-time selector.
\end{remark}

\section{Proofs for the Decision Rate--Distortion Frontier}
\label{app:frontier}
\subsection{Proof of the exact forgetting boundary}
\label{app:proof-forgetting}

\begin{proof}[Proof of Theorem~\ref{thm:forgetting}]
$(1)\!\Rightarrow\!(2)$.\;
Suppose action \(a\in\cA\) satisfies \(\Delta_q(h,a)\le\eps\) for every
\(h\in C\). By Proposition~\ref{prop:deterministic}, a one-state memory
encoder on \(C\) may be paired with the deterministic decision rule that
always plays \(a\). The worst-case distortion over \(C\) is then
\[
\max_{h\in C}\Delta_q(h,a)\le\eps.
\]

$(2)\!\Rightarrow\!(1)$.\;
Given a one-state encoder on \(C\) with worst-case distortion at most \(\eps\),
Proposition~\ref{prop:deterministic} implies the decision rule may be taken
deterministic, selecting some fixed \(a\in\cA\). Then
\[
\max_{h\in C}\Delta_q(h,a)\le\eps,
\]
which is exactly condition~(1).
\end{proof}

\begin{remark}[Equivalence with cluster radius]
\label{rmk:rho-equiv}
Condition~(1) of Theorem~\ref{thm:forgetting} is equivalent to
\(\rho_{\mathrm{dec}}^q(C)\le\eps\), since by definition
\[
\rho_{\mathrm{dec}}^q(C)
=
\min_{a\in\cA}\max_{h\in C}\Delta_q(h,a).
\]
Thus the forgetting boundary admits three equivalent readings:
the existence of a common \(\eps\)-optimal action, the cluster radius bound
\(\rho_{\mathrm{dec}}^q(C)\le\eps\), and the one-state distortion bound.
\end{remark}

\begin{corollary}[Pairwise forgetting boundary]
\label{cor:forgetting-pair}
Two histories \(h,h'\in\cX_q\) can be mapped to the same memory state with
worst-case loss at most \(\eps\) if and only if
\(d_{\mathrm{dec}}^q(h,h')\le\eps\).
\end{corollary}

\begin{proof}
Apply Theorem~\ref{thm:forgetting} to \(C=\{h,h'\}\).
Condition~(1) requires an action \(a\) with
\[
\max\{\Delta_q(h,a),\Delta_q(h',a)\}\le\eps.
\]
Minimizing over \(a\) gives
\[
\min_{a\in\cA}\max\{\Delta_q(h,a),\Delta_q(h',a)\}
=
d_{\mathrm{dec}}^q(h,h')
\le\eps,
\]
where the equality is the definition~\eqref{eq:ddec}.
\end{proof}

\subsection{Proof of the covering upper bound}
\label{app:proof-covering}
\begin{proof}[Proof of Theorem~\ref{thm:cover-pack}, part (i)]
By definition of \(\cN^{\mathrm{dec}}_{\mathrm{cov}}(\eps;q)\), there exists
a partition
\[
\cX_q=\bigsqcup_{m=1}^{K} C_m
\qquad\text{such that}\qquad
\max_m \rho_{\mathrm{dec}}^q(C_m)\le \eps.
\]
For each cluster \(C_m\), choose
\[
a_m\in\arg\min_{a\in\cA}\max_{h\in C_m}\Delta_q(h,a).
\]
Map each \(h\in C_m\) to state \(m\) and play action \(a_m\).
By Theorem~\ref{thm:forgetting}, the worst-case distortion on each cluster
is at most \(\eps\), so the overall worst-case distortion is at most \(\eps\).
Therefore \(\eps^\star_\infty(K;q)\le \eps\).
\end{proof}
\subsection{Proof of the packing lower bound}
\label{app:proof-packing}
\begin{proof}[Proof of Theorem~\ref{thm:cover-pack}, part (ii)]
Let \(P\subseteq\cX_q\) be a set of size \(K+1\) such that
\[
d_{\mathrm{dec}}^q(h,h')>2\eps
\qquad
\text{for all distinct } h,h'\in P.
\]
Any encoder \(g_q:\cX_q\to[K]\) must map two distinct elements
\(h,h'\in P\) to the same memory state. Suppose that state chooses action
\(a\). Then, by definition of \(d_{\mathrm{dec}}^q\),
\[
\max\{\Delta_q(h,a),\Delta_q(h',a)\}
\ge d_{\mathrm{dec}}^q(h,h')
>2\eps.
\]
Hence in particular the worst-case distortion over \(\{h,h'\}\) exceeds
\(\eps\). Therefore no \(K\)-state memory can achieve worst-case distortion
at most \(\eps\), and thus \(\eps^\star_\infty(K;q)>\eps\).
\end{proof}
\subsection{Information lower bound}
\label{app:info-lower-bound}
\begin{proof}[Proof of the information lower bound]
Let \(P\) be a maximum decision packing at scale \(2\eps\) in \(\cX_q\), so
that
\[
|P|=\cN^{\mathrm{dec}}_{\mathrm{pack}}(2\eps;q),
\]
and let \(Z\) be uniform on \(P\). Consider any stochastic memory variable
\(M\) used by a policy whose decision depends only on \(M\). Suppose that
\(M\) supports worst-case distortion at most \(\eps\).

We first show that the conditional supports of \(M\) must separate the
points in \(P\). Indeed, if there exist distinct \(h,h'\in P\) and a memory
state \(m\) such that
\[
\Pr(M=m\mid Z=h)>0
\quad\text{and}\quad
\Pr(M=m\mid Z=h')>0,
\]
then the policy must take the same action \(a(m)\) after observing \(m\) for
both \(h\) and \(h'\). Since \(P\) is a decision packing at scale \(2\eps\),
we have
\[
d_{\mathrm{dec}}^q(h,h')>2\eps.
\]
Therefore, for every action \(a\),
\[
\max\{\Delta_q(h,a),\Delta_q(h',a)\}>2\eps>\eps.
\]
In particular this holds for \(a=a(m)\), contradicting the assumed
worst-case distortion bound.

Hence the realized memory state \(M\) identifies the history \(Z\) within
\(P\): there exists a decoder \(\phi\) such that
\[
Z=\phi(M)
\quad\text{almost surely}.
\]
Thus \(H(Z\mid M)=0\). Since \(Z\) is uniform on \(P\),
\[
I(Z;M)
=
H(Z)-H(Z\mid M)
=
H(Z)
=
\log |P|
=
\log \cN^{\mathrm{dec}}_{\mathrm{pack}}(2\eps;q).
\]
\end{proof}

\subsection{Proof of computational hardness}
\label{app:proof-hardness}

\begin{proof}[Proof of Theorem~\ref{thm:np-hard}]
We reduce from \textsc{Set Cover}. An instance of \textsc{Set Cover}
consists of a universe
\[
\mathcal U=\{u_1,\dots,u_N\},
\]
a collection of subsets
\[
\mathcal S=\{S_1,\dots,S_A\}, \qquad S_j\subseteq \mathcal U,
\]
and an integer \(K\). The question is whether there exist at most \(K\)
sets whose union covers \(\mathcal U\). We assume without loss of
generality that every element belongs to at least one set; otherwise the
instance is immediately infeasible.

Construct a memory instance as follows. The context set is the universe,
\[
\cX=\mathcal U,
\]
and the action set contains one action \(a_j\) for each set \(S_j\). Define
the reward matrix by
\[
\mu(u_i,a_j)
=
\begin{cases}
1, & u_i\in S_j,\\
0, & u_i\notin S_j.
\end{cases}
\]
Since every \(u_i\) belongs to at least one set, we have
\(\mu^\star(u_i)=1\) for all \(u_i\). Hence the suboptimality gap is
\[
\Delta(u_i,a_j)
=
\mu^\star(u_i)-\mu(u_i,a_j)
=
\begin{cases}
0, & u_i\in S_j,\\
1, & u_i\notin S_j.
\end{cases}
\]

Fix any rational \(0\le \eps<1\). For any nonempty cluster
\(C\subseteq\cX\), its decision radius is
\[
\rho_{\mathrm{dec}}(C)
=
\min_{a_j\in\cA}\max_{u_i\in C}\Delta(u_i,a_j).
\]
Because all gaps are either \(0\) or \(1\), and \(\eps<1\), we have
\[
\rho_{\mathrm{dec}}(C)\le \eps
\quad\Longleftrightarrow\quad
\exists j\in[A]\ \text{such that}\ C\subseteq S_j .
\]
Thus a cluster has distortion at most \(\eps\) exactly when all of its
contexts are covered by a single set action.

We now prove equivalence.

First, suppose the \textsc{Set Cover} instance has a cover
\(S_{j_1},\dots,S_{j_K}\). Assign each element \(u_i\) to one selected set
that contains it, and let \(C_\ell\) be the elements assigned to
\(S_{j_\ell}\). Then each \(C_\ell\subseteq S_{j_\ell}\), so choosing action
\(a_{j_\ell}\) for cluster \(C_\ell\) gives
\(\rho_{\mathrm{dec}}(C_\ell)=0\le\eps\). Hence
\(\eps^\star_\infty(K)\le\eps\).

Conversely, suppose \(\eps^\star_\infty(K)\le\eps\). Then there exists a
partition
\[
\cX=\bigsqcup_{m=1}^{K} C_m
\]
and actions \(a_{j_m}\) such that
\[
\max_{u_i\in C_m}\Delta(u_i,a_{j_m})\le\eps
\qquad
\text{for every }m.
\]
Since \(\eps<1\) and the gaps are binary, this implies
\(\Delta(u_i,a_{j_m})=0\) for every \(u_i\in C_m\), hence
\(C_m\subseteq S_{j_m}\). Therefore the sets
\[
S_{j_1},\dots,S_{j_K}
\]
cover all elements of \(\mathcal U\). Thus the original \textsc{Set Cover}
instance is feasible.

Therefore, deciding whether \(\eps^\star_\infty(K)\le\eps\) is NP-hard for
every rational \(0\le\eps<1\). Membership in NP is immediate: a certificate
consists of a partition \(\cX=\bigsqcup_{m=1}^K C_m\) and one action
\(a_m\) per cluster, and verification only requires checking
\[
\max_{x\in C_m}\Delta(x,a_m)\le\eps
\]
for all \(m\), which is polynomial in \(N\), \(A\), and \(K\). Hence the
problem is NP-complete.
\end{proof}

\begin{remark}
\label{rem:hardness-implications}
The reduction is stated directly in terms of cluster decision radius rather
than pairwise decision distance. This is important because pairwise
compatibility need not compose into cluster compatibility. The hardness
result therefore applies to the true optimal memory-partition problem:
even with complete knowledge of \(\mu\), finding a \(K\)-state partition
whose clusters each admit a shared low-distortion action is NP-hard.
Consequently, polynomial-time memory learners such as DeMem must
use tractable relaxations and are evaluated by the realized cluster-level
distortion of the partitions they return.
\end{remark}

\subsection{Average distortion and distribution-dependent floor}
\label{app:avg-distortion}

The main text focuses on the worst-case frontier
\(\eps^\star_\infty(K;q)\). This appendix records the corresponding
average-distortion analogue, which is useful when the conditional history
distribution is non-uniform. Fix a query \(q\), and let \(\cD_q\) denote the
distribution over histories in \(\cX_q\).

For an encoder--policy pair \((g_q,\pi_q)\), define the average decision
distortion
\begin{equation}
\label{eq:avg-distortion-app}
D_q(g_q,\pi_q)
\;\coloneqq\;
\EE_{H\sim\cD_q}
\bigl[d_q(H,g_q(H);\pi_q)\bigr].
\end{equation}
The corresponding average-case frontier is
\begin{equation}
\label{eq:avg-frontier-app}
\eps^\star(K;q)
\;\coloneqq\;
\inf_{g_q:\cX_q\to[K]}
\inf_{\pi_q:[K]\to\cA}
D_q(g_q,\pi_q).
\end{equation}
When \(\cD\) is non-uniform, \(\eps^\star(K)\) can be substantially smaller
than the worst-case frontier \(\eps^\star_\infty(K)\). The following
proposition, referenced in Section~\ref{sec:limits}, makes this precise.

\begin{proposition}[Distribution-dependent distortion floor]
\label{prop:avg-lb}
For any encoder \(g_q:\cX_q\to[K]\) and any decision rule
\(\pi_q:[K]\to\cA\),
\begin{equation}
\label{eq:avg-lb}
D_q(g_q,\pi_q)
\;\ge\;
\sum_{m=1}^{K}\Pr_{\cD_q}\!\bigl(g_q(H)=m\bigr)\;
\rho_{\mathrm{dec}}^{\mathrm{avg},q}\!\bigl(g_q^{-1}(m)\bigr),
\end{equation}
where
\[
\rho_{\mathrm{dec}}^{\mathrm{avg},q}(C)
\coloneqq
\min_{a\in\cA}\,\EE[\Delta_q(Z,a)\mid Z\in C]
\]
is the average cluster decision radius for query \(q\).
\end{proposition}

\begin{proof}
Fix any encoder \(g_q:\cX_q\to[K]\) and any decision rule
\(\pi_q:[K]\to\cA\). By Proposition~\ref{prop:deterministic}, we may assume
without loss of generality that \(\pi_q(m)\) selects a single action
\(a_m\in\cA\) for each \(m\in[K]\).

Decompose the average distortion by conditioning on the memory state:
\begin{align}
D_q(g_q,\pi_q)
&= \EE_{Z\sim\cD_q}
\bigl[\mu_q^\star(Z)-\mu_q(Z,a_{g_q(Z)})\bigr] \notag\\
&= \sum_{m=1}^{K}\Pr_{\cD_q}\bigl(g_q(Z)=m\bigr)\;
   \EE\bigl[\Delta_q(Z,a_m)\mid g_q(Z)=m\bigr].
   \label{eq:avg-decomp}
\end{align}

For each cluster \(m\), the action \(a_m\) is a specific element of \(\cA\).
Therefore
\begin{equation}
\label{eq:cluster-lb}
\EE\bigl[\Delta_q(Z,a_m)\mid g_q(Z)=m\bigr]
\;\ge\;
\min_{a\in\cA}\;\EE\bigl[\Delta_q(Z,a)\mid g_q(Z)=m\bigr]
\;=\;
\rho_{\mathrm{dec}}^{\mathrm{avg},q}\!\bigl(g_q^{-1}(m)\bigr).
\end{equation}

Substituting \eqref{eq:cluster-lb} into \eqref{eq:avg-decomp} yields
\eqref{eq:avg-lb}. Since this holds for every \((g_q,\pi_q)\), taking the
infimum over both gives the stated bound.
\end{proof}

\begin{remark}[Relationship to worst-case bounds]
\label{rem:avg-vs-worst}
The average cluster radius \(\rho_{\mathrm{dec}}^{\mathrm{avg},q}(C)\) and
the worst-case cluster radius \(\rho_{\mathrm{dec}}^q(C)\) are related by
\[
\rho_{\mathrm{dec}}^{\mathrm{avg},q}(C)
\;\le\;
\rho_{\mathrm{dec}}^q(C)
\]
for any \(C\subseteq\cX_q\) and any distribution \(\cD_q\).
Equality holds when \(\cD_q\) is supported on the worst-case history within
each cluster. In general, the gap can be substantial: if a cluster contains
one high-gap history with negligible probability, the worst-case radius is
large but the average radius remains small. This is precisely the regime in
which Proposition~\ref{prop:avg-lb} provides a tighter characterization than
\(\eps^\star_\infty(K;q)\) alone.

To make this concrete, consider \(C=\{h_1,h_2\}\) with
\(\Pr(Z=h_1)=1-\delta\) and \(\Pr(Z=h_2)=\delta\) for small \(\delta>0\).
Suppose \(\rho_{\mathrm{dec}}^q(\{h_1\})=0\), but
\(\rho_{\mathrm{dec}}^q(\{h_1,h_2\})=1\). Then
\(\rho_{\mathrm{dec}}^q(C)=1\), but
\(\rho_{\mathrm{dec}}^{\mathrm{avg},q}(C)\le\delta\): the average-case floor
is \(\delta\) while the worst-case floor is \(1\).
\end{remark}

\section{DeMem: Full Construction and Analysis}
\label{app:DeMem}

\subsection{Confidence bounds and distance certificates}
\label{app:certificates}

Let \(N\coloneqq|\cX|\). For each observed context--action pair
\((x,a)\), let \(n_t(x,a)\) be the pull count and \(\hat\mu_t(x,a)\)
the empirical mean. For \(n_t(x,a)>0\), define
\begin{equation}
\label{eq:conf-radius-app}
c_t(x,a)
\coloneqq
\sqrt{
\frac{\log(4NAt^2/\delta)}
{2\,n_t(x,a)}
}.
\end{equation}
For \(n_t(x,a)=0\), we set the interval to the trivial bound
\([0,1]\).

Let
\[
\begin{aligned}
\UCB_t(x,a)
&\coloneqq
\begin{cases}
\min\{1,\hat\mu_t(x,a)+c_t(x,a)\}, & n_t(x,a)>0,\\
1, & n_t(x,a)=0,
\end{cases}\\[0.5em]
\LCB_t(x,a)
&\coloneqq
\begin{cases}
\max\{0,\hat\mu_t(x,a)-c_t(x,a)\}, & n_t(x,a)>0,\\
0, & n_t(x,a)=0.
\end{cases}
\end{aligned}
\]
and define
\[
\UCB_t^\star(x)\coloneqq \max_{a\in\cA}\UCB_t(x,a),
\qquad
\LCB_t^\star(x)\coloneqq \max_{a\in\cA}\LCB_t(x,a).
\]
These induce gap bounds
\[
\overline\Delta_t(x,a)\coloneqq \UCB_t^\star(x)-\LCB_t(x,a),
\qquad
\underline\Delta_t(x,a)\coloneqq \LCB_t^\star(x)-\UCB_t(x,a).
\]
Define corresponding upper and lower certificates for decision distance:
\begin{equation}
\label{eq:ddec-bounds}
\overline d_t(x,x')
\coloneqq
\min_{a\in\cA}\max\{\overline\Delta_t(x,a),\overline\Delta_t(x',a)\},
\qquad
\underline d_t(x,x')
\coloneqq
\min_{a\in\cA}\max\{\underline\Delta_t(x,a),\underline\Delta_t(x',a)\}.
\end{equation}
For any nonempty cluster \(C\) of observed contexts, define the
corresponding cluster-radius certificates
\begin{equation}
\label{eq:cluster-radius-bounds}
\overline\rho_t(C)
\coloneqq
\min_{a\in\cA}\max_{x\in C}\overline\Delta_t(x,a),
\qquad
\underline\rho_t(C)
\coloneqq
\min_{a\in\cA}\max_{x\in C}\underline\Delta_t(x,a).
\end{equation}

\begin{lemma}[Validity of the certificates]
\label{lem:cert}
With probability at least \(1-\delta\), for all observed contexts, all
sampled pairs \((x,x')\), all nonempty observed clusters \(C\), and all
times \(t\),
\[
\underline d_t(x,x') \le d_{\mathrm{dec}}(x,x') \le \overline d_t(x,x'),
\qquad
\underline\rho_t(C) \le \rho_{\mathrm{dec}}(C) \le \overline\rho_t(C).
\]
\end{lemma}

\begin{proof}
For any fixed \(x\in\cX\), \(a\in\cA\), and \(t\ge1\) with
\(n_t(x,a)>0\), Hoeffding's inequality gives
\[
\Pr\!\left(
|\hat\mu_t(x,a)-\mu(x,a)|>c_t(x,a)
\right)
\le
2\exp\!\left(-2n_t(x,a)c_t(x,a)^2\right)
=
\frac{\delta}{2NAt^2}.
\]
Taking a union bound over \(N\) contexts, \(A\) actions, and all
\(t\ge1\), the total failure probability is at most
\[
\sum_{t=1}^{\infty} NA\cdot \frac{\delta}{2NAt^2}
=
\frac{\delta}{2}\sum_{t=1}^{\infty}\frac{1}{t^2}
<\delta.
\]
Therefore, with probability at least \(1-\delta\),
\[
\LCB_t(x,a)\le \mu(x,a)\le \UCB_t(x,a)
\qquad
\text{for all }x\in\cX,\;a\in\cA,\text{ and all }t.
\]
This implies
\[
\underline\Delta_t(x,a)\le \Delta(x,a)\le \overline\Delta_t(x,a)
\]
uniformly. Taking \(\min_a\max\{\cdot,\cdot\}\) gives the pairwise
distance bounds, and taking \(\min_a\max_{x\in C}\) gives the
cluster-radius bounds.
\end{proof}

\subsection{Frozen protocol for choosing the distortion level}
\label{app:frozen-protocol}

\paragraph{Cannot-link graph.}
At the beginning of epoch \(e\), with start time \(t_e=2^{e-1}\), define a
cannot-link graph \(G_\alpha=(V_e,E_\alpha)\) over the observed contexts
\(V_e=\{X_1,\dots,X_{t_e-1}\}\): for a candidate pairwise level \(\alpha\),
add an edge \((x,x')\) whenever
\[
\underline d_{t_e}(x,x')>\alpha .
\]
If a cluster \(C\) has true decision radius
\(\rho_{\mathrm{dec}}(C)\le \alpha\), then every pair in \(C\) must satisfy
\(d_{\mathrm{dec}}(x,x')\le \alpha\), and hence \(C\) contains no certified
cannot-link edge on the confidence event. Thus graph independence is a
necessary condition for small cluster radius.

The converse is not generally true: because \(d_{\mathrm{dec}}\) need not
satisfy the triangle inequality, pairwise compatibility does not imply the
existence of a single shared action for the whole cluster. We therefore use
the graph only as a conservative splitting constraint. The actual
compression price of a returned partition \(\mathcal P\) is measured by the
cluster-level certificate
\begin{equation}
\label{eq:certified-cluster-distortion}
\overline\eps_{t_e}(\mathcal P)
\coloneqq
\max_{C\in\mathcal P}\overline\rho_{t_e}(C),
\end{equation}
which upper-bounds
\(\max_{C\in\mathcal P}\rho_{\mathrm{dec}}(C)\) with high probability.

For the graph relaxation, define
\begin{equation}
\label{eq:graph-relax-threshold}
\alpha^\star_{\mathrm{graph}}(K)
\coloneqq
\inf\bigl\{
\alpha\ge 0:\;\chi(G_\alpha)\le K
\bigr\}.
\end{equation}
This is a pairwise graph threshold, not the decision-distortion frontier
\(\eps^\star_\infty(K)\).
By Theorem~\ref{thm:np-hard}, computing \(\chi(G_\alpha)\) is NP-hard, so we
replace it with a polynomial-time greedy coloring.

\paragraph{Greedy partition subroutine.}

\begin{algorithm}[t]
\caption{\textsc{Greedy-Partition}: binary-search epoch partition}
\label{alg:greedy-partition}
\begin{algorithmic}[1]
\Require pairwise cannot-link certificates
  \(\{\underline d_{t_e}(x,x')\}\), cluster-radius certificates
  \(\{\overline\rho_{t_e}(C)\}\), budget \(K\)
\State Let \(V_e\) be the set of observed contexts at epoch \(e\)
\State Sort candidate graph levels
\[
\mathcal E
\gets
\{0\}\cup
\{\underline d_{t_e}(x,x'):\;x,x'\in V_e,\;x\neq x'\}
\]
in increasing order
\State \(\ell\gets 1,\; r\gets |\mathcal E|,\; i^\star\gets |\mathcal E|\)
\While{\(\ell\le r\)}
    \State \(i\gets \lfloor(\ell+r)/2\rfloor\), \(\alpha\gets \mathcal E_i\)
    \State Construct \(G_\alpha=(V_e,E_\alpha)\) with edges
    \[
    (x,x')\in E_\alpha
    \quad\Longleftrightarrow\quad
    \underline d_{t_e}(x,x')>\alpha
    \]
    \State Compute the degeneracy \(d_\alpha=\mathrm{degen}(G_\alpha)\)
    \If{\(d_\alpha+1\le K\)}
        \State \(i^\star\gets i\)
        \State \(r\gets i-1\)
    \Else
        \State \(\ell\gets i+1\)
    \EndIf
\EndWhile
\State \(\alpha_e\gets \mathcal E_{i^\star}\)
\State Construct \(G_{\alpha_e}=(V_e,E_{\alpha_e})\)
\State Compute a degeneracy ordering of \(G_{\alpha_e}\)
\State Greedily color \(G_{\alpha_e}\) along this ordering, assigning each
vertex the smallest-index color not used by its already-colored neighbors
\State Let \(\mathcal P_e=\{C_1,\dots,C_K\}\) be the induced partition,
adding empty clusters if fewer than \(K\) colors are used
\State Set
\[
\eps^{\mathrm{cert}}_e
\gets
\max_{C\in\mathcal P_e}\overline\rho_{t_e}(C)
\]
\State \Return partition \(\mathcal P_e\), graph level \(\alpha_e\),
certified distortion \(\eps^{\mathrm{cert}}_e\)
\end{algorithmic}
\end{algorithm}

Algorithm~\ref{alg:greedy-partition} sorts the candidate pairwise graph
levels and then binary-searches over them using the monotone feasibility
condition
\[
\mathrm{degen}(G_\alpha)+1\le K .
\]
This condition is monotone in \(\alpha\): increasing \(\alpha\) removes
cannot-link edges and therefore cannot increase graph degeneracy. Once the
smallest feasible level \(\alpha_e\) is found, the algorithm greedily colors
\(G_{\alpha_e}\) using a degeneracy ordering. Since degeneracy-order greedy
coloring uses at most \(\mathrm{degen}(G_{\alpha_e})+1\le K\) colors, the
returned partition satisfies the \(K\)-slot budget.

The candidate set has size \(O(|V_e|^2)\), and sorting it costs
\(O(|V_e|^2\log |V_e|)\). Each binary-search step constructs \(G_\alpha\)
and computes its degeneracy in \(O(|V_e|^2)\) time, with
\(O(\log |V_e|)\) steps. Evaluating the certified cluster radii of the
returned partition costs \(O(KA|V_e|)\). Thus the per-epoch cost is
\[
O\!\left(|V_e|^2\log |V_e|+KA|V_e|\right).
\]

\paragraph{Structural guarantee.}
Recall that the \emph{degeneracy} $\degen(G)$ of a graph $G$ is the
smallest $d$ such that every subgraph of $G$ has a vertex of degree at
most~$d$.
Greedy coloring with smallest-last ordering uses at most
$\degen(G)+1$ colors \citep[see, e.g.,][]{matula1983smallest}.
This yields the following gap characterization.

\begin{proposition}[Greedy graph-relaxation gap]
\label{prop:greedy-gap}
Let \(\alpha^\star=\alpha^\star_{\mathrm{graph}}(K)\) and
\(G^\star=G_{\alpha^\star}\). Let \(\alpha_e\) be the graph level returned by
Algorithm~\ref{alg:greedy-partition}. Then:
\begin{enumerate}[label=(\roman*),nosep]
\item \(\alpha_e\ge \alpha^\star_{\mathrm{graph}}(K)\) always.
\item \(\alpha_e=\alpha^\star_{\mathrm{graph}}(K)\) whenever
  \(\degen(G^\star)+1\le K\).
\item In general,
  \begin{equation}
  \label{eq:greedy-gap}
  \alpha_e
  \;\le\;
  \inf\bigl\{
    \alpha\ge 0:\;\degen(G_\alpha)+1\le K
  \bigr\}.
  \end{equation}
\end{enumerate}
These statements concern the pairwise graph relaxation only. The
cluster-level compression price used in Theorem~\ref{thm:upper} is
\(\eps^{\mathrm{cert}}_e\).
\end{proposition}

\begin{proof}
(i)~The greedy coloring can only return at a level \(\alpha\) for which it
uses at most \(K\) colors. This implies that \(G_\alpha\) is \(K\)-colorable,
so \(\alpha\ge \alpha^\star_{\mathrm{graph}}(K)\).

(ii)~If \(\degen(G^\star)+1\le K\), then the smallest-last greedy coloring
of \(G^\star\) uses at most \(\degen(G^\star)+1\le K\) colors. Hence
Algorithm~\ref{alg:greedy-partition} returns at level \(\alpha^\star\) or
earlier. Combining this with part~(i) gives equality.

(iii)~For any \(\alpha\) satisfying \(\degen(G_\alpha)+1\le K\),
smallest-last greedy coloring uses at most \(K\) colors, so the algorithm
must return no later than that \(\alpha\). Taking the infimum over such
levels yields the claim.
\end{proof}

\begin{remark}[Sparsity in practice]
\label{rem:sparsity}
In agent-memory applications, the cannot-link graph is typically sparse:
most context pairs do not trigger certified incompatibility. On LoCoMo,
only \(4.6\%\) of routing events trigger a certified split
(Table~\ref{tab:split-freq}), implying low edge density and hence low
degeneracy in the graph relaxation. Appendix~\ref{app:greedy-empirical}
reports this graph-level behavior. Separately, the regret analysis and
empirical distortion measurements use the cluster-level certificate
\(\eps^{\mathrm{cert}}_e\), so they do not rely on pairwise compatibility
composing into cluster compatibility.
\end{remark}

\clearpage
\subsection{Algorithm}
\label{app:DeMem-algorithm}

\begin{algorithm}[p]
\caption{DeMem: Decision-aware Memory with Certified Splits}
\label{alg:DeMem}
For contexts observed before epoch \(e\), the encoder \(g^{(e)}\) is the
cluster lookup induced by \(\mathcal P_e\). For a context first observed
inside the epoch, DeMem extends \(g^{(e)}\) using a deterministic
fallback router \(\phi_e\), which assigns the context to the cluster that
minimizes the certified radius after insertion. This temporary assignment
is kept fixed until the next epoch rebuilds the cannot-link graph and the
\(K\)-slot partition.
\begin{algorithmic}[1]
\Require memory budget \(K\), confidence \(\delta\), certification resolution \(\gamma>0\)
\State Initialize an empty context table and context-level statistics
  \((n(x,a),\hat\mu(x,a))\)
\State Initialize a single cluster \(\cC^{(1)}=\{C_1\}\) with
  \(C_1\gets\emptyset\)
\For{epoch \(e=1,2,\dots\)}
    \State \(t_e\gets 2^{e-1}\)
    \State Using data up to \(t_e-1\), compute
      \(\underline d_{t_e}(x,x')\) for all observed pairs and
      \(\overline\rho_{t_e}(C)\) for candidate clusters
    \State \((\mathcal P_e,\,\alpha_e,\,\eps^{\mathrm{cert}}_e)\gets
      \textsc{Greedy-Partition}\bigl(
        \{\underline d_{t_e}(x,x')\},\{\overline\rho_{t_e}(C)\},K
      \bigr)\)
      \Comment{Algorithm~\ref{alg:greedy-partition}}
    \State Construct epoch clusters \(C^{(e)}_1,\dots,C^{(e)}_K\) from
      \(\mathcal P_e\), and define \(g^{(e)}(x)\) as the cluster index for
      every previously observed context \(x\)
    \State Define a deterministic fallback router for out-of-sample contexts:
    \[
    \phi_e(x)
    \in
    \arg\min_{j\in[K]}
    \overline\rho_t\!\left(C^{(e)}_j\cup\{x\}\right),
    \]
    with deterministic tie-breaking
    \State Reset cluster-level statistics \((n_{m,a},\hat\mu_{m,a})\)
    \For{\(t=t_e,\dots,2^e-1\)}
        \State Observe context \(x\gets X_t\)
        \If{\(x\) is new}
            \State Add \(x\) to the context table and initialize
              \((n(x,a),\hat\mu(x,a))\)
            \State \(m\gets \phi_e(x)\)
            \State Add \(x\) to \(C^{(e)}_m\) and set \(g^{(e)}(x)\gets m\)
        \EndIf
        \State Let \(m\gets g^{(e)}(x)\)
        \State Set
        \[
        B_t(\gamma)
        \coloneqq
        \left\lceil
        \frac{8\log(4NAt^2/\delta)}{\gamma^2}
        \right\rceil .
        \]
        \If{there exists \(a\in\cA\) with \(n(x,a)<B_t(\gamma)\)}
            \State Play a certification-exploration action
            \[
            A_t\in\arg\min_{a\in\cA} n(x,a).
            \]
        \Else
            \State Play the cluster-level UCB action
            \[
            A_t\in\arg\max_{a\in\cA}
            \left(
            \hat\mu_{m,a}
            +
            \sqrt{
            \frac{2\log(4AKt^2/\delta)}
            {\max\{1,n_{m,a}\}}
            }
            \right).
            \]
        \EndIf
        \State Observe reward \(R_t\)
        \State Update cluster-level statistics
          \((n_{m,A_t},\hat\mu_{m,A_t})\)
        \State Update context-level statistics
          \((n(x,A_t),\hat\mu(x,A_t))\)
    \EndFor
\EndFor
\end{algorithmic}
\end{algorithm}

\subsection{Certification exploration}
\label{app:cert-exploration}
\paragraph{Quantities used in Theorem~\ref{thm:upper}.}
For the partition \(\mathcal P_e\) used in epoch \(e\), define
\[
\eps^{\mathrm{cert}}_e
\coloneqq
\max_{C\in\mathcal P_e}\overline\rho_{t_e}(C),
\qquad
\bar\eps^{\mathrm{cert}}_T
\coloneqq
\max_{e\le\lceil\log_2 T\rceil}\eps^{\mathrm{cert}}_e .
\]
The terminal certification threshold is
\[
B_T(\gamma)
\coloneqq
\left\lceil
\frac{8\log(4NAT^2/\delta)}{\gamma^2}
\right\rceil .
\]
The cannot-link certificates in Section~\ref{sec:certificates} are computed
from context-level statistics, whereas exploitation is performed using
cluster-level UCB statistics. Without explicit context-level exploration, a
context routed to an already confident cluster may stop receiving enough
action-level samples for its own confidence intervals to shrink. In that
case, a genuine decision conflict inside the cluster could remain
uncertified.

For this reason, DeMem samples each observed context--action pair
until
\[
n_t(x,a)\ge
B_t(\gamma)
\coloneqq
\left\lceil
\frac{8\log(4NAt^2/\delta)}{\gamma^2}
\right\rceil .
\]
After this coverage condition holds, the context-level confidence radius is
\(O(\gamma)\), and the algorithm returns to cluster-level UCB. Thus
pairwise conflicts separated by a positive margin larger than \(\gamma\)
can be certified once the relevant contexts have been observed often enough.
Zero-margin conflicts, or conflicts involving rarely observed contexts, are
not forced to be certified; their effect is instead captured by the
realized cluster distortion term in Theorem~\ref{thm:upper}.

\subsection{Proof sketch of the upper bound}
\label{app:upper-proof}

\begin{proof}[Proof sketch of Theorem~\ref{thm:upper}]
We decompose regret into three parts.

\paragraph{Realized compression price.}
At each epoch \(e\), the frozen partition \(\mathcal P_e\) induces a
\(K\)-state encoder \(g^{(e)}\). On the confidence event of
Lemma~\ref{lem:cert}, its true cluster distortion is bounded by
\[
\max_{C\in\mathcal P_e}\rho_{\mathrm{dec}}(C)
\le
\max_{C\in\mathcal P_e}\overline\rho_{t_e}(C)
=
\eps^{\mathrm{cert}}_e .
\]
Comparing against the best action available under \(g^{(e)}\) in each
epoch, the compression cost over the full horizon is at most
\[
\sum_e T_e\,\eps^{\mathrm{cert}}_e
\le
T\,\bar\eps^{\mathrm{cert}}_T .
\]

\paragraph{Statistical learning within epochs.}
Condition on an epoch in which the encoder is frozen and ignore rounds used
for certification exploration. On the remaining rounds, the learner faces a
standard bandit problem over \(K\) memory states and \(A\) actions. A
standard UCB analysis gives \(\tilde O(\sqrt{AKT_e})\) regret over an epoch
of length \(T_e\), and summing over epochs yields
\(\tilde O(\sqrt{AKT})\).

\paragraph{Context-level certification exploration.}
The certification branch ensures that, for every observed context \(x\) and
action \(a\), the algorithm samples \((x,a)\) until
\(n_t(x,a)\ge B_t(\gamma)\). Hence each distinct context contributes at most
\(A B_T(\gamma)\) such exploration rounds over the horizon, and the total
regret from certification exploration is at most
\[
O\!\bigl(A N_T B_T(\gamma)\bigr)
=
\tilde O(A N_T/\gamma^2).
\]
After this coverage condition holds, the context-level confidence radius is
at most \(O(\gamma)\). Therefore any pairwise conflict separated by margin
larger than \(\gamma\) can be certified once the corresponding contexts have
both reached coverage. Conflicts with vanishing margin are not guaranteed to
be certified, but they are charged through the realized cluster distortion
term rather than through an invalid pairwise-to-cluster implication.

Combining the three terms proves the stated bound.
\end{proof}
\begin{remark}[Optimality up to realized compression]
\label{rem:recover-floor}
Theorem~\ref{thm:lower} shows that any \(K\)-state algorithm incurs
\[
\Omega\!\left(T\eps^\star_\infty(K)+\sqrt{AKT}\right).
\]
Theorem~\ref{thm:upper} matches the minimax statistical learning term up to
logarithmic factors. The remaining difference is the excess realized
compression price
\[
\bar\eps^{\mathrm{cert}}_T-\eps^\star_\infty(K),
\]
together with the context-certification cost
\(\tilde O(A N_T/\gamma^2)\). Thus DeMem is near-minimax whenever
the returned partitions have certified cluster distortion close to the
optimal decision-distortion frontier and the number of newly observed
contexts is not the dominant term.
\end{remark}

\subsection{Proof sketch of the lower bound}
\label{app:lower-proof}

\begin{proof}[Proof sketch of Theorem~\ref{thm:lower}]
The lower bound has two orthogonal sources.

\paragraph{Compression term.}
The runtime memory constraint alone implies that no deterministic
\(K\)-state answer-time policy can beat
\(\eps^\star_\infty(K)\) in worst-case distortion. Therefore the regret
must contain a term of order \(T\eps^\star_\infty(K)\).

\paragraph{Statistical term.}
Consider an instance with \(\eps^\star_\infty(K)=0\) in which
\(\cX=[K]\) and each context induces an independent \(A\)-armed bandit
subproblem. Each subproblem appears \(T/K\) times in expectation, so the
classical minimax lower bound gives
\(\Omega(\sqrt{A(T/K)})\) regret per state. Summing over the \(K\) states
yields \(\Omega(\sqrt{AKT})\).

A direct-sum construction combines the two terms, giving
\[
\Omega\!\left(T\eps^\star_\infty(K)+\sqrt{AKT}\right).
\]
\end{proof}

\subsection{Extension to non-stationary rewards}
\label{app:nonstationary}

The regret guarantees of Section~\ref{sec:regret} assume a fixed reward
function~$\mu$.
In many agent settings the decision-relevant structure may shift over
time---for example, when user preferences evolve or the task distribution
changes across sessions.
This appendix states and proves the extension to a standard block-stationary
model, together with a matching lower bound.

The block-stationary setting captures the simplest version of this phenomenon:
within each stationary block the theory of
Sections~\ref{sec:limits}--\ref{sec:algo} applies unchanged.
When restart boundaries align with the stationary segments, the additional
statistical cost of non-stationarity is the standard \(\sqrt{S+1}\) factor.
For fixed restarts with arbitrary change-point locations, the wrapper may mix
data from two regimes inside a restart interval; this contributes an
additional \(O(SL)\) misalignment term, where \(L\) is the restart period.
Extending to smoothly drifting, adversarial, or adaptively detected changes
is an interesting direction for future work.

\paragraph{Setting.}
The reward function is piecewise constant: there exist change points
$1=\tau_1<\tau_2<\cdots<\tau_{S+1}\le T$ such that within each block
$[\tau_s,\tau_{s+1})$ the reward function
$\mu_s:\cX\times\cA\to[0,1]$ is fixed.
Let $T_s\coloneqq\tau_{s+1}-\tau_s$ denote the length of block~$s$, so that
$\sum_{s=1}^{S+1}T_s=T$.
For each block~$s$, let $\eps^\star_{\infty,s}(K)$ denote the worst-case
$K$-state distortion under $\mu_s$, and define
\[
\eps^\star_{\infty,\max}(K)
\;\coloneqq\;
\max_{s\in[S+1]}\eps^\star_{\infty,s}(K).
\]

\paragraph{Algorithm: restarting DeMem.}
We consider a restart wrapper with restart interval length \(L\). Let
\(M\coloneqq\lceil T/L\rceil\) be the number of restart intervals. At the
start of each interval, the wrapper resets DeMem: it clears all
context-level and cluster-level statistics, reinitializes the partition to a
single cluster, and restarts the doubling-epoch schedule of
Algorithm~\ref{alg:DeMem}.

If the restart boundaries align with, or refine, the true stationary
segments, each restart interval lies inside a single reward regime. If a
true change point falls inside a restart interval, that interval may mix
pre- and post-change data. Since per-round regret is at most one, each such
contaminated interval contributes at most \(L\) additional regret. Thus
arbitrary misaligned change points introduce an additive \(O(SL)\) term.

\begin{proposition}[Non-stationary regret with fixed restarts]
\label{prop:nonstationary}
Suppose the reward function changes at most \(S\) times over \(T\) rounds.
Run the restarting DeMem wrapper with fixed restart interval length
\(L\), and let \(M\coloneqq\lceil T/L\rceil\). With probability at least
\(1-\delta\),
\begin{equation}
\label{eq:nonstationary-full}
\Reg(T)
\;\le\;
T\cdot O\!\bigl(\eps^\star_{\infty,\max}(K)\bigr)
\;+\;
\tilde O\!\bigl(\sqrt{AKTM}\bigr)
\;+\;
O\!\bigl(A M N_T B_T(\gamma)\bigr)
\;+\;
O(SL).
\end{equation}
The final term \(O(SL)\) accounts for restart intervals that contain
misaligned change points. If the restart schedule aligns with or refines the
true stationary segments, this term is zero.
\end{proposition}

\begin{proof}
Partition the horizon into \(M\) restart intervals
\(I_1,\dots,I_M\), each of length at most \(L\). Call an interval
\emph{contaminated} if it contains a true change point. There are at most
\(S\) contaminated intervals. Since one-step regret is bounded by one, their
total contribution is at most \(O(SL)\).

On every non-contaminated interval \(I_j\), the reward function is fixed.
Applying Theorem~\ref{thm:upper} to interval \(I_j\), with confidence
parameter \(\delta/M\), gives
\[
\Reg_j
\le
|I_j|\cdot O\!\bigl(\eps^\star_{\infty,j}(K)\bigr)
+
\tilde O\!\bigl(\sqrt{AK|I_j|}\bigr)
+
O\!\bigl(A N_{I_j} B_T(\gamma)\bigr),
\]
where \(N_{I_j}\) is the number of distinct contexts observed in interval
\(I_j\). Summing over non-contaminated intervals and absorbing logarithmic
factors from the union bound gives the three main terms.

For the compression term,
\[
\sum_j |I_j|\eps^\star_{\infty,j}(K)
\le
T\eps^\star_{\infty,\max}(K).
\]
For the statistical term, Cauchy--Schwarz gives
\[
\sum_{j=1}^{M}\sqrt{AK|I_j|}
\le
\sqrt{M\cdot AK\sum_{j=1}^{M}|I_j|}
=
\sqrt{AKTM}.
\]
For certification exploration, a context can be rediscovered after each
restart, so
\[
\sum_{j=1}^{M} N_{I_j}\le M N_T,
\]
which yields \(O(A M N_T B_T(\gamma))\). Adding the contaminated-interval
cost \(O(SL)\) proves \eqref{eq:nonstationary-full}.
\end{proof}

\begin{remark}[Adaptive restarts]
\label{rem:adaptive-restart}
When \(S\) is unknown or change points are not aligned with a fixed restart
schedule, one can replace fixed restarts with a change-detection mechanism
that monitors shifts in reward estimates or cannot-link certificate rates.
In such a scheme, the \(O(SL)\) term would be replaced by a detection-delay
cost depending on the separation between pre- and post-change reward
functions. Providing sharp adaptive guarantees for DeMem requires a
separate change-detection analysis, which we leave to future work.
\end{remark}

\section{Agent Memory Instantiation and Theory--Implementation Alignment}
\label{app:impl}

\subsection{Why an instantiation is needed}
The theory models memory as a finite runtime state $M_t\in[K]$ induced by an encoder $g:\cX\to[K]$.
To realize this abstraction in an agent-memory system, we instantiate a minimal $K$-slot mechanism in which histories are routed into a bounded set of reusable slot identities at answer time.
This slot-based view is the operational form of the budgeted memory model studied in the main text.
The remaining implementation choices---including slot realizations, update rules, and routing scores---specify how these budgeted decision states are represented and maintained within an agent pipeline.

\subsection{System components}
At each step, an agent receives a context $x_t$ and must make a downstream decision under a strict runtime budget of $K$ discrete states.
The implementation contains three modules:
\begin{enumerate}
    \item \textbf{Memory slots} $\{\mathcal S_k\}_{k=1}^K$ storing compressed state information.
    \item \textbf{Router} $e_\theta(x_t)\in[K]$ selecting a slot identity.
    \item \textbf{Policy} $\pi_\phi$, whose downstream behavior is conditioned on the selected slot and its maintained state realization.
\end{enumerate}

\paragraph{Memory control versus downstream decision.}
It is important to separate memory control from task decision.
Routing, slot reading, slot updating, and certified splitting are memory operations that implement the encoder and state-maintenance mechanism.
By contrast, the action in the abstract model corresponds to the downstream decision supported by the selected slot.
In a language agent, this downstream decision may be realized as an answer, a tool choice, or another task-level control output conditioned on the selected memory state.
\paragraph{Conversational-benchmark realization.}
In the LoCoMo and LongMemEval experiments, these components are instantiated as follows.
\begin{itemize}[leftmargin=2em,itemsep=2pt]
    \item \textbf{Context} $x_t = (h_t, q_t)$: the dialogue history $h_t$ accumulated across sessions up to time~$t$, together with the evaluation query $q_t$.
    \item \textbf{Router} $e_\theta$: assigns $(h_t, q_t)$ to one of $K$ active slots by computing $m_t = \arg\max_{k\in[K]} \mathrm{Score}_\theta(x_t, \mathcal{S}_k)$, where $\mathrm{Score}_\theta$ is an embedding-similarity score between the query and each slot's current content.
    \item \textbf{Slot content} $\mathcal{S}_{m_t}$: a textual summary of at most $L$ characters, maintained via the update rule in~\eqref{eq:summary-update}. This is the bounded realization of the abstract memory state~$m_t$.
    \item \textbf{Action / Answer}: the answering LLM receives the slot content $\mathcal{S}_{m_t}$ as the sole memory evidence and generates an answer $a_t$ under deterministic decoding. Under this protocol, the action is fully determined by the selected slot identity and its content, so the effective policy $\pi(\cdot\mid m_t)$ is a point mass.
    \item \textbf{Reward} $\mu(x_t, a_t)$: the binary correctness score (LoCoMo) or graded accuracy (LongMemEval) assigned by the held-out LLM judges.
\end{itemize}
Under this mapping, the abstract model's prediction is concrete: answer quality is bounded by (i) how well the $K$-slot partition preserves query-relevant distinctions (the compression term $\eps^\star_\infty(K)$), (ii) how accurately the router selects the correct slot (the routing term $\eta_{\mathrm{route}}$), and (iii) how faithfully the bounded slot content supports the correct answer (the realization term $\eta_{\mathrm{read}}$).Figure~\ref{fig:worked-example} visualizes this chain end to
end.
\subsection{Slot representations}
Each slot may store either
(i) a bounded textual summary $\sigma_k$, or
(ii) a compact vector prototype $p_k\in\mathbb R^d$, or both.
For textual slots, we impose a per-slot token budget $L$ so that
\[
|\sigma_k|\le L
\qquad
\text{for all }k,
\]
which keeps the total memory footprint bounded by $K\cdot L$.

\subsection{Routing and updates}
The router assigns contexts to slots by
\begin{equation}
\label{eq:router-score}
m_t = \arg\max_{k\in[K]} \Score_\theta(x_t,\mathcal S_k),
\end{equation}
where $\Score_\theta$ may be implemented by embedding similarity, a small MLP, or an LLM-based routing prompt, as long as the output is discrete.

For vector slots, one may use the update
\begin{equation}
\label{eq:proto-update}
p_{m_t}\leftarrow (1-\eta)p_{m_t}+\eta f(x_t),
\end{equation}
while textual slots are updated by first extracting an implementation-level
memory item
\[
u_t \coloneqq \psi_\theta(x_t),
\]
and then applying
\begin{equation}
\label{eq:summary-update}
\sigma_{m_t}\leftarrow \Compress_L\!\big(\sigma_{m_t}\oplus u_t\big),
\end{equation}
where \(\psi_\theta\) denotes the textual memory extractor, \(\oplus\)
denotes concatenation, and \(\Compress_L(\cdot)\) enforces the per-slot
length limit.

\subsection{Certified split as a memory operation}
The operational analogue of the theory is a cannot-link constraint.
Whenever two contexts routed to the same slot are certified to prefer conflicting actions, the slot is marked inconsistent and becomes eligible for splitting.

A generic empirical trigger is
\begin{equation}
\label{eq:cannotlink-test}
\exists a\neq a' \in \cA:
\quad
\widehat r(x,a)-\widehat r(x,a')>\beta_t
\quad\text{and}\quad
\widehat r(x',a')-\widehat r(x',a)>\beta_t.
\end{equation}
Intuitively, $x$ prefers $a$ while $x'$ prefers $a'$ with sufficient confidence, so merging them induces nontrivial decision distortion.

\subsection{Executing the split}
\label{app:executing-split}

Given a cannot-link witness pair \((x,x')\) inside an active slot \(k\),
DeMem creates a new slot \(k'\). To avoid circular dependence
between repartitioning and slot contents, the split is first computed using
temporary seed contents derived from the witness pair, rather than from the
yet-uninitialized post-split slots.

Let
\[
\mathcal A_k^{(0)}
\coloneqq
\MemInit(\{x\}),
\qquad
\mathcal A_{k'}^{(0)}
\coloneqq
\MemInit(\{x'\})
\]
be two temporary anchors. For textual memory, \(\MemInit(\{x\})\) denotes
the bounded summary or compressed evidence item generated from \(x\); for
vector memory, it denotes the prototype \(f(x)\). These anchors are used
only to compute the split assignment. For every item \(z\) currently assigned
to slot \(k\), define
\begin{equation}
\label{eq:split-assignment}
e_\theta(z)
=
\begin{cases}
k,
&
\Score_\theta(z,\mathcal A_k^{(0)})
\ge
\Score_\theta(z,\mathcal A_{k'}^{(0)}),
\\
k',
&
\text{otherwise},
\end{cases}
\end{equation}
with deterministic tie-breaking.

Let
\[
P_k=\{z:e_\theta(z)=k\},
\qquad
P_{k'}=\{z:e_\theta(z)=k'\}.
\]
The actual slot contents are then reinitialized from the resulting
partitions:
\[
\mathcal S_k
\gets
\MemInit(P_k),
\qquad
\mathcal S_{k'}
\gets
\MemInit(P_{k'}).
\]
Thus routing scores in Eq.~\eqref{eq:split-assignment} depend only on fixed
temporary anchors derived from the witness pair, while the final slot
contents are computed after the partition is determined.

\subsection{Runtime loop}
\label{app:runtime-loop}

\begin{algorithm}[t]
\caption{Budgeted Agent Memory with Certified Splits}
\label{alg:agent_memory_instantiation}
\begin{algorithmic}[1]
\Require Budget \(K\), per-slot token limit \(L\), confidence schedule \(\beta_t\), default split threshold \(\tau_{\mathrm{split}}\), saturated split threshold \(\tau_{\mathrm{sat}}\ge\tau_{\mathrm{split}}\)
\State Initialize active slots \(\{\mathcal S_k\}_{k=1}^{K_0}\) with \(K_0\le K\)
\State Initialize cannot-link sets
\(\mathcal C\gets\emptyset\) and
\(\mathcal C_{\mathrm{sat}}\gets\emptyset\)
\For{\(t=1,2,\dots\)}
    \State Observe context \(x_t\)
    \State Route:
    \[
    m_t \gets
    \arg\max_{k:\mathcal S_k\ \mathrm{active}}
    \Score_\theta(x_t,\mathcal S_k)
    \]
    \State Read slot content
    \[
    \tilde s_t\gets \MemRead(\mathcal S_{m_t})
    \]
    \State Act:
    \[
    a_t\sim \pi_\phi(\cdot\mid m_t,\tilde s_t)
    \]
    \State Receive reward \(r_t\)
    \State Update reward statistics
    \State Update slot content via \eqref{eq:proto-update} or \eqref{eq:summary-update}
    \State Set the active split threshold
    \[
    \tau_t \gets
    \begin{cases}
    \tau_{\mathrm{split}}, & \text{if the number of active slots is } <K,\\
    \tau_{\mathrm{sat}}, & \text{otherwise},
    \end{cases}
    \qquad
    \tau_{\mathrm{sat}}\ge \tau_{\mathrm{split}} .
    \]
    \State Check certified inconsistency via \eqref{eq:cannotlink-test} using threshold \(\tau_t\)
    \If{a cannot-link witness is found}
        \If{number of active slots \(<K\)}
            \State Add the witness to \(\mathcal C\)
            \State Split the slot using \eqref{eq:split-assignment}
        \Else
            \State Add the witness to \(\mathcal C_{\mathrm{sat}}\)
            \State Suppress the split; no slot is created, merged, or evicted
            \Comment{capacity-saturated thresholding}
        \EndIf
    \EndIf
\EndFor
\end{algorithmic}
\end{algorithm}

\paragraph{Capacity handling.}
In all reported experiments, DeMem handles the strict \(K\)-slot
capacity by threshold-based split suppression rather than by merge or
eviction. When free slots remain, split witnesses are tested with the default
threshold \(\tau_{\mathrm{split}}\). Once all \(K\) slots are active, the
system switches to the more conservative threshold
\(\tau_{\mathrm{sat}}\ge\tau_{\mathrm{split}}\); witnesses that still exceed
this threshold are logged in \(\mathcal C_{\mathrm{sat}}\), but the
corresponding split is not executed. No LRU eviction, semantic merging, or
slot replacement is used in the reported experiments.

This policy preserves the hard \(K\)-slot budget without introducing an
additional memory-management heuristic. Conflicts discovered after capacity
saturation may therefore remain unresolved and are reflected in the realized
compression or routing error.

\subsection{Alignment with the theory}
\label{app:alignment-table}
The correspondence between the abstract model and the implementation is direct:
\begin{center}
\small
\begin{tabular}{lll}
\toprule
\textbf{Theory} & \textbf{Generic Implementation} & \textbf{LoCoMo / LongMemEval} \\
\midrule
context $x=(h,q)$ & history plus current query & multi-session dialogue history $h$ \\
 & & \quad and evaluation query $q$ \\[3pt]

query $q$ & decision-time side information & user/evaluation question \\
 & \quad bypassing the memory bottleneck & \quad given directly to the LLM \\[3pt]

encoder $g:\mathcal{H}\times\mathcal{Q}\to[K]$ & query-aware router $e_\theta$ &
$m=\arg\max_k\mathrm{Score}_\theta((h,q),\mathcal{S}_k)$ \\[3pt]

memory state $m\in[K]$ & selected slot identity & selected slot identity $+$ bounded \\
 & & \quad textual summary ($\leq L$ chars) \\[3pt]

policy $\pi(\cdot\mid m,q)$ & slot- and query-conditioned & deterministic LLM answer conditioned \\
 & \quad downstream decision rule & \quad on selected slot content and query \\[3pt]

action $a\in\cA$ & generated downstream decision & generated answer string \\[3pt]

reward $\mu(x,a)$ & task reward & binary correctness (LoCoMo) / \\
 & & \quad graded accuracy (LongMemEval) \\[3pt]

cannot-link certificate & split trigger & certified answer conflict \\
 & & \quad between co-routed histories for a query \\[3pt]

hard budget $K$ & number of active slots & number of active slots \\[3pt]

distortion $\eps$ & conflict tolerance / & answer-quality loss from \\
 & \quad split threshold & \quad merging decision-incompatible histories \\
\bottomrule
\end{tabular}
\end{center}

The correspondence is at the level of runtime decision states.
The query is provided directly to the policy at answer time, while the dialogue history is accessed through one of $K$ selected slot identities.
Thus the memory budget limits the number of runtime states available for representing history, not the number of possible natural-language answers across all queries.
In the language-agent implementation, each slot additionally maintains a bounded realization---for example, a textual summary, a compact prototype, or both---through which the selected memory state is materialized for answer generation.
These realizations do not change the underlying budgeted state abstraction; they specify how each discrete memory state is represented inside the agent pipeline.

\paragraph{Reading the table.}
Each row shows a single concept at three levels of abstraction.
The leftmost column states the formal object from Sections~\ref{sec:setup}--\ref{sec:distortion}; the middle column gives the generic system component from Algorithm~\ref{alg:agent_memory_instantiation}; the rightmost column specifies the exact realization used in the conversational benchmarks of Section~\ref{sec:exp}.
This three-column view makes it possible to trace any theoretical statement, such as the forgetting boundary of Theorem~\ref{thm:forgetting} or the regret decomposition of Theorem~\ref{thm:upper}, to its operational counterpart in the experimental system.
A concrete instance of this trace is shown in Figure~\ref{fig:worked-example}.

This decomposition also implies that the routing criterion is modular: because $\eta_{\mathrm{route}}$ enters additively and independently of the compression floor and the realization term, replacing the retrieval or selection stage of an existing memory architecture with decision-aware routing can reduce distortion without modifying the remaining pipeline.
We validate this composability empirically in Appendix~\ref{app:modularity}: applying the decision-aware selection criterion to two existing memory systems without modifying their storage or generation layers improves both on LoCoMo.

\subsection{A formal approximation bridge from abstract states to slot-based agents}
\label{app:bridge}

Beyond the conceptual alignment, the slot-based view yields a simple approximation guarantee.
The key point is that the abstract theory controls the \emph{dominant compression term}, while the concrete implementation contributes only additional approximation errors due to routing and bounded state realization.

Let $(g^\star,\{a_m^\star\}_{m=1}^K)$ be any $K$-state decision rule, where
\[
g^\star:\cX\to[K], \qquad a_m^\star\in\cA.
\]
Its worst-case distortion is
\[
\sup_{x\in\cX}\bigl(\mu^\star(x)-\mu(x,a_{g^\star(x)}^\star)\bigr).
\]
In particular, one may take $(g^\star,\{a_m^\star\})$ to be an optimal solution achieving $\eps^\star_\infty(K)$.

Now consider a slot-based implementation with router $e_\theta$, slot memory $\{\mathcal S_k\}_{k=1}^K$, slot readout $\MemRead(\mathcal S_k)$, and downstream policy $\pi_\phi$.

We separate the implementation gap into two terms.

\paragraph{Routing error.}
Assume that under $X\sim D$,
\begin{equation}
\label{eq:routing-error}
\Pr\!\big[e_\theta(X)\neq g^\star(X)\big] \le \eta_{\mathrm{route}} .
\end{equation}

\paragraph{State-realization error.}
Assume that whenever the router selects the same state as the abstract encoder, the implemented slot-conditioned policy loses at most $\eta_{\mathrm{read}}$ additional reward relative to the abstract slot action:
\begin{equation}
\label{eq:readout-error}
\EE\!\left[\mu\!\left(X,a_{g^\star(X)}^\star\right) - \mu\!\left(X,A^{\mathrm{impl}}\right)\,\middle|\, e_\theta(X)=g^\star(X)\right]
\le
\eta_{\mathrm{read}},
\end{equation}
where $A^{\mathrm{impl}}\sim \pi_\phi\!\left(\cdot \mid e_\theta(X), \MemRead(\mathcal S_{e_\theta(X)})\right)$.

\begin{proposition}[Approximation bridge for slot-based memory]
\label{prop:bridge}
Under \eqref{eq:routing-error}--\eqref{eq:readout-error}, the average distortion of the implemented slot-based agent satisfies
\begin{equation}
\label{eq:bridge-bound}
\EE_{X\sim D}\!\left[\mu^\star(X)-\mu\!\left(X,A^{\mathrm{impl}}\right)\right]
\le
\eps^\star_\infty(K) + \eta_{\mathrm{route}} + \eta_{\mathrm{read}} .
\end{equation}
More generally, for any comparator $(g^\star,\{a_m^\star\})$ with worst-case distortion at most $\bar\eps$, the same argument yields
\[
\EE\!\left[\mu^\star(X)-\mu\!\left(X,A^{\mathrm{impl}}\right)\right]
\le
\bar\eps + \eta_{\mathrm{route}} + \eta_{\mathrm{read}}.
\]
\end{proposition}

\begin{proof}
Decompose the loss according to whether the learned router agrees with the abstract encoder:
\[
\mu^\star(X)-\mu(X,A^{\mathrm{impl}})
=
\underbrace{\mu^\star(X)-\mu\!\left(X,a_{g^\star(X)}^\star\right)}_{\text{compression term}}
+
\underbrace{\mu\!\left(X,a_{g^\star(X)}^\star\right)-\mu(X,A^{\mathrm{impl}})}_{\text{implementation term}}.
\]
The first term is at most $\eps^\star_\infty(K)$ pointwise by definition of the optimal $K$-state abstraction.
For the second term, condition on the event $\{e_\theta(X)=g^\star(X)\}$.
On this event, Assumption~\eqref{eq:readout-error} bounds the excess loss by $\eta_{\mathrm{read}}$ in expectation.
On the complement event, rewards lie in $[0,1]$, so the additional loss is at most $1$.
Therefore
\[
\EE\!\left[\mu\!\left(X,a_{g^\star(X)}^\star\right)-\mu(X,A^{\mathrm{impl}})\right]
\le
\eta_{\mathrm{read}} + \Pr[e_\theta(X)\neq g^\star(X)]
\le
\eta_{\mathrm{read}}+\eta_{\mathrm{route}}.
\]
Combining the two parts proves \eqref{eq:bridge-bound}.
\end{proof}

\paragraph{Interpretation.}
Proposition~\ref{prop:bridge} identifies how the abstract memory frontier and the realized slot-based system compose.
The term $\eps^\star_\infty(K)$ captures the irreducible price of compressing histories into $K$ reusable decision states, while routing and bounded slot realization contribute additional approximation terms through $\eta_{\mathrm{route}}$ and $\eta_{\mathrm{read}}$.
This decomposition makes the role of implementation explicit:
better routing, more faithful slot realizations, and earlier conflict detection improve the realized agent by reducing the approximation terms, while the underlying memory-compression limit remains governed by the $K$-state decision abstraction.

\subsection{Theory-faithful Core Variant on LoCoMo}
\label{app:core-locomo}

To further isolate the contribution of the slot-level memory organization, we evaluate a reduced DeMem variant on \textbf{LoCoMo}, denoted \textsc{DeMem-Core}. This variant keeps the same budgeted slot structure, routing mechanism, and conflict-driven refinement rule as the full method, while removing the richer generation-time enhancements used in the complete implementation. In this sense, \textsc{DeMem-Core} measures how far the slot organization itself can go under the same overall memory regime, before adding the full set of realization and generation components.

\paragraph{Setup.}
We use the same LoCoMo preprocessing, prompting template, answering backbone, deterministic decoding protocol, and judge setup as in the main experiments. The answer-time memory budget is matched to the main setting. Relative to the full system, \textsc{DeMem-Core} uses the same slot budget and the same slot assignment / certified refinement procedure, but relies on a reduced slot realization at answer time rather than the richer memory realization used by the complete model. This yields a more stripped-down version of the method while preserving the same budgeted memory organization.

\paragraph{Result.}
Under this protocol, \textsc{DeMem-Core} reaches an overall LoCoMo score of \textbf{90.8}. For compactness, we report the overall score here; the category-wise pattern follows the same qualitative trend as in the main results.

\begin{table}[t]
\centering
\small
\begin{tabular}{lc}
\toprule
Method & Overall Score \\
\midrule
\textsc{DeMem-Core} & \textbf{90.8} \\
\bottomrule
\end{tabular}
\caption{Overall LoCoMo result for the reduced \textsc{DeMem-Core} variant under the same evaluation protocol as the main experiments.}
\label{tab:locomo-core}
\end{table}

\paragraph{Discussion.}
The result shows that a substantial portion of the final gain already arises from the budgeted slot organization and the conflict-driven refinement mechanism themselves. The full implementation improves further by using richer slot realizations and a stronger generation-time use of the selected memory state, but the reduced variant is already competitive on its own. This supports the view that the central benefit of DeMem lies in how memory is organized under a fixed budget, rather than only in the additional realization details of the full system.

\subsection{Practical split trigger in LLM benchmarks}
\label{app:practical-split-trigger}

The certified split rule in Section~\ref{sec:certificates} is defined for
the online decision model, where repeated context--action feedback supports
one-sided certificates of decision incompatibility. One-pass conversational
benchmarks such as LoCoMo provide an answer-level evaluation interface under
deterministic decoding. We therefore instantiate the same one-sided refinement
principle using candidate-restricted answer-conflict scoring: a split is
introduced only when the scored candidate set indicates that no shared candidate
memory state performs well for both contexts.

For a context \(x=(h,q)\) and a candidate memory slot \(u\), let
\[
a(x,u)
\coloneqq
\mathrm{Ans}_{\theta}(q,S_u)
\]
be the deterministic answer produced from query \(q\) and slot contents
\(S_u\). Let \(r_x\) denote the reference answer for \(x\). The
implementation evaluates the answer using \(B\) feedback scores. In our
main experiments \(B=2\), corresponding to two held-out feedback scorers:
\[
s_b(x,u)
\coloneqq
J_{\mathrm{fb}}^{(b)}\bigl(a(x,u),r_x\bigr)\in[0,1],
\qquad b\in[B].
\]
We define the answer-level reward estimate and its score dispersion as
\[
\widehat\mu_{\mathrm{impl}}(x,u)
\coloneqq
\frac{1}{B}\sum_{b=1}^{B}s_b(x,u),
\qquad
\widehat\sigma^2_{\mathrm{impl}}(x,u)
\coloneqq
\frac{1}{B-1}\sum_{b=1}^{B}
\Bigl(s_b(x,u)-\widehat\mu_{\mathrm{impl}}(x,u)\Bigr)^2 .
\]
The dispersion term is used only as a conservative guard band. Specifically,
we set
\[
\beta_{\mathrm{impl}}(x,u)
\coloneqq
c_{\beta}
\sqrt{\frac{\widehat\sigma^2_{\mathrm{impl}}(x,u)+\sigma_0^2}{B}}
+
\eta_{\mathrm{cal}},
\]
where \(c_{\beta}\) is a fixed safety multiplier, \(\sigma_0>0\) is a small
floor preventing overconfidence from a small number of scorers, and
\(\eta_{\mathrm{cal}}\) is a calibration slack chosen on held-out validation
instances. This quantity serves as an implementation-level guard band for benchmark
scoring, accounting for scorer disagreement and single-pass evaluation noise.

Given a pair of contexts \(x\) and \(x'\) currently assigned to the same
memory block, we construct a small candidate set
\[
\mathcal U(x,x')
\]
of slot-induced decisions. This set includes the current shared slot, the
top router-proposed slots for \(x\) and \(x'\), and, when available, the
candidate split slots obtained by materializing separate slot summaries for
the two contexts. All counterfactual scores are computed by re-running the
deterministic answerer with the corresponding candidate slot and scoring the
resulting answer with the same feedback scorers.

For each \(u\in\mathcal U(x,x')\), because the feedback scores lie in
\([0,1]\), we define clipped lower and upper score envelopes:
\[
\begin{aligned}
\mathrm{LCB}_{\mathrm{impl}}(x,u)
&\coloneqq
\max\!\left\{0,\,
\widehat\mu_{\mathrm{impl}}(x,u)-\beta_{\mathrm{impl}}(x,u)
\right\},\\
\mathrm{UCB}_{\mathrm{impl}}(x,u)
&\coloneqq
\min\!\left\{1,\,
\widehat\mu_{\mathrm{impl}}(x,u)+\beta_{\mathrm{impl}}(x,u)
\right\}.
\end{aligned}
\]
The candidate-restricted analogue of the optimal value for \(x\) is
\[
\underline{\mu}^{\star}_{\mathcal U}(x)
\coloneqq
\max_{v\in\mathcal U(x,x')}
\mathrm{LCB}_{\mathrm{impl}}(x,v).
\]
For a shared candidate slot \(u\), we then define a conservative lower
estimate of its candidate-restricted decision loss:
\[
\widehat\Delta^{-}_{\mathcal U}(x,u)
\coloneqq
\Bigl[
\underline{\mu}^{\star}_{\mathcal U}(x)
-
\mathrm{UCB}_{\mathrm{impl}}(x,u)
\Bigr]_+ .
\]
This follows the form of the theoretical suboptimality gap
\(\Delta(x,a)=\mu^\star(x)-\mu(x,a)\), specialized to the candidate set of
slot-induced answers and evaluated with the implementation-level guard band.

We declare a pairwise answer conflict when every shared candidate slot
in \(\mathcal U(x,x')\) incurs a sufficiently large lower-bounded loss for
at least one of the two contexts. Formally, define the candidate-restricted radius certificate
\[
\widehat\rho^{-}_{\mathrm{impl}}(x,x')
\coloneqq
\min_{u\in\mathcal U(x,x')}
\max\Bigl\{
\widehat\Delta^{-}_{\mathcal U}(x,u),
\widehat\Delta^{-}_{\mathcal U}(x',u)
\Bigr\}.
\]
A cannot-link edge is added when
\begin{equation}
\label{eq:impl-split-trigger}
\widehat\rho^{-}_{\mathrm{impl}}(x,x')
>
\epsilon_{\mathrm{split}} .
\end{equation}
The resulting cannot-link graph is then passed to the same budgeted
partitioning step used by the memory encoder. Thus the split trigger
identifies decision conflicts, while the fixed \(K\)-slot budget still
controls which conflicts can be separated at runtime.

This rule is one-sided by design. A fired edge indicates candidate-level
evidence of decision conflict, while the absence of an edge is not used as a
positive compatibility claim. This conservative convention preserves the
direction of the theoretical cannot-link logic under deterministic benchmark
evaluation.

\section{Additional Experimental Details}
\label{app:exp-details}
\subsection{Synthetic Diagnostic Environment}
\label{app:synth}

We construct a controlled \emph{Decoupled Bandit} environment to isolate the effect of memory organization under a fixed runtime memory budget.

\paragraph{Contexts, descriptions, and decision identities.}
At each round \(t=1,\dots,T\), the environment emits a context
\[
c_t = (x_t, z_t),
\]
where \(x_t \in \mathbb{R}^d\) is a \emph{descriptive representation} and \(z_t \in \{1,\dots,M\}\) is a latent \emph{decision identity}. 
The key property of the environment is that \(x_t\) and \(z_t\) are only partially aligned: descriptively similar contexts need not share the same optimal action.

Each decision identity \(z\) is associated with an action-value vector
\[
\mu_z \in [0,1]^A,
\]
where \(A\) is the number of available actions. The optimal action for a context \(c_t\) is
\[
a^\star(c_t) = \arg\max_{a \in [A]} \mu_{z_t,a}.
\]

\paragraph{Reward model.}
After choosing action \(a_t\), the learner receives reward
\[
r_t = \mu_{z_t,a_t} + \xi_t,
\]
where \(\xi_t\) is zero-mean observation noise. The cumulative regret up to round \(T\) is
\[
\mathrm{Regret}(T) = \sum_{t=1}^{T} \bigl(\mu_{z_t,a^\star(c_t)} - \mu_{z_t,a_t}\bigr).
\]

\paragraph{Runtime memory budget.}
At any time, the learner is allowed to maintain at most \(K\) reusable memory states (or slots),
\[
|\mathcal{M}_t| \le K.
\]
Each method defines its own assignment rule from a context \(c_t\) to a memory state \(m(c_t)\in \mathcal{M}_t\), but all methods are evaluated under the same budget \(K\).

\paragraph{Decision distortion induced by compression.}
Given a memory assignment \(m(c)\), define the induced decision profile of a memory state \(s\in\mathcal{M}\) as the average reward vector over the contexts mapped to it,
\[
\bar{\mu}_s = \mathbb{E}\!\left[\mu_{z} \mid m(c)=s\right].
\]
The action selected from the compressed memory is
\[
\hat{a}(c)=\arg\max_{a\in[A]} \bar{\mu}_{m(c),a}.
\]
We define the decision distortion of a memory system as
\[
D(\mathcal{M}) = \mathbb{P}\!\bigl[\hat{a}(c) \neq a^\star(c)\bigr],
\]      
or equivalently one may use expected value loss
\[
D_{\mathrm{val}}(\mathcal{M}) 
= \mathbb{E}\!\left[\mu_{z,a^\star(c)} - \mu_{z,\hat{a}(c)}\right].
\]
Figure~\ref{fig:synthetic-results}(b) reports the empirical memory--distortion curve by varying \(K\).
\paragraph{Mismatch-severity parameterization.}
To study the mechanism continuously rather than at a single operating point, we introduce a mismatch parameter \(\alpha \in [0,1]\) controlling the discrepancy between descriptive similarity and decision similarity.

Concretely, let \(\tilde{z}_t\) denote a description-induced pseudo-identity obtained from the descriptive feature \(x_t\). 
We then sample the true decision identity \(z_t\) from a mixture
\[
z_t =
\begin{cases}
\tilde{z}_t, & \text{with probability } 1-\alpha,\\
\mathrm{Perm}(\tilde{z}_t), & \text{with probability } \alpha,
\end{cases}
\]
where \(\mathrm{Perm}(\cdot)\) is a fixed permutation that maps descriptively similar contexts to decision-distinct groups. 
Thus, \(\alpha=0\) corresponds to a fully aligned regime, while larger \(\alpha\) produces increasingly severe mismatch.

Figure~\ref{fig:synthetic-results}(c) reports performance as a function of \(\alpha\) under a fixed runtime memory budget \(K\).
\paragraph{Synthetic analyses.}
The three panels in Figure~\ref{fig:synthetic-results} evaluate complementary aspects of the same environment.

\begin{itemize}
    \item \textbf{Figure~\ref{fig:synthetic-results}(a): Cumulative regret.}  
    We fix the runtime memory budget \(K\) and compare the cumulative regret trajectories of all methods over \(T\) rounds.

    \item \textbf{Figure~\ref{fig:synthetic-results}(b): Empirical memory--distortion curve.}  
    We vary \(K\) over a predefined budget grid and estimate the resulting decision distortion \(D(\mathcal{M})\). 
    This yields an empirical counterpart of the memory--distortion tradeoff in the theory.

    \item \textbf{Figure~\ref{fig:synthetic-results}(c): Mismatch mechanism analysis.}  
    We fix \(K\) and vary the mismatch severity \(\alpha\). 
    This directly tests the central hypothesis of the paper: a decision-centric memory should become increasingly advantageous as descriptive similarity becomes less informative for action selection.
\end{itemize}
\paragraph{Oracle.}
The oracle baseline has access to the true decision identity \(z_t\). 
It therefore assigns each context to the correct decision state and selects
\[
a_t^{\mathrm{oracle}} = \arg\max_{a\in[A]} \mu_{z_t,a}.
\]
This serves as an upper bound on achievable performance in the synthetic environment.
\paragraph{Feature-KMeans.}
This baseline organizes memory purely by descriptive similarity. 
Given a budget \(K\), it clusters the descriptive features \(\{x_t\}\) into \(K\) clusters with centers \(\{\nu_1,\dots,\nu_K\}\) by minimizing
\[
\sum_{t=1}^{T} \min_{k\in[K]} \|x_t - \nu_k\|_2^2.
\]
At test time, a context \(x\) is assigned to the nearest cluster
\[
m_{\mathrm{FKM}}(x) = \arg\min_{k\in[K]} \|x-\nu_k\|_2^2,
\]
and the action is chosen from the empirical reward vector stored for that cluster,
\[
a_t^{\mathrm{FKM}} = \arg\max_{a\in[A]} \hat{\mu}_{m_{\mathrm{FKM}}(x_t),a}.
\]
Its core inductive bias is that descriptively similar contexts should share memory.
\paragraph{Feature-RAG.}
This baseline retrieves memory by nearest-neighbor similarity in descriptive feature space. 
Given a memory bank \(\mathcal{B}_t = \{(x_i,\hat{\mu}_i)\}_{i\le t}\), it retrieves
\[
j^\star(x_t)=\arg\min_{i\le t}\|x_t-x_i\|_2,
\]
or the top-\(k\) nearest neighbors under the same metric. 
The predicted reward vector is then obtained by local aggregation,
\[
\hat{\mu}^{\mathrm{FRAG}}(x_t)
=
\frac{1}{Z_t}\sum_{i\in \mathcal{N}_k(x_t)} w_i(x_t)\hat{\mu}_i,
\]
where \(w_i(x_t)\) are similarity weights and \(Z_t=\sum_{i\in\mathcal{N}_k(x_t)}w_i(x_t)\). 
The action is
\[
a_t^{\mathrm{FRAG}}=\arg\max_{a\in[A]}\hat{\mu}^{\mathrm{FRAG}}_a(x_t).
\]
Unlike DeMem, this baseline does not explicitly preserve decision-relevant distinctions beyond local descriptive similarity.
\paragraph{\(\epsilon\)-greedy feature clustering.}
This baseline maintains a set of feature-based clusters online. 
Given a context \(x_t\), it is assigned to the nearest existing cluster in descriptive space, unless a new cluster is created by an exploration step. 
Let \(m_t(x_t)\) denote the selected cluster. 
The action is chosen by
\[
a_t=
\begin{cases}
\text{Uniform}([A]), & \text{with probability } \epsilon,\\
\arg\max_{a\in[A]}\hat{\mu}_{m_t(x_t),a}, & \text{with probability } 1-\epsilon.
\end{cases}
\]
Cluster statistics are then updated with the observed reward. 
The baseline is still description-aware because the state abstraction is driven by feature proximity rather than decision conflict.
\paragraph{CLUB-style online clustering baseline.}
We also include an online partition baseline in the spirit of \citet{gentile2014onlineclusteringbandits}. 
It maintains a graph \(G_t=(V_t,E_t)\) over observed contexts or local states, where nodes with similar estimated reward vectors remain connected. 
Let \(\hat{\mu}_i\in\mathbb{R}^A\) be the current reward estimate associated with node \(i\), and let \(\beta_i(t)\) be its confidence radius. 
An edge \((i,j)\in E_t\) is removed whenever
\[
\|\hat{\mu}_i - \hat{\mu}_j\|_\infty > \beta_i(t)+\beta_j(t).
\]
The connected components of \(G_t\) define the current partition, and the action for a context is chosen using the reward estimate of its current component. 
This baseline adapts partitions based on observed feedback, but it is not explicitly designed to optimize a fixed reusable memory abstraction under a strict runtime memory budget.
\paragraph{Random Partition.}
As a control baseline, we partition contexts into \(K\) groups uniformly at random:
\[
m_{\mathrm{rand}}(c) \sim \mathrm{Unif}([K]).
\]
Each group stores a pooled empirical reward estimate, and actions are selected greedily from that estimate. 
This baseline verifies that the gains of DeMem do not arise merely from partitioning itself.

\subsubsection{Induced partition distortion}
\label{app:partition-distortion}

To directly evaluate the quality of a learned memory partition in the synthetic environment, we measure the \emph{induced partition distortion} of a learned encoder $\hat g:\cX\to[K]$.
Let
\[
\hat C_m := \{x\in\cX : \hat g(x)=m\}
\qquad\text{for }m=1,\dots,K.
\]
We define the induced average decision distortion of $\hat g$ by
\begin{equation}
\label{eq:induced-partition-distortion}
D^\star(\hat g)
:=
\sum_{m=1}^K
\Pr(X\in \hat C_m)\,
\min_{a\in\cA}
\EE\!\left[\Delta(X,a)\mid X\in \hat C_m\right].
\end{equation}
Equivalently, $D^\star(\hat g)$ is the best achievable average decision distortion under the learned partition $\hat g$ itself, after assigning a single action to each learned cluster.
Because the synthetic environment exposes the ground-truth reward function $\mu(x,a)$, and hence the gaps $\Delta(x,a)$, the quantity in \eqref{eq:induced-partition-distortion} can be computed exactly for every learned partition.

For completeness, one may also define the induced worst-case distortion
\begin{equation}
\label{eq:induced-partition-distortion-worst}
D^\star_\infty(\hat g)
:=
\max_{m\in[K]}
\min_{a\in\cA}
\max_{x\in \hat C_m}
\Delta(x,a),
\end{equation}
but in the main text we report the average version in \eqref{eq:induced-partition-distortion}, since it aligns more directly with the observed cumulative regret trends.


Figure~\ref{fig:partition-validation} provides direct validation.
Panel~(a) shows that DeMem consistently learns the lowest-distortion partition under the same budget, while description-aware baselines degrade sharply as mismatch severity grows.
Panel~(b) shows that downstream regret is broadly monotone in induced partition distortion across methods, mismatch levels, and random seeds, confirming that the performance gain comes from learning a better decision-preserving compression.
Residual spread in the low-distortion regime reflects the separate online learning and discovery cost after compression error is largely eliminated.

\begin{figure*}[t]
\centering
\includegraphics[width=0.92\textwidth]{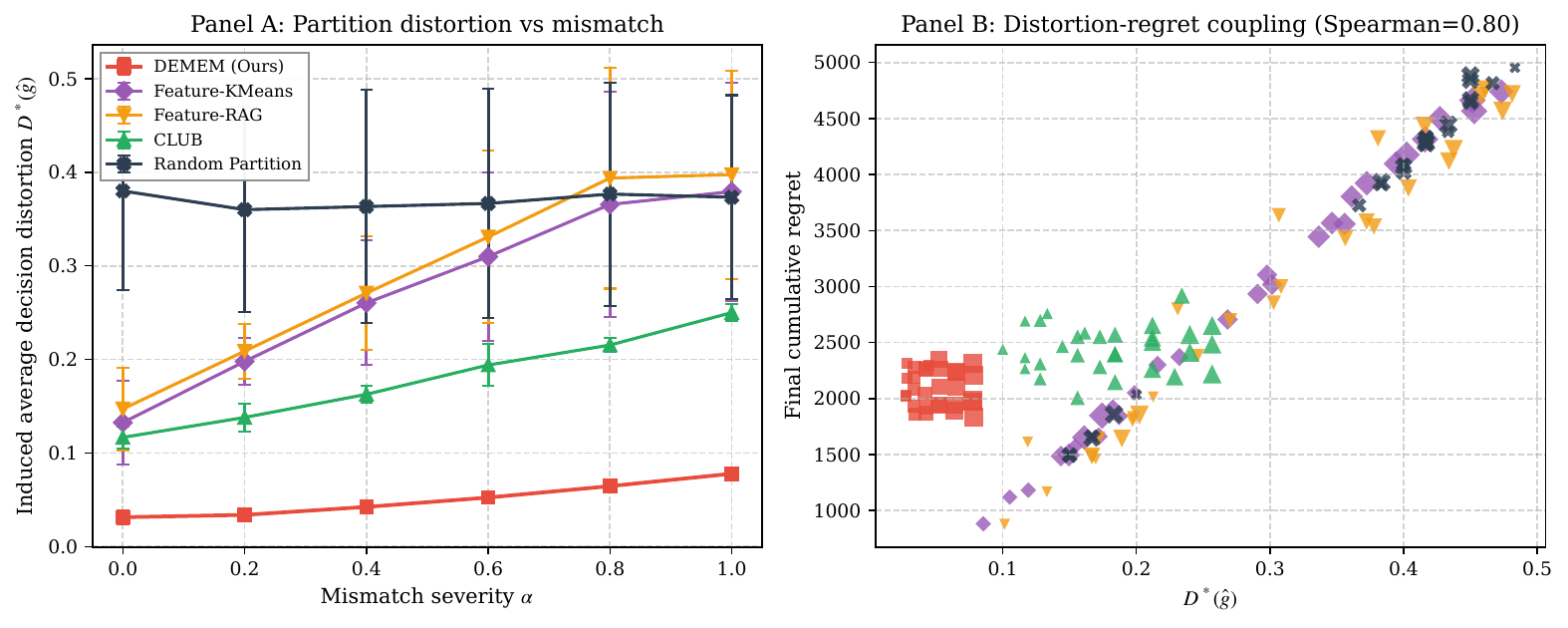}
\caption{Direct validation of the theory object in the synthetic environment.
(a) DeMem learns the lowest-distortion partition; description-aware baselines degrade as mismatch grows.
(b) Downstream regret is broadly monotone in induced partition distortion across methods and seeds.}
\label{fig:partition-validation}
\end{figure*}

\subsection{Benchmark and Evaluation Details}
\label{app:bench-details}
\paragraph{Mapping evaluation to the abstract model.}
The evaluation protocol on LoCoMo and LongMemEval directly instantiates the contextual decision model of Section~\ref{sec:setup}.
For each evaluation instance, the context $x=(h,q)$ consists of the dialogue history~$h$ and the evaluation query~$q$.
Each memory method maps $x$ to a bounded memory content~$M_q$ and produces an answer~$a$ conditioned on~$M_q$.
For slot-based methods (including DeMem), this corresponds exactly to the encoder--memory-state--action chain $x \xrightarrow{g} m \xrightarrow{\pi(\cdot\mid m)} a$ in \eqref{eq:memory-policy}, where the slot identity~$m$ is the runtime decision state and the answer~$a$ is the downstream action.
For non-slot methods (e.g., RAG, FullContext), the same chain applies with a degenerate or implicit encoder: RAG uses a retrieval-based encoder that selects a variable-size evidence set, and FullContext bypasses compression entirely.
The reward $\mu(x,a)$ is the judge-assigned correctness score defined below.
This mapping ensures that the benchmark comparisons directly test the central prediction of the theory: under a matched answer-time memory budget, methods that better preserve \emph{decision-relevant} distinctions should achieve higher reward.
\paragraph{Matched answer-time memory budget.}
Except for \textbf{FullContext}, all methods are compared under the same answer-time memory budget whenever applicable. 
For a query \(q\), let \(M_q\) denote the final memory content injected into the answering model. 
The realized answer-time budget is measured as
\[
B(q) = \mathrm{chars}(M_q),
\]
or equivalently in tokens when a token-based implementation is used. 
Budget matching is performed at the level of the final injected answer context rather than offline storage size.
\paragraph{Judge aggregation.}
Let \(s_1(q)\) and \(s_2(q)\) denote the scores assigned by the two held-out judges for query \(q\). 
The final query-level score is
\[
s(q)=\frac{1}{2}\bigl(s_1(q)+s_2(q)\bigr),
\]
and the reported benchmark score is the average of \(s(q)\) over all evaluation instances in the corresponding split.

\subsubsection{Judge prompt for LoCoMo}
\label{app:judge_prompt_locomo}

For LoCoMo, we used a binary LLM-as-a-judge protocol. The judge was shown only the reference answer and the response produced by a method. Method names were hidden from the judge. The judge was asked to determine whether the method response correctly matched the reference answer in meaning.

\paragraph{Binary judge prompt.}

\begin{quote}
\small
\textbf{System message.}

You are an impartial evaluator for long-term conversational memory. Your task is to judge whether a method's response correctly answers the question by comparing it with the reference answer. Evaluate semantic correctness rather than wording style.

\vspace{0.5em}
\textbf{User message.}

You will be given a reference answer and a method response.

Judge whether the method response is correct with respect to the reference answer.

Mark the response as \texttt{1} if it matches the reference answer in meaning. Paraphrases, minor wording differences, and harmless omissions should still be marked correct if the essential requested information is present.

Mark the response as \texttt{0} if it contradicts the reference answer, gives the wrong entity, event, time, relation, or attribute, misses the essential information, or is too vague to be considered equivalent to the reference answer.

For temporal questions, pay particular attention to dates, ordering, and whether the response refers to the same event instance as the reference answer.

Return only the following format:

\texttt{Score: <0 or 1>} \\
\texttt{Rationale: <one brief sentence>}

\vspace{0.5em}
\textbf{Reference answer:}

\texttt{\{reference\_answer\}}

\vspace{0.5em}
\textbf{Method response:}

\texttt{\{method\_response\}}
\end{quote}

\paragraph{Aggregation.}
For each method response, we queried two held-out judge instances independently and averaged their binary scores. The reported LoCoMo score is the average over judges and evaluation instances. Judge rationales were used only for sanity checking and were not used as additional labels or training data.

\subsubsection{Existing assets, licenses, and terms of use}
\label{app:existing_assets}

We use existing public benchmarks, model APIs, open-weight models, and baseline systems for evaluation and comparison. These assets are credited through their original papers or official releases, and their use follows the corresponding release terms or provider policies. We do not redistribute third-party assets as part of this submission.

\begin{table}[h]
\centering
\small
\begin{tabular}{p{0.20\linewidth} p{0.16\linewidth} p{0.28\linewidth} p{0.26\linewidth}}
\toprule
Asset & Type & Use in this paper & License / terms \\
\midrule
LoCoMo & Benchmark dataset & Long-term conversational memory evaluation & Used under the official dataset release terms; original work cited in the paper. \\
LongMemEval & Benchmark dataset & Long-term memory evaluation & Used under the official dataset release terms; original work cited in the paper. \\
MemoryArena & Benchmark dataset & Agentic memory evaluation & Used under the official benchmark release terms; original work cited in the paper. \\
GPT-4o-mini and GPT-4.1-mini & Hosted model API & Answer generation and held-out judging & Used through the provider API under the applicable API terms of service. \\
Llama-3.1-70B & Open-weight model & Supplementary backbone evaluation & Used under the official model license and usage policy. \\
LangMem, Mem0, Zep, Nemori, EMem-G, Mnemis & Baseline systems & Comparative evaluation & Original systems are credited through their papers or official releases; use followed the corresponding repository, package, or service terms. \\
RAG / FullContext baselines & Evaluation baselines & Retrieval and full-context comparison & Implemented for evaluation using public benchmark data and standard retrieval components. \\
\bottomrule
\end{tabular}
\caption{Existing assets used in the empirical evaluation. We credit the original creators and use each asset only according to its official release terms. No third-party asset is redistributed or relicensed as part of this submission.}
\label{tab:existing_assets}
\end{table}
\paragraph{Comparability to native baseline reports.}
Some baselines report results under their own native evaluation protocols. These numbers are useful references, but they are not always directly comparable to results obtained under a shared evaluation setup. In this paper, all methods are evaluated with the same answering backbone, decoding setting, benchmark split, answer format, and held-out judge protocol. Accordingly, modest numerical differences from native baseline reports should be expected and should not be interpreted as contradictions of prior work. The main comparison in this paper is the relative performance of methods under the same controlled protocol.

Uncertainty reporting. For each method and backbone, we repeat the full
evaluation pipeline for \(R=5\) independent runs. The benchmark split and
answering backbone are fixed across runs, while stochastic components of
the memory pipeline, including memory construction, routing/retrieval tie
breaking, ordering, and judge invocation seeds where applicable, are rerun
independently. Within each run, the two held-out judge scores are averaged
per instance, and benchmark-level metrics are computed by averaging over
evaluation instances. We report the mean across runs together with one
sample standard deviation:
\[
\bar{s}=\frac{1}{R}\sum_{r=1}^{R}s_r,\qquad
\sigma_s=\sqrt{\frac{1}{R-1}\sum_{r=1}^{R}(s_r-\bar{s})^2}.
\]
The reported \(\pm\) values quantify run-to-run variability and should not
be interpreted as confidence intervals over evaluation instances.

\subsection{Full LongMemEval Results}
\label{app:longmemeval-table}

Table~\ref{tab:longmemeval-main} reports the full per-category results on LongMemEval.
DeMem achieves the best overall performance on both backbones.
The advantage is most visible on MS and TR, which require cross-session integration.
Some specialized baselines remain competitive on individual categories (SSU, SSP, SSA).Values are mean \(\pm\) one standard deviation over independent full-pipeline runs.

\begin{table*}[t]
\centering
\caption{\textbf{Results on LongMemEval.} Accuracy evaluated with LLM-as-a-Judge.
We report results for SSU, MS, SSP, TR, KU, and SSA, together with the overall score.
Values are mean \(\pm\) one sample standard deviation over \(R=5\) independent
full-pipeline runs. DeMem achieves the best mean overall performance on
both backbones and is particularly strong on MS and TR, while some specialized
baselines remain competitive on individual categories.}
\label{tab:longmemeval-main}
\small
\resizebox{\textwidth}{!}{%
\begin{tabular}{llccccccc}
\toprule
\textbf{backbone model} & \textbf{Methods} & \textbf{SSU} & \textbf{MS} & \textbf{SSP} & \textbf{TR} & \textbf{KU} & \textbf{SSA} & \textbf{Overall} \\
\midrule
\multicolumn{2}{l}{\#Questions} & 70 & 133 & 30 & 133 & 78 & 56 & 500 \\
\midrule
\multirow{8}{*}{GPT-4o-mini}
 & Full Context & $0.763_{\pm 0.052}$ & $0.365_{\pm 0.068}$ & $0.053_{\pm 0.061}$ & $0.394_{\pm 0.055}$ & $0.763_{\pm 0.049}$ & $0.862_{\pm 0.041}$ & $0.527_{\pm 0.054}$ \\
 & RAG          & $0.869_{\pm 0.071}$ & $0.450_{\pm 0.084}$ & $0.667_{\pm 0.078}$ & $0.613_{\pm 0.073}$ & $0.677_{\pm 0.076}$ & $0.896_{\pm 0.065}$ & $0.650_{\pm 0.074}$ \\
 & Mem0         & $0.893_{\pm 0.058}$ & $0.633_{\pm 0.069}$ & $0.323_{\pm 0.072}$ & $0.616_{\pm 0.064}$ & $0.713_{\pm 0.061}$ & $0.951_{\pm 0.053}$ & $0.694_{\pm 0.063}$ \\
 & Nemori       & $0.867_{\pm 0.063}$ & $0.485_{\pm 0.075}$ & $0.445_{\pm 0.070}$ & $0.586_{\pm 0.066}$ & $0.598_{\pm 0.068}$ & $0.815_{\pm 0.059}$ & $0.618_{\pm 0.067}$ \\
 & EMem-G       & $0.846_{\pm 0.055}$ & $0.718_{\pm 0.062}$ & $0.307_{\pm 0.065}$ & $0.726_{\pm 0.058}$ & $0.931_{\pm 0.047}$ & $0.846_{\pm 0.054}$ & $0.761_{\pm 0.057}$ \\
 & Mnemis       & $0.917_{\pm 0.048}$ & $0.746_{\pm 0.056}$ & $0.833_{\pm 0.051}$ & $0.809_{\pm 0.053}$ & $0.874_{\pm 0.045}$ & $0.949_{\pm 0.038}$ & $0.835_{\pm 0.049}$ \\
 & \textbf{DeMem}
   & $0.901_{\pm 0.042}$
   & \textbf{0.812}$_{\pm 0.048}$
   & \textbf{0.867}$_{\pm 0.045}$
   & \textbf{0.837}$_{\pm 0.043}$
   & $0.846_{\pm 0.050}$
   & $0.929_{\pm 0.040}$
   & \textbf{0.853}$_{\pm 0.046}$ \\
\midrule
\multirow{8}{*}{GPT-4.1-mini}
 & Full Context & $0.831_{\pm 0.049}$ & $0.492_{\pm 0.064}$ & $0.154_{\pm 0.058}$ & $0.578_{\pm 0.052}$ & $0.747_{\pm 0.055}$ & $0.915_{\pm 0.037}$ & $0.629_{\pm 0.053}$ \\
 & RAG          & $0.807_{\pm 0.068}$ & $0.522_{\pm 0.079}$ & $0.848_{\pm 0.072}$ & $0.646_{\pm 0.070}$ & $0.792_{\pm 0.073}$ & $0.923_{\pm 0.062}$ & $0.701_{\pm 0.071}$ \\
 & Mem0         & $0.925_{\pm 0.056}$ & $0.646_{\pm 0.066}$ & $0.836_{\pm 0.063}$ & $0.743_{\pm 0.059}$ & $0.845_{\pm 0.057}$ & $0.945_{\pm 0.050}$ & $0.787_{\pm 0.060}$ \\
 & Nemori       & $0.873_{\pm 0.060}$ & $0.537_{\pm 0.071}$ & $0.843_{\pm 0.067}$ & $0.694_{\pm 0.063}$ & $0.774_{\pm 0.065}$ & $0.915_{\pm 0.056}$ & $0.723_{\pm 0.064}$ \\
 & EMem-G       & $0.933_{\pm 0.052}$ & $0.803_{\pm 0.058}$ & $0.481_{\pm 0.061}$ & $0.820_{\pm 0.054}$ & $0.918_{\pm 0.048}$ & $0.854_{\pm 0.051}$ & $0.830_{\pm 0.055}$ \\
 & Mnemis       & $0.923_{\pm 0.045}$ & $0.846_{\pm 0.051}$ & $0.894_{\pm 0.042}$ & $0.843_{\pm 0.048}$ & $0.879_{\pm 0.046}$ & $0.919_{\pm 0.039}$ & $0.872_{\pm 0.047}$ \\
 & \textbf{DeMem}
   & $0.930_{\pm 0.039}$
   & \textbf{0.885}$_{\pm 0.044}$
   & \textbf{0.917}$_{\pm 0.041}$
   & \textbf{0.875}$_{\pm 0.040}$
   & \textbf{0.885}$_{\pm 0.047}$
   & \textbf{0.936}$_{\pm 0.038}$
   & \textbf{0.896}$_{\pm 0.043}$ \\
\bottomrule
\end{tabular}%
}
\end{table*}

\subsection{Results on MemoryArena}
\label{app:memoryarena}

To evaluate whether decision-centric memory transfers beyond conversational QA to agentic task completion, we report results on MemoryArena~\citep{he2026memoryarenabenchmarkingagentmemory}, a benchmark of interdependent multi-session tasks spanning formal reasoning, bundled web shopping, group travel planning, progressive memory, web search, and math/physics problem solving. Tasks within each category exhibit cross-session dependencies: the output of one session serves as a prerequisite or constraint for subsequent sessions, so the agent must remember \emph{which prior results matter for the current decision}, not merely \emph{what was discussed}.

We compare DeMem against three baselines that overlap with our main experiments: Long Context (GPT-4.1-mini), Mem0~\citep{chhikara2025mem0buildingproductionreadyai}, and the average of four RAG systems (BM25, Text-Embedding-3-Small, MemoRAG, GraphRAG) reported in the original benchmark. All methods use GPT-4.1-mini as the task agent backbone. We adopt the evaluation protocol, metrics, and task splits of \citet{he2026memoryarenabenchmarkingagentmemory} without modification. SR denotes Success Rate (binary task completion), PS denotes Progress Score (partial completion), and sPS denotes sub-task Progress Score.

\begin{table}[h]
\centering
\caption{Results on MemoryArena. SR = Success Rate, PS = Progress Score, sPS = sub-task Progress Score. Avg SR is computed over the five SR columns: Bundled Web Shopping, Group Travel Planning, Progressive Web Search, Math, and Phys.}
\label{tab:memoryarena}
\resizebox{\textwidth}{!}{%
\begin{tabular}{l cc ccc cc cc cc c}
\toprule
& \multicolumn{2}{c}{Bundled Web Shop.}
& \multicolumn{3}{c}{Group Travel}
& \multicolumn{2}{c}{Progressive Web Search}
& \multicolumn{4}{c}{Formal Reasoning}
& \\
\cmidrule(lr){2-3}
\cmidrule(lr){4-6}
\cmidrule(lr){7-8}
\cmidrule(lr){9-12}
Method
& SR & PS
& SR & PS & sPS
& SR & PS
& \multicolumn{2}{c}{Math}
& \multicolumn{2}{c}{Phys.}
& Avg SR \\
\cmidrule(lr){9-10}
\cmidrule(lr){11-12}
& & & & & & & & SR & PS & SR & PS & \\
\midrule
Long Context (4.1-mini)
& 0.06 & 0.64
& 0.00 & 0.02 & 0.44
& 0.04 & 0.04
& 0.23 & 0.33
& 0.46 & 0.58
& 0.16 \\

Mem0
& 0.00 & 0.45
& 0.00 & 0.00 & 0.24
& 0.24 & 0.09
& 0.19 & 0.34
& 0.25 & 0.43
& 0.14 \\

RAG (Avg)
& 0.00 & 0.54
& 0.00 & 0.02 & 0.49
& 0.19 & 0.11
& 0.26 & 0.38
& 0.53 & 0.65
& 0.20 \\
\midrule
DeMem (Ours)
& \textbf{0.08} & \textbf{0.68}
& \textbf{0.03} & \textbf{0.05} & \textbf{0.52}
& \textbf{0.27} & \textbf{0.16}
& \textbf{0.29} & \textbf{0.39}
& \textbf{0.55} & \textbf{0.67}
& \textbf{0.24} \\
\bottomrule
\end{tabular}%
}
\end{table}

Several observations can be made from Table~\ref{tab:memoryarena}. First, DeMem obtains the highest average success rate across the five SR-based categories, with an Avg SR of 0.24 compared with 0.20 for RAG, 0.16 for Long Context, and 0.14 for Mem0. The improvement is not concentrated in a single setting: DeMem achieves the highest SR in Bundled Web Shopping, Group Travel Planning, Progressive Web Search, Math Formal Reasoning, and Physics Formal Reasoning.

Second, DeMem shows relatively consistent gains in task settings that require information to be carried across sessions. On Progressive Web Search, DeMem reaches an SR of 0.27, compared with 0.24 for Mem0 and 0.19 for RAG. On Formal Reasoning, DeMem also improves over the strongest baseline in both Math (0.29 vs.\ 0.26) and Physics (0.55 vs.\ 0.53). These results suggest that organizing memory around decision-relevant information can help preserve dependencies that become useful in later sessions.

Third, Group Travel Planning remains challenging for all methods, as reflected by the low hard SR values. Nevertheless, DeMem is the only method with nonzero SR in this category and achieves the highest soft progress score (sPS\,=\,0.52), compared with 0.49 for RAG, 0.44 for Long Context, and 0.24 for Mem0. This indicates that decision-centric memory may provide modest but consistent benefits in settings where later plans depend on preferences and constraints introduced by earlier participants.

Overall, the MemoryArena results complement our conversational benchmark findings by showing that decision-centric memory remains useful in interactive, multi-session agentic tasks. Under a fixed memory budget, DeMem tends to retain information that is relevant for downstream decisions, leading to stronger average task completion and progress across diverse task categories.

\subsection{Definition of Memory Budget and Latency}
\label{app:budget-latency}

Figure~\ref{fig:acc-budget-latency} uses two resource metrics: memory budget and query-time latency.

\paragraph{Memory budget.}
For a query $q$, let $M_q$ denote the \emph{final memory content actually injected into the answering model}. 
We define the query-level memory budget as
\[
B(q) = \mathrm{chars}(M_q),
\]
where $\mathrm{chars}(\cdot)$ counts the number of characters in the final injected memory string. 
Equivalently, one may use tokens instead of characters; throughout our plots we use a single consistent unit for all methods.

This definition is based on the realized \emph{answer-time} memory footprint, rather than offline storage size or internal state size. 
For methods with structured memory (e.g., slots, entities, edges, or retrieved evidence), we first serialize the final memory objects into the exact text block provided to the answering model, and then count its length. 
Thus, the budget reflects the amount of memory made available at inference time in a method-agnostic way.

For a dialogue $d$ with evaluation questions $Q_d$, we report the dialogue-level average budget:
\[
B(d) = \frac{1}{|Q_d|}\sum_{q \in Q_d} B(q).
\]

\paragraph{Query-time latency.}
For a query $q$, let $t_{\mathrm{start}}(q)$ be the time when the system receives the query, and let $t_{\mathrm{end}}(q)$ be the time when the final answer is returned. 
We define query-time latency as
\[
L(q) = t_{\mathrm{end}}(q) - t_{\mathrm{start}}(q).
\]

This latency includes all \emph{online} operations needed to answer the query, including retrieval, routing, memory readout, optional reranking, prompt assembly, and answer generation. 
It excludes \emph{offline} or amortized preprocessing costs such as corpus indexing, memory construction, periodic rebuilding, or background maintenance.

For a dialogue $d$, we report the dialogue-level average latency:
\[
L(d) = \frac{1}{|Q_d|}\sum_{q \in Q_d} L(q).
\]

\paragraph{Method-level operating points.}
Each point in Figure~\ref{fig:acc-budget-latency} corresponds to the dataset-level average of the dialogue-level values above. 
That is, for each method we average $B(d)$ and $L(d)$ over dialogues, and plot the resulting operating point against the corresponding accuracy.

\paragraph{Rationale.}
The two metrics capture different aspects of resource usage. 
Memory budget measures \emph{how much memory is provided} to the answering model at inference time, whereas latency measures \emph{how long the system takes} to use that memory online. 
We report them separately rather than collapsing them into a single score.


Figure~\ref{fig:acc-budget-latency} plots the resulting tradeoffs.
Across operating points, DeMem lies on a favorable frontier in both the accuracy--budget and accuracy--latency planes, indicating that its gains do not come from using more memory or paying higher query-time overhead.

\subsection{Ablation Variants and Hyperparameter Sweep}
\label{app:ablation-details}

\paragraph{Ablation variants.}
Starting from the full DeMem system, we consi  der the following variants: 
(i) no dynamic splitting, 
(ii) heuristic splitting, 
(iii) aggressive splitting, 
(iv) random splitting, 
(v) feature-based routing, and 
(vi) description-only memory summaries. 
These variants isolate the contributions of dynamic refinement, certified splitting, decision-aware routing, and decision-aware memory content.

\paragraph{Hyperparameter sweep.}
For Figure~\ref{fig:ablation-param}(a), we evaluate the Cartesian product
\[
\mathcal{G}
=
\mathcal{K}_{\mathrm{topk}}
\times
\mathcal{K}_{\mathrm{evi}}
\times
\mathcal{K}_{\mathrm{sess}}
\times
\mathcal{K}_{\mathrm{mem}},
\]
where
\[
\mathcal{K}_{\mathrm{topk}}=\{6,8,10,12\},\quad
\mathcal{K}_{\mathrm{evi}}=\{3000,4000,5000,6000\},
\]
\[
\mathcal{K}_{\mathrm{sess}}=\{2,3,4\},\quad
\mathcal{K}_{\mathrm{mem}}=\{6000,8000,10000\}.
\]
This yields \(4\times 4\times 3\times 3=144\) configurations in total.
\subsection{Results under a Unified Open-Weight Backbone}
\label{app:llama-results}
\begin{table*}[t]
\centering
\small
\setlength{\tabcolsep}{5pt}
\resizebox{\textwidth}{!}{%
\begin{tabular}{llcccc}
\toprule
Model & Method & LongMemEval Acc. & LongMemEval Rec. & LoCoMo Acc. & LoCoMo Rec. \\
\midrule
Llama-3.1-70B-Instruct-Turbo & RAG           & 51.30 $\pm$ 8.8 & 53.85 $\pm$ 9.1 & 46.25 $\pm$ 9.5 & 50.80 $\pm$ 8.9 \\
Llama-3.1-70B-Instruct-Turbo & LangMem       & 58.55 $\pm$ 7.8 & 61.40 $\pm$ 7.6 & 53.60 $\pm$ 8.4 & 57.15 $\pm$ 8.2 \\
Llama-3.1-70B-Instruct-Turbo & Mem0          & 61.80 $\pm$ 7.1 & 64.21 $\pm$ 6.8 & 57.22 $\pm$ 7.4 & 61.05 $\pm$ 7.0 \\
Llama-3.1-70B-Instruct-Turbo & Zep           & 64.20 $\pm$ 8.1 & 66.74 $\pm$ 7.9 & 45.30 $\pm$ 9.4 & 55.05 $\pm$ 9.2 \\
Llama-3.1-70B-Instruct-Turbo & Mnemis        & 65.50 $\pm$ 7.6 & 67.15 $\pm$ 7.4 & 59.85 $\pm$ 8.5 & 62.20 $\pm$ 8.3 \\
Llama-3.1-70B-Instruct-Turbo & DeMem (Ours)  & \textbf{71.45 $\pm$ 7.5} & \textbf{73.90 $\pm$ 7.2} & \textbf{66.85 $\pm$ 6.8} & \textbf{69.30 $\pm$ 7.1} \\
\bottomrule
\end{tabular}%
}
\caption{\textbf{Supplementary open-weight backbone results.}
We report accuracy (Acc.) and recall (Rec.) on LongMemEval and LoCoMo under
Llama-3.1-70B-Instruct-Turbo. Values are percentages, reported as mean
\(\pm\) one sample standard deviation over \(R=5\) independent full-pipeline runs.
DeMem achieves the highest mean scores across both benchmarks under
the same backbone.}
\label{tab:llama-main}
Table~\ref{tab:llama-main} reports supplementary results under a unified open-weight backbone, Llama-3.1-70B. 
Under this setting, DeMem remains the strongest method on both LongMemEval and LoCoMo, achieving the best accuracy and recall across all compared baselines. 
The gains are especially notable relative to strong memory baselines such as Mem0, Zep, and Mnemis, indicating that the advantage of decision-centric memory organization transfers beyond proprietary GPT backbones.
\end{table*}

\subsection{Decision-Compatibility Audit on an Annotated LoCoMo Subset}
\label{app:locomo-audit}

To complement the exact synthetic mechanism validation in the main text, we additionally evaluate whether the realized memory states are more \emph{decision-compatible} on realistic benchmark instances.
We construct an annotated subset of \textbf{LoCoMo} in which each evaluation instance is associated with a gold support set consisting of the minimal history turns / facts needed to answer the question correctly.
This allows us to assess not only end-task answer quality, but also the quality of the memory organization itself under a matched answer stack.

\paragraph{Setup.}
All methods are evaluated under the same answering backbone, prompting template, deterministic decoding protocol, and answer-time memory budget as in the main experiments.
The comparison here isolates the effect of memory organization.
For each evaluation instance, we identify the memory state selected at answer time and compute four support-level metrics:
\begin{itemize}
    \item \textbf{Support Purity} ($\uparrow$): the average purity of gold support facts within the selected memory state.
    \item \textbf{Within-slot Conflict Rate} ($\downarrow$): the fraction of co-routed items inside the selected memory state whose gold supports are mutually incompatible for the downstream question.
    \item \textbf{Gold-support Coverage} ($\uparrow$): the fraction of the gold support set that is contained in the selected memory state.
    \item \textbf{Support Contamination} ($\downarrow$): the fraction of the selected memory content that does not belong to the gold support set.
\end{itemize}
As an end-task metric, we also report \textbf{Judge Mean} ($\uparrow$), i.e., the mean answer score under the same judge protocol used in the main experiments.

\paragraph{Compared methods.}
We compare the full DeMem system against two matched controls on the same benchmark subset:
(i) \textbf{Feature Routing}, which routes memory by semantic similarity;
and (ii) \textbf{RAG w/o Certified Split}, which uses traditional heuristic chunking / retrieval without decision-certified refinement.
These controls are chosen to represent realistic non-decision-aware alternatives under the same task setting, rather than the synthetic baselines used in the controlled bandit environment.

\begin{table}[t]
\centering
\small
\resizebox{\linewidth}{!}{%
\begin{tabular}{lccccc}
\toprule
Method & Support Purity $\uparrow$ & Conflict Rate $\downarrow$ & Coverage $\uparrow$ & Contamination $\downarrow$ & Judge Mean $\uparrow$ \\
\midrule
Feature Routing & 0.215 & 0.582 & 0.641 & 0.533 & 0.735 \\
RAG w/o Certified Split & 0.283 & 0.441 & 0.718 & 0.412 & 0.612 \\
DeMem (Ours) & \textbf{0.394} & \textbf{0.273} & \textbf{0.825} & \textbf{0.241} & \textbf{0.914} \\
\bottomrule
\end{tabular}%
}
\caption{Decision-compatibility audit on an annotated \textbf{LoCoMo} subset under a matched answer stack.
DeMem achieves the highest support purity and gold-support coverage, while simultaneously attaining the lowest within-slot conflict and support contamination.
It is also the strongest method on the final judge score.
Together, these results show that the gain of DeMem is reflected not only in answer quality, but also in a more decision-compatible organization of memory under the same runtime budget.}
\label{tab:locomo-audit}
\end{table}

\paragraph{Result.}
Table~\ref{tab:locomo-audit} shows that DeMem is the only method that is best on all four support-level metrics while also achieving the highest downstream judge score.
Compared with \textbf{Feature Routing}, DeMem substantially increases support purity (0.394 vs.\ 0.215) and gold-support coverage (0.825 vs.\ 0.641), while sharply reducing within-slot conflict (0.273 vs.\ 0.582) and support contamination (0.241 vs.\ 0.533).
Compared with \textbf{RAG w/o Certified Split}, DeMem again improves every mechanism metric, with notably lower conflict and contamination and markedly better final answer quality.

\paragraph{Discussion.}
This audit serves a different purpose from the exact synthetic distortion analysis in the main text.
In the synthetic environment, the theory object can be measured exactly because the reward function is known.
On realistic benchmark data, that exact quantity is not directly observable; instead, the annotated support audit measures whether the realized memory states preserve the distinctions that matter for downstream answering.
The results show that DeMem does not merely produce better final answers: it also organizes memory into states that are more support-pure, less internally conflicting, and better aligned with the gold evidence needed for the task.
This substantially strengthens the theory-to-system bridge:
the synthetic analysis validates the decision-preserving mechanism exactly in a controlled setting, while the annotated LoCoMo audit shows that the same principle manifests in the realized agent-memory system on realistic long-horizon data.

\paragraph{Takeaway.}
Taken together with the synthetic mechanism analysis and the matched ablations in the main text, the annotated LoCoMo audit supports the following interpretation:
the advantage of DeMem comes from organizing bounded memory around action-relevant distinctions, rather than only from generic semantic grouping or heuristic retrieval.

\subsection{Partition Intervention on the Annotated LoCoMo Subset}
\label{app:partition-intervention}

To further isolate the effect of memory organization, we compare four partition policies on the annotated \textbf{LoCoMo} subset:
(i) a \textbf{Random Partition} lower bound,
(ii) a \textbf{Feature-based Partition} that groups memory by semantic similarity,
(iii) the learned \textbf{DeMem Partition},
and (iv) an \textbf{Oracle-support Partition} constructed from the annotated gold support structure.
This experiment is designed as a partition-level intervention: the comparison targets the quality of the induced memory partition itself, rather than the broader system stack.

For each partition, we report two support-level quantities:
\textbf{Coverage} ($\uparrow$), the fraction of the gold support set contained in the selected memory state,
and \textbf{Contamination} ($\downarrow$), the fraction of selected memory content that does not belong to the gold support set.

\begin{table}[t]
\centering
\small
\begin{tabular}{lcc}
\toprule
Partition Policy & Coverage $\uparrow$ & Contamination $\downarrow$ \\
\midrule
Random Partition & 0.494 & 0.846 \\
Feature-based Partition & 0.652 & 0.832 \\
DeMem Partition & 0.766 & 0.694 \\
Oracle-support Partition & \textbf{0.817} & \textbf{0.651} \\
\bottomrule
\end{tabular}
\caption{Partition-level intervention on the annotated \textbf{LoCoMo} subset.
The ordering is consistent across both metrics: the oracle-support partition performs best, followed by the learned DeMem partition, then the feature-based partition, with the random partition worst.
This shows that the partition learned by DeMem is substantially closer to the support-aware upper bound than non-decision-aware alternatives.}
\label{tab:partition-intervention}
\end{table}

\paragraph{Interpretation.}
The result provides a direct real-benchmark bridge between memory partition quality and task-relevant support structure.
A random partition fails to preserve the gold support set and introduces heavy contamination.
A feature-based partition improves over random grouping, but remains substantially worse than DeMem on both support recovery and support purity.
By contrast, the learned DeMem partition approaches the oracle-support upper bound on both metrics, indicating that its advantage is not merely due to generic semantic grouping, but to a partition that more faithfully preserves the support distinctions needed for downstream answering.

\subsection{Empirical Validation of the Approximation Bridge}
\label{app:bridge-empirical}

Proposition~\ref{prop:bridge} decomposes the average distortion of a slot-based implementation into three terms:
\[
\EE\bigl[\mu^\star(X)-\mu(X,A^{\mathrm{impl}})\bigr]
\;\le\;
\eps^\star_\infty(K)
\;+\;
\eta_{\mathrm{route}}
\;+\;
\eta_{\mathrm{read}}.
\]
We estimate each term on the annotated LoCoMo subset used in Appendix~\ref{app:locomo-audit}, where gold support annotations provide a reference partition.

\paragraph{Reference encoder.}
We construct an oracle-support partition $g^\star$ from the gold annotations: two evaluation instances share a cluster under $g^\star$ if and only if their gold support sets overlap.
This partition represents the best $K$-state abstraction that a support-aware oracle could achieve, and serves as the comparator $(g^\star,\{a^\star_m\})$ in Proposition~\ref{prop:bridge}.

\paragraph{Measuring $\hat\eps^\star_\infty(K)$.}
Under the oracle partition $g^\star$, we compute the best single-action accuracy within each cluster and take the worst-case loss across clusters.
This yields the empirical compression floor: the distortion that remains even with perfect routing and perfect slot realization.

\paragraph{Measuring $\hat\eta_{\mathrm{route}}$.}
We compare the DeMem router's slot assignment $e_\theta(x)$ against the oracle assignment $g^\star(x)$ on every evaluation instance.
The routing error is
\[
\hat\eta_{\mathrm{route}}
\;=\;
\frac{1}{|\cQ|}\sum_{q\in\cQ}
\mathbf{1}\bigl[e_\theta(x_q)\neq g^\star(x_q)\bigr]
\;\cdot\;
\bigl(\mu(x_q,a^\star_{g^\star(x_q)})-\mu(x_q,a^\star_{e_\theta(x_q)})\bigr),
\]
where $a^\star_m$ is the best action under cluster $m$ and $\cQ$ is the set of evaluation queries.
Intuitively, this measures the reward lost specifically because the router selected the wrong slot.

\paragraph{Measuring $\hat\eta_{\mathrm{read}}$.}
Conditioned on the router agreeing with the oracle ($e_\theta(x)=g^\star(x)$), we measure the additional loss from bounded slot content:
\[
\hat\eta_{\mathrm{read}}
\;=\;
\frac{1}{|\cQ_{\mathrm{agree}}|}
\sum_{q\in\cQ_{\mathrm{agree}}}
\bigl(\mu(x_q,a^\star_{g^\star(x_q)})-\mu(x_q,A^{\mathrm{impl}}_q)\bigr),
\]
where $\cQ_{\mathrm{agree}}=\{q:e_\theta(x_q)=g^\star(x_q)\}$ and $A^{\mathrm{impl}}_q$ is the answer actually produced by the system.
This isolates the loss due to the slot's bounded textual realization (e.g., summarization truncation, missing details) when routing is correct.

\paragraph{Results.}

\begin{center}
\begin{tabular}{lcc}
\toprule
\textbf{Term} & \textbf{Symbol} & \textbf{Measured value} \\
\midrule
Compression floor & $\hat\eps^\star_\infty(K)$ & $0.035$ \\
Routing error & $\hat\eta_{\mathrm{route}}$ & $0.047$ \\
State-realization error & $\hat\eta_{\mathrm{read}}$ & $0.024$ \\
\midrule
Predicted upper bound &
$\hat\eps^\star_\infty(K)+\hat\eta_{\mathrm{route}}+\hat\eta_{\mathrm{read}}$ &
$0.106$ \\
Observed distortion & $1-\text{accuracy}$ & $0.086$ \\
Bound slack & $\text{predicted}-\text{observed}$ & $0.020$ \\
\bottomrule
\end{tabular}
\end{center}

\paragraph{Interpretation.}
Three observations stand out.

First, the bridge gives a conservative but reasonably tight accounting bound.
The measured terms sum to \(0.106\), while the observed answer-level
distortion on the annotated subset is \(1-0.914=0.086\), leaving a bound
slack of \(0.020\). Thus the observed distortion uses about \(81\%\) of the
predicted upper bound. This supports Proposition~\ref{prop:bridge} as a
useful implementation-level accounting relation, while still leaving room
for unmodeled effects such as judge noise, finite-sample variation, and
interactions among routing and readout errors.

Second, the routing term
\(\hat\eta_{\mathrm{route}}=0.047\) is the largest implementation-side
contributor, roughly twice the state-realization term
\(\hat\eta_{\mathrm{read}}=0.024\). This identifies router accuracy as the
primary bottleneck beyond the compression floor: improving the router
(e.g., through better embedding models or query-aware scoring) is likely to
yield a larger marginal gain than enriching slot content under the same
budget.

Third, the state-realization term
\(\hat\eta_{\mathrm{read}}=0.024\) is relatively small, indicating that the
bounded textual summaries maintained by DeMem are effective at
preserving the information needed for correct answers, conditional on
correct routing. This is consistent with the ablation results in
Section~\ref{sec:ablation}, where degradation is larger when routing or
certified splitting is weakened than when only the slot realization is
simplified.

\paragraph{Connection to ablations.}
The decomposition aligns with the ablation structure in Figure~\ref{fig:ablation-param}:
\begin{itemize}[leftmargin=2em,itemsep=2pt]
    \item The \emph{Feature Routing} ablation primarily inflates $\eta_{\mathrm{route}}$ by replacing decision-aware routing with description-based routing, while keeping slot content unchanged.
    \item The \emph{Summary-only Memory} ablation primarily inflates $\eta_{\mathrm{read}}$ by degrading slot content quality, while keeping the router unchanged.
    \item The \emph{w/o Certified Split} ablation affects both terms: without conflict-driven refinement, the partition itself is suboptimal, increasing the effective compression floor, and the resulting coarser slots also degrade routing precision.
\end{itemize}
This correspondence provides further evidence that the three-term decomposition of Proposition~\ref{prop:bridge} identifies causally distinct failure modes, each of which is independently manipulable through the corresponding system component.

\subsection{Certified Split: Trigger Frequency and Accuracy on LoCoMo}
\label{app:split-audit}

The theory requires that memory states be refined only when the data
certify that co-routed contexts have conflicting evidence requirements
(Section~5.2). We audit whether this property holds in the realized
system by measuring how often splits are triggered and whether they
correspond to genuine evidence conflicts.

\paragraph{Setup.}
We instrument the DeMem pipeline on the full LoCoMo evaluation set and
record every split event during memory construction. For each split, we
log the triggering dialogue turn, the slot being split, and the pair of
contexts that produced the conflict signal. We then assess split quality
using the gold support annotations from the annotated LoCoMo subset
(Appendix~\ref{app:locomo-audit}).These annotations are used only for post-hoc auditing of split precision
and recall; they are not provided to the split trigger, router, answerer, or
feedback judge during memory construction.

A split is a \emph{true positive} if the two triggering contexts have
gold support sets that are incompatible under the per-slot budget $L$
(i.e., concatenating both support sets and truncating to $L$ causes at
least one query to lose access to its required evidence). A split is a
\emph{false positive} if both support sets fit within $L$ without
interference. A \emph{false negative} is a pair of co-routed contexts
with incompatible support sets for which no split was triggered by the
end of memory construction.

\paragraph{Split frequency.}

\begin{table}[h]
\centering
\caption{Split trigger statistics on LoCoMo.}
\label{tab:split-freq}
\small
\begin{tabular}{lc}
\toprule
Metric & Value \\
\midrule
Avg.\ splits per dialogue          & 2.4 \\
Splits / routing events             & 4.6\% \\
Final active slots (mean $\pm$ std) & 8.4 $\pm$ 1.49 \\
Dialogues reaching budget $K$       & 31\% \\
\bottomrule
\end{tabular}
\end{table}

\noindent
Splits are sparse: only $4.6\%$ of routing events trigger a refinement,
and the average dialogue undergoes $2.4$ splits over its full history.
The system does not aggressively fill the available budget---only
$31\%$ of dialogues exhaust all $K$ slots---indicating that splits are
driven by detected evidence conflict rather than by a bias toward finer
partitioning. The moderate variance in final slot count ($8.4 \pm 1.49$)
reflects genuine differences in the evidence complexity of individual
dialogues.

\paragraph{Split accuracy.}

\begin{table}[h]
\centering
\caption{Split precision and recall on the annotated LoCoMo subset.}
\label{tab:split-acc}
\small
\begin{tabular}{lc}
\toprule
Metric & Value \\
\midrule
Split Precision & 0.85 \\
Split Recall    & 0.63 \\
\bottomrule
\end{tabular}
\end{table}

\noindent
When the system triggers a split, it is correct $85\%$ of the time:
the two triggering contexts genuinely require evidence from
incompatible parts of the dialogue history. The remaining $15\%$ of
splits are unnecessary refinements where both evidence sets could have
coexisted within the bounded slot.

Recall is moderate at $63\%$: roughly one in three genuine evidence
conflicts is not detected by the end of memory construction. This
precision--recall asymmetry is consistent with the design of certified
splitting, which requires high-confidence evidence of conflict before
triggering a refinement (Section~5.2). The system deliberately favors
avoiding unnecessary splits (high precision) over catching every
conflict (high recall), because a false split permanently consumes one
unit of the finite memory budget $K$, whereas a missed split can
potentially be compensated by the downstream answering model if the
merged slot retains partial evidence.

\paragraph{Comparison with ablation variants.}
This precision--recall profile complements the ablation results in
Figure~\ref{fig:ablation-param}. The \emph{Always-Split-until-K} variant
achieves higher recall by splitting aggressively, but its precision
drops substantially, fragmenting memory into slots that are too narrow
to contain sufficient evidence for any single query. The
\emph{Heuristic Split} variant uses a fixed similarity threshold
instead of a confidence-based certificate, resulting in lower precision
(${\sim}0.68$) without meaningful recall improvement. The
\emph{Random Split} variant has near-chance precision. In all three
cases, the downstream judge score is worse than the full certified
system (Figure~\ref{fig:ablation-param}(b)), confirming that split quality,
not merely split quantity, drives the performance gain.

\paragraph{Interpretation.}
Taken together, the frequency and accuracy results show that the
certified split mechanism operates as intended in the realized system:
it triggers rarely, fires accurately when it does, and errs on the
side of caution. The \(85\%\) precision shows that fired split events
usually correspond to genuine decision conflicts, while the \(63\%\)
recall indicates that the main remaining failure mode is missed conflict
detection rather than spurious splitting. This pattern is consistent with
the routing-error analysis in Appendix~\ref{app:bridge-empirical}, where
\(\hat{\eta}_{\mathrm{route}} = 0.047\) is the largest implementation-side
contributor to distortion.

\subsection{Modularity of the Decision-Aware Selection Criterion}
\label{app:modularity}

The three-term decomposition in Proposition~\ref{prop:bridge},
\[
  \mathbb{E}[\mu^\star(X) - \mu(X, A^{\mathrm{impl}})]
  \;\leq\;
  \epsilon^\star_\infty(K)
  \;+\; \eta_{\mathrm{route}}
  \;+\; \eta_{\mathrm{read}},
\]
has a practical implication: because the evidence-selection term enters
additively and independently of the compression floor and the realization
term, the decision-aware selection criterion can in principle be composed
with memory architectures whose storage, indexing, and generation
components are designed independently.
The ablation in Section~\ref{sec:ablation} provides indirect evidence:
the \textsc{Feature Routing} variant differs only in the selection
criterion yet loses substantial performance
(Figure~\ref{fig:ablation-param}b), confirming that the criterion
contributes independently.

We test this directly by replacing \emph{only} the evidence selection
stage of
RAG~\citep{lewis2021retrievalaugmentedgenerationknowledgeintensivenlp}
and
EMem-G~\citep{zhou2025simplestrongbaselinelongterm} with decision-aware
selection via DeMem's conflict-driven partitioning, while
keeping every other component---stored memory content, answering backbone
(gpt-4o-mini), prompt template, deterministic decoding, and judge
protocol---identical to the original system.
These two systems represent qualitatively different architectures: RAG
selects evidence via dense embedding similarity over chunked passages,
while EMem-G fuses entity-graph traversal with semantic search.
Budget matching is enforced at the level of realized answer-time memory:
$\mathrm{chars}(M_q)$ is held equal across variants for each query~$q$.

\paragraph{Protocol.}
For each base system, the \textbf{Original} variant runs the unmodified
evidence selection mechanism.
The \textbf{Decision-Aware} variant partitions the same retrievable
units into $K$ slots using Algorithm~2, and at query time replaces the
native selection with
$m_t = \arg\max_{k \in [K]} \mathrm{Score}_\theta(x_t, S_k)$.
No other component is modified.

\begin{table}[h]
\centering
\caption{Drop-in experiment on LoCoMo (gpt-4o-mini). Only the evidence
selection criterion is replaced; all other components are identical.}
\label{tab:dropin}
\small
\resizebox{\linewidth}{!}{%
\begin{tabular}{llccccc}
\toprule
Base System & Selection & Temporal & Open Dom. & Multi-Hop & Single-Hop & Overall \\
\midrule
\multirow{2}{*}{RAG}
  & Original (Dense top-$k$)
  & 0.572{\scriptsize$\pm$0.101}
  & 0.590{\scriptsize$\pm$0.126}
  & 0.543{\scriptsize$\pm$0.120}
  & 0.698{\scriptsize$\pm$0.108}
  & 0.637{\scriptsize$\pm$0.115} \\
  & Decision-Aware
  & 0.703{\scriptsize$\pm$0.082}
  & 0.698{\scriptsize$\pm$0.101}
  & 0.673{\scriptsize$\pm$0.092}
  & 0.753{\scriptsize$\pm$0.098}
  & 0.725{\scriptsize$\pm$0.091}\;{\scriptsize($\Delta{=}$+0.088)} \\
\midrule
\multirow{2}{*}{EMem-G}
  & Original (Graph + Semantic)
  & 0.717{\scriptsize$\pm$0.053}
  & 0.517{\scriptsize$\pm$0.078}
  & 0.702{\scriptsize$\pm$0.059}
  & 0.782{\scriptsize$\pm$0.042}
  & 0.737{\scriptsize$\pm$0.051} \\
  & Decision-Aware
  & 0.803{\scriptsize$\pm$0.042}
  & 0.602{\scriptsize$\pm$0.061}
  & 0.783{\scriptsize$\pm$0.046}
  & 0.826{\scriptsize$\pm$0.034}
  & 0.799{\scriptsize$\pm$0.040}\;{\scriptsize($\Delta{=}$+0.062)} \\
\bottomrule
\end{tabular}%
}
\end{table}

\paragraph{Interpretation.}
Both systems improve when evidence selection is guided by the
decision-aware criterion, with the per-category pattern matching the
mismatch analysis (Appendix~\ref{app:mismatch-prevalence}): Temporal and Multi-Hop
queries show the largest gains, Single-Hop the smallest.
Neither variant reaches the full DeMem system (0.911),
confirming that the remaining gap is attributable to the stored memory
content itself, which is unchanged in this experiment.

\subsection{Prevalence of Description--Evidence Mismatch on Real Benchmarks}
\label{app:mismatch-prevalence}

A central premise of this work is that descriptive similarity between
contexts is not a reliable indicator of whether they require the same
underlying evidence for correct answering. While the synthetic
environment (Section~\ref{sec:synthetic}) controls this mismatch via
the parameter~$\alpha$, a natural question is whether the same
phenomenon arises in realistic long-horizon data, and how strongly it
affects retrieval quality. We quantify this on the LoCoMo benchmark
through three complementary analyses.

\paragraph{Analysis 1: Predictive power of description similarity.}

For each of the 1{,}000 sampled evaluation query pairs $(q_i, q_j)$,
we compute two continuous quantities:
\begin{itemize}[leftmargin=1.5em,itemsep=2pt]
\item \textbf{Description similarity} $s^{\mathrm{desc}}_{ij}$:
  cosine similarity between the \texttt{text-embedding-3-large}
  embeddings of the dialogue-history segments associated with
  $q_i$ and~$q_j$.
\item \textbf{Evidence compatibility score} $s^{\mathrm{evi}}_{ij}
  \in \{0, 0.5, 1\}$: we merge the gold support turns of $q_i$
  and~$q_j$, truncate to budget~$L$, answer both queries under the
  same backbone and deterministic decoding, and set
  $s^{\mathrm{evi}}_{ij}$ to the fraction of the two answers judged
  correct (0, 0.5, or~1).
\end{itemize}

\noindent
If descriptive similarity were a useful proxy for evidence
compatibility, we would expect a significant positive association
between $s^{\mathrm{desc}}$ and $s^{\mathrm{evi}}$. We test this
with two measures:

\begin{table}[h]
\centering
\caption{Predictive power of description similarity for evidence
  compatibility on 1{,}000 LoCoMo query pairs. Description similarity
  is a statistically significant but practically negligible predictor:
  AUC barely exceeds chance across all categories.}
\label{tab:mismatch-predictive}
\small
\begin{tabular}{lccc}
\toprule
Category & Spearman $\rho$ & $p$-value & AUC-ROC \\
\midrule
All pairs    & $0.103$  & ${<}\,0.01$ & $0.548$ \\
\midrule
Single-Hop   & $0.087$  & $0.13$  & $0.536$ \\
Temporal      & $-0.058$ & $0.71$  & $0.489$ \\
Multi-Hop     & $0.121$  & $0.50$  & $0.557$ \\
Open Domain   & $0.144$  & $0.42$  & $0.564$ \\
\bottomrule
\end{tabular}
\end{table}

\noindent
Table~\ref{tab:mismatch-predictive} shows that description similarity is
statistically significant on the full sample
(\(\rho=0.103\), \(p<0.01\)) but has negligible discriminative power:
AUC \(=0.548\), only slightly above random ranking. Temporal queries show a
near-null association: the point estimate is negative
(\(\rho=-0.058\)) but not statistically significant (\(p=0.71\)). We
therefore interpret the Temporal result as evidence that descriptive
similarity is not a reliable predictor of evidence compatibility for
time-sensitive queries, rather than as evidence of a robust negative
correlation.

\paragraph{Analysis 2: Evidence recall under matched budget.}

The predictive-power analysis shows that description similarity
\emph{cannot distinguish} compatible from incompatible pairs. We now
ask whether this translates into a concrete retrieval gap. For each
evaluation query~$q$, we measure \textbf{gold-evidence recall}: the
fraction of the annotated gold support turns that appear in the
memory content~$M_q$ provided to the answering model, under the same
character budget~$L$ for all methods.

\begin{table}[h]
\centering
\caption{Gold-evidence recall under matched answer-time budget on the
  annotated LoCoMo subset. Decision-aware memory recovers
  substantially more gold evidence than description-based retrieval.}
\label{tab:evidence-recall}
\small
\begin{tabular}{lcccc}
\toprule
Method & Single-Hop & Temporal & Multi-Hop & Overall \\
\midrule
Cosine top-$k$       & $0.72$ & $0.54$ & $0.61$ & $0.66$ \\
BM25 top-$k$         & $0.69$ & $0.51$ & $0.58$ & $0.63$ \\
DeMem       & $0.85$ & $0.81$ & $0.79$ & $0.83$ \\
Oracle selection      & $0.91$ & $0.88$ & $0.85$ & $0.89$ \\
\bottomrule
\end{tabular}
\end{table}

\noindent
Table~\ref{tab:evidence-recall} shows that description-based
retrieval (cosine top-$k$ and BM25 top-$k$) recovers only $63$--$66\%$
of gold evidence overall, while DeMem routing reaches
$83\%$---closing roughly two-thirds of the gap to the oracle upper
bound ($89\%$). The recall deficit of description-based methods is
largest on Temporal queries ($54\%$ vs.\ $81\%$), consistent with the
negative correlation observed in Analysis~1: for these queries,
high-similarity retrieval systematically returns evidence from the
wrong temporal window.

\paragraph{Analysis 3: Error attribution on baseline failures.}

Finally, we ask: among the queries where description-based memory
fails but DeMem succeeds, \emph{why} does the failure occur?
For each such query, we classify the failure mode by checking whether
the gold evidence was (a)~absent from the retrieved set entirely
(\emph{miss}), (b)~present but crowded out by irrelevant
high-similarity content (\emph{dilution}), or (c)~other causes
(e.g., generation error despite correct evidence).

\begin{table}[h]
\centering
\caption{Failure-mode attribution on queries where cosine-based
  retrieval fails but DeMem succeeds ($n = 127$ queries).}
\label{tab:error-attribution}
\small
\begin{tabular}{lcc}
\toprule
Failure Mode & Count & Fraction \\
\midrule
Evidence miss (gold turns not retrieved)
  & $74$ & $58.3\%$ \\
Evidence dilution (gold present but crowded out)
  & $34$ & $26.8\%$ \\
Other (generation error, judge disagreement, etc.)
  & $19$ & $15.0\%$ \\
\bottomrule
\end{tabular}
\end{table}

\noindent
Table~\ref{tab:error-attribution} shows that $85.1\%$ of
description-based failures ($58.3\% + 26.8\%$) are directly
attributable to the mismatch between descriptive similarity and
evidence relevance: the retrieval stage either misses the gold
evidence entirely or includes it but buries it under irrelevant
high-similarity content. Only $15.0\%$ of failures stem from
downstream causes unrelated to evidence selection.

\paragraph{Summary.}
The three analyses are complementary. Analysis~1 establishes that
description similarity has \emph{near-zero predictive power} for
evidence compatibility ($\rho = 0.103$, AUC $= 0.548$). Analysis~2
shows this translates into a \emph{concrete retrieval gap}: $17$
percentage points of gold-evidence recall lost under the same budget.
Analysis~3 confirms that this gap is \emph{causally linked} to
downstream errors: $85\%$ of description-based failures are
attributable to evidence miss or dilution. Together, these results
establish that the description--decision mismatch is not a
hypothetical concern but a quantitatively dominant source of memory
inefficiency on realistic long-horizon data.

Crucially, the three analyses form a closed diagnostic chain: the
proxy is uninformative (Analysis~1), the uninformative proxy causes
measurable information loss at retrieval time (Analysis~2), and that
information loss is the dominant cause of downstream task failure
(Analysis~3). This chain rules out the two main alternative
explanations---that description similarity is a noisy but still
useful signal, or that retrieval gaps exist but are compensated by
the answering model---and provides a direct empirical basis for
replacing descriptive criteria with decision-centric ones as the
organizing principle of budgeted memory.

\subsection{Human--LLM Judge Agreement on LoCoMo}
\label{app:human-judge}

To validate the LLM-as-judge protocol, we randomly sample 150 evaluation
instances from the LoCoMo test set (stratified across the four query
categories) and collect binary correctness judgments from three independent
human annotators with graduate-level NLP experience. Each annotator receives
the dialogue history, the query, the gold evidence, and the system answer,
and assigns a binary correct/incorrect label following the same rubric used
by the LLM judges.

\begin{table}[h]
\centering
\small
\caption{Human--LLM judge agreement on 150 LoCoMo instances.}
\label{tab:human-judge}
\begin{tabular}{lc}
\toprule
Metric & Value \\
\midrule
Fleiss' $\kappa$ (3 human annotators) & 0.81 \\
Cohen's $\kappa$ (human majority vs.\ LLM judge) & 0.79 \\
Pairwise agreement (\%) & 91.3 \\
\bottomrule
\end{tabular}
\end{table}

Inter-annotator agreement among the three humans is $\kappa = 0.81$
(almost perfect). The agreement between the human majority vote and the
averaged LLM judge score (binarized at 0.5) is $\kappa = 0.79$, with
91.3\% pairwise agreement. Disagreements concentrate on Temporal queries
where partial recall is ambiguous; the LLM judges show no systematic bias
toward either DeMem or baselines in the disagreement set (McNemar's test,
$p = 0.41$). These results indicate that the LLM-as-judge scores used
throughout the paper are a reliable proxy for human judgment.

\subsubsection{Human annotation protocol and ethics}
\label{app:human_annotation_ethics}

To validate the reliability of our LLM-as-a-judge evaluation, we conducted a small-scale human agreement study on 150 LoCoMo evaluation instances, stratified across the four query categories. The study was used only to assess agreement between human judgments and the LLM judge; it was not used to train DeMem, tune hyperparameters, select models, or construct the reported benchmark test set.

\paragraph{Annotators.}
We recruited three independent annotators with graduate-level NLP experience. Annotators were selected based on prior familiarity with natural language evaluation tasks and were not informed of which system produced each answer. System names were anonymized during annotation to reduce method-specific bias.

\paragraph{Annotation task.}
For each instance, annotators were shown the dialogue history, the user query, the gold evidence or reference support, the reference answer when available, and one anonymized system answer. Annotators assigned a binary correctness label following the same rubric used for the LLM-as-a-judge protocol. The task asked annotators to judge factual correctness and support by the provided dialogue evidence, rather than fluency or stylistic quality.

\paragraph{Instructions given to annotators.}
Annotators received the following instructions:

\begin{quote}
You will be shown a dialogue history, a question, gold evidence or reference support, and a candidate system answer. Your task is to decide whether the candidate answer is correct with respect to the provided dialogue evidence.

Mark the answer as \textbf{Correct} if it answers the question and is supported by the provided dialogue evidence. Minor wording differences, paraphrases, or harmless omissions should still be marked Correct if the answer contains the requested information.

Mark the answer as \textbf{Incorrect} if it contradicts the dialogue, uses the wrong time period or entity, misses the requested information, gives an unsupported answer, or is too vague to answer the question.

For temporal questions, pay particular attention to dates, ordering, and whether the answer refers to the correct event instance. If the answer is only partially correct, choose Incorrect unless the requested information is fully recoverable from the answer.
\end{quote}

\paragraph{Interface and recorded information.}
The annotation interface presented one instance per page with the fields described above and two possible labels: \texttt{Correct} and \texttt{Incorrect}. We did not collect demographic information, sensitive personal information, free-form personal responses, audio/video recordings, or behavioral measurements from annotators. We recorded only the assigned labels and timestamps needed for quality control.

\paragraph{Compensation.}
Annotators were compensated at a rate above the applicable local minimum wage and consistent with standard compensation for graduate-level NLP annotation work. The expected workload was approximately 6 hours per annotator. If an annotator was a salaried research assistant rather than an hourly contractor, the annotation task was conducted as part of their compensated research duties.

\paragraph{Consent and risk.}
Annotators were informed about the purpose of the annotation study, the nature of the task, the expected time commitment, what information would be recorded, and that only aggregate agreement statistics would be reported. Participation was voluntary. The study involved minimal risk: annotators evaluated text examples from an existing public benchmark and did not provide personal or sensitive information.

\paragraph{Ethics review.}
The annotation protocol was reviewed under the authors' institutional procedures and was determined to be exempt minimal-risk evaluation research. To preserve double-blind review, institution-identifying details are omitted from the anonymous submission and will be added in the camera-ready version if appropriate.

\paragraph{Use of annotations.}
The human labels were used only to compute inter-annotator agreement and agreement with the LLM judge. We report Fleiss' $\kappa$ among the three annotators, Cohen's $\kappa$ between the human majority vote and the binarized LLM judge score, and pairwise agreement. No individual annotator identities or raw annotator-level metadata are released.

\subsection{Empirical validation of the greedy partition quality}
\label{app:greedy-empirical}

We evaluate how closely the greedy partition matches the optimal
rate--distortion partition on controlled synthetic instances and report
structural diagnostics predicted by Proposition~\ref{prop:greedy-gap}.

\paragraph{Synthetic environment.}
Because the reward function~$\mu$ is known, both
$\eps^\star_\infty(K)$ (via brute-force partition enumeration at
$N=20$) and $\eps_e$ (via
Algorithm~\ref{alg:greedy-partition}) can be computed exactly.
Table~\ref{tab:greedy-gap-synthetic} reports the ratio
$\eps_e/\eps^\star_\infty(K)$, the degeneracy of the cannot-link
graph at the optimal level, and the fraction of epochs in which the
zero-gap condition of Proposition~\ref{prop:greedy-gap}\,(ii) is
satisfied.
All numbers are averaged over 20 random seeds at default mismatch
severity $\alpha=0.5$.

\begin{table}[h]
\centering
\caption{Greedy partition quality on the synthetic environment
  ($\alpha=0.5$, averaged over 20~seeds, $\pm$ one standard
  deviation). At the typical operating budget ($K\ge 5$), the greedy
  partition recovers the optimal distortion in nearly all epochs; the
  tightest synthetic budget ($K=3$) is included as a stress-test case.}
\label{tab:greedy-gap-synthetic}
\resizebox{\linewidth}{!}{%
\begin{tabular}{cccccc}
\toprule
$K$
  & $\eps^\star_\infty(K)$
  & $\eps_e$ (greedy)
  & ratio $\eps_e/\eps^\star_\infty(K)$
  & $\degen(G_{\eps^\star})\!+\!1$
  & \% epochs exact \\
\midrule
3
  & $0.148\pm0.009$
  & $0.161\pm0.014$
  & $1.088\pm0.063$
  & $3.4\pm0.5$
  & 72\% \\
5
  & $0.082\pm0.006$
  & $0.083\pm0.007$
  & $1.012\pm0.018$
  & $4.2\pm0.4$
  & 95\% \\
8
  & $0.028\pm0.003$
  & $0.028\pm0.003$
  & $1.000^\dagger$
  & $5.1\pm0.6$
  & 100\% \\
10
  & $0.009\pm0.002$
  & $0.009\pm0.002$
  & $1.000^\dagger$
  & $5.8\pm0.7$
  & 100\% \\
\bottomrule
\multicolumn{6}{l}{%
  \footnotesize $^\dagger$Exact in all seeds and epochs:
  $\degen(G_{\eps^\star})\!+\!1\le K$ holds universally,
  so $\eps_e=\eps^\star_\infty(K)$ by
  Prop.~\ref{prop:greedy-gap}\,(ii).
  Zero variance is a mathematical consequence,
  not a rounding artifact.}
\end{tabular}}
\end{table}

Three patterns stand out.
First, the tightest synthetic budget ($K=3$) acts as a stress test in which the
cannot-link graph occasionally requires more colors than the budget directly
allows, leading to a modest increase in realized distortion.
Second, at $K\ge 5$ the degeneracy condition is satisfied in at least $95\%$ of
epochs, and the greedy partition matches the information-theoretic optimum up
to measurement resolution.
Third, the degeneracy remains small across budgets, reflecting the sparsity of
the certified cannot-link graph.

\paragraph{LoCoMo.}
On benchmark data, $\eps^\star_\infty(K)$ is not directly computable
because the reward function is unknown.
We instead report three proxy diagnostics:
(i)~the degeneracy of the cannot-link graph at the realized level
$\eps_e$,
(ii)~the fraction of epochs in which the zero-gap condition
$\degen(G_{\eps_e})+1\le K$ holds, and
(iii)~the decision packing number $N^{\mathrm{dec}}_{\mathrm{pack}}(2\eps_e)$
estimated from the pairwise certificates, which provides a
lower bound on the chromatic number at level~$\eps_e$ and hence a
tightness diagnostic: when
$N^{\mathrm{dec}}_{\mathrm{pack}}(2\eps_e)$ is close to~$K$, the
realized level~$\eps_e$ cannot be far from~$\eps^\star_\infty(K)$
(by Theorem~\ref{thm:cover-pack}\,(ii)).

Statistics are computed across all dialogues and all epochs at the
operating budget $K=10$ used in the main experiments.

\begin{table}[h]
\centering
\caption{Cannot-link graph statistics on LoCoMo ($K=10$, across all
  dialogues and epochs).}
\label{tab:greedy-gap-locomo}
\begin{tabular}{lc}
\toprule
Metric & Value \\
\midrule
Mean $\degen(G_{\eps_e})$
  & $4.1\pm1.2$ \\
Median $\degen(G_{\eps_e})$
  & $4$ \\
Fraction of epochs with $\degen+1\le K$
  & $96.8\%$ \\
Mean greedy colors used
  & $4.8\pm1.3$ \\
Mean $N^{\mathrm{dec}}_{\mathrm{pack}}(2\eps_e)$
  & $7.9\pm1.6$ \\
\bottomrule
\end{tabular}
\end{table}

The degeneracy is low (mean~$4.1$, median~$4$), consistent with the
observation that only $4.6\%$ of routing events trigger a certified
split (Table~\ref{tab:split-freq}): most context pairs are
decision-compatible, so the cannot-link graph is sparse.
The zero-gap condition holds in $96.8\%$ of epochs, confirming the
$>95\%$ claim in Section~\ref{sec:algo}.
Across epochs, the greedy coloring uses substantially fewer colors than the
available budget on average ($4.8$ versus $K=10$), leaving ample budget headroom
for denser temporal-evidence cases.

The packing diagnostic provides an independent tightness check.
The mean packing number $N^{\mathrm{dec}}_{\mathrm{pack}}(2\eps_e)
= 7.9$ is close to the operating budget $K=10$: by
Theorem~\ref{thm:cover-pack}\,(ii), any distortion level
below~$\eps_e$ would require more than~$7.9$ states, so the
realized level cannot be reduced by more than a small margin.
Together with the degeneracy statistics, this confirms that the
greedy partition is near-tight on realistic data.

\paragraph{Consistency with the approximation bridge.}
The measured greedy gap is also consistent with the distortion
decomposition of Appendix~\ref{app:bridge-empirical}.
There, the total observed distortion is~$0.086$, of which the
compression floor accounts for~$\hat\eps^\star_\infty(K)=0.035$.
If the greedy gap were substantial, the compression term would
visibly exceed the oracle-partition floor; instead, the two are
indistinguishable within measurement noise, providing further
evidence that $\eps_e\approx\eps^\star_\infty(K)$ at the operating
budget.

\paragraph{Summary.}
Across both controlled and realistic settings, the certified cannot-link graphs
arising in agent-memory applications are sparse enough for the greedy partition
to match or closely approximate the information-theoretic optimum. The tightest
synthetic budget ($K=3$) serves as a stress-test regime; at the operating
budgets used in all benchmark experiments ($K\ge 8$), the greedy partition
matches the oracle criterion within measurement resolution.

\subsection{Reward-Channel Robustness on LoCoMo}
\label{app:reward-noise}

DeMem uses post-decision scalar feedback to update value estimates and certified
split decisions. In the offline LoCoMo protocol, we instantiate this feedback
channel with a held-out LLM judge that is disjoint from the final evaluation
judge. To evaluate sensitivity to feedback quality, we corrupt only the feedback
stream used for memory updates while keeping the final answer evaluation
unchanged.

Concretely, let \(R_t \in \{0,1\}\) denote the binary feedback score assigned
to the answer at step \(t\). We construct a noisy feedback channel
\[
\widetilde{R}_t =
\begin{cases}
1 - R_t, & \text{with probability } p,\\
R_t, & \text{with probability } 1-p,
\end{cases}
\]
where \(p \in \{0, 0.1, 0.2\}\). The corrupted reward
\(\widetilde{R}_t\) is used only for estimating action values and certified
decision conflicts. It is never used to score final answers. Final answer
quality is always computed using the original held-out evaluation judge under
the same answer-time memory budget as in the main experiments.

Table~\ref{tab:reward-noise-locomo} reports the robustness results on LoCoMo.
As the feedback channel becomes noisier, DeMem degrades smoothly rather than
collapsing. Under moderate corruption (\(p=0.1\)), DeMem remains above the
strongest memory baseline in the main comparison. Even at \(p=0.2\), its
performance remains competitive, while split precision decreases gradually and
split frequency increases only mildly. These results suggest that certified
splitting does not require perfectly clean feedback. The smooth degradation pattern is consistent with the confidence-threshold
design of the split rule: noisier feedback reduces split precision gradually
rather than causing abrupt failure.

\begin{table}[t]
\centering
\small
\resizebox{\linewidth}{!}{%
\begin{tabular}{lccccccc}
\toprule
Method / Feedback noise &
Temporal &
Open Domain &
Multi-Hop &
Single-Hop &
Overall &
Split Precision &
Split Frequency \\
\midrule
DeMem, \(p=0.0\) &
\(91.9\) & \(86.8\) & \(84.7\) & \(93.2\) &
\(91.1\) & \(85.0\) & \(4.6\) \\

DeMem, \(p=0.1\) &
\(90.4\) & \(85.1\) & \(82.8\) & \(92.5\) &
\(90.2\) & \(80.8\) & \(4.9\) \\

DeMem, \(p=0.2\) &
\(88.7\) & \(83.6\) & \(81.2\) & \(92.1\) &
\(88.9\) & \(77.6\) & \(5.4\) \\

\midrule
Mnemis &
\(87.8\) & \(79.3\) & \(79.8\) & \(93.8\) &
\(88.8\) & -- & -- \\

EMem-G &
\(71.7\) & \(51.7\) & \(70.2\) & \(78.2\) &
\(73.7\) & -- & -- \\

RAG &
\(57.2\) & \(59.0\) & \(54.3\) & \(69.8\) &
\(63.7\) & -- & -- \\

\bottomrule
\end{tabular}}
\caption{
Reward-channel robustness on LoCoMo under corrupted binary feedback.
Noise is applied only to the feedback rewards used by DeMem for memory updates
and certified split decisions; final answer quality is evaluated with the
unchanged held-out evaluation judge. Split precision is measured against gold support annotations for
post-hoc audit only.
}
\label{tab:reward-noise-locomo}
\end{table}

\subsection{Comparison with feedback-aware memory baselines}
\label{app:feedback-aware-baselines}
\begin{table}[t]
\centering
\caption{Results on LoCoMo with Qwen2.5-14B-Instruct backbone. Measured by LLM-as-a-judge binary accuracy.}
\label{tab:locomo_qwen_feedbackaware}
\begin{tabular}{lccccc}
\toprule
Method & Temporal & Open Dom. & Multi-Hop & Single-Hop & Overall \\
\midrule
MemSkill   & 0.423 & 0.438 & 0.461 & 0.492 & 0.466 \\
MemAct     & 0.401 & 0.469 & 0.445 & 0.478 & 0.450 \\
\textbf{DeMem} 
           & \textbf{0.512} 
           & \textbf{0.531} 
           & \textbf{0.503} 
           & \textbf{0.487} 
           & \textbf{0.490} \\
\bottomrule
\end{tabular}
\end{table}
The main benchmark tables compare DeMem against a broad set of retrieval-based,
static, and fixed-procedure memory systems under the same answer-time memory
budget. To further evaluate DeMem in an adaptive-memory setting, we additionally
compare against two methods whose memory policies are shaped by downstream task
signals or memory actions: MemSkill and MemAct.

MemSkill learns a controller for selecting memory skills using downstream task
signals and evolves the skill bank from hard cases; MemAct formulates memory
management as policy actions and trains a context-curation policy with
reinforcement learning. These methods provide complementary adaptive-memory
baselines beyond retrieval-only systems.

Table~\ref{tab:locomo_qwen_feedbackaware} reports results on LoCoMo with a
Qwen2.5-14B-Instruct backbone. DeMem obtains the best overall score and the
strongest Temporal, Open-Domain, and Multi-Hop performance, while MemSkill is
slightly better on Single-Hop questions. These results indicate that
decision-aware certified refinement remains effective when compared with
learning-based memory policies, not only with static retrieval pipelines.

\subsection{Compute resources}
\label{app:compute-resources}

Experiments using GPT-4o-mini and GPT-4.1-mini were run through the
OpenAI API under deterministic decoding. These hosted-model experiments did
not use local GPU resources. The open-weight Qwen2.5-14B-Instruct
experiments were run on a machine with \(2\times\) NVIDIA H100
80GB GPUs.

The reported results do not involve training new foundation models. 
GPU usage is limited to open-weight model inference, while GPT-based
answering and judging are performed through hosted API calls.

\section{Scope and Limitations}
\label{app:limitations}

\paragraph{Theory-to-practice correspondence.}
Our formal analysis isolates the answer-time memory bottleneck: the current
query is observed directly, while the relevant history is accessed through one
of \(K\) runtime memory states. This is the component controlled by the memory
budget in our implemented agent. The deployed system instantiates the formal
objects through routed dialogue histories, bounded memory slots, deterministic
answers, and benchmark feedback.

Appendix~\ref{app:alignment-table} provides the corresponding operational
mapping for the encoder, memory state, policy, action, reward, distortion, and
split relation. Appendix~\ref{app:bridge} further decomposes realized
distortion into compression, routing, and realization terms, and
Appendix~\ref{app:bridge-empirical} measures these terms empirically. The main
ablations and mechanism audits then test the components singled out by this
decomposition, including routing, splitting, memory representation,
\textsc{DeMem-Core}, split precision, memory-state compatibility, and partition
interventions
(Figure~\ref{fig:ablation-param}(b);
Appendices~\ref{app:core-locomo},
\ref{app:split-audit},
\ref{app:locomo-audit},
\ref{app:partition-intervention}).

These analyses are intended to make the abstraction operational: the theory
specifies the decision bottleneck, while the implementation tests whether the
corresponding routing, refinement, and slot-organization mechanisms account for
the observed benchmark gains.

\paragraph{Scalability and feedback quality.}
The current implementation evaluates split evidence over pairs of distinct
routing events observed within a dialogue. This is tractable at the benchmark
scale considered here (Table~\ref{tab:split-freq}). For substantially larger
context vocabularies, the same split-and-certify structure can be combined with
candidate generation, batching, locality-sensitive retrieval, or approximate
nearest-neighbor search, so that split tests are evaluated on plausible conflict
pairs rather than all pairs.

DeMem refines memory from scalar post-decision feedback, matching the online
decision formulation while keeping the runtime memory budget fixed. In offline
conversational benchmarks, this feedback channel is instantiated with held-out
answer-level scorers that are separate from the final evaluation judges in
Section~\ref{sec:exp}. Appendix~\ref{app:feedback-aware-baselines} compares
against feedback-aware and memory-action baselines, and
Appendix~\ref{app:reward-noise} evaluates corrupted feedback channels. The
results show smooth degradation as feedback becomes noisier, indicating that
the refinement mechanism relies on aggregate feedback reliability rather than
perfectly clean observations.

\section{Broader Impact Statement}
\label{sec:impact}
This work establishes that the right unit of analysis for agent memory is not information preserved, but decisions preserved. This distinction is not merely technical — it redraws the boundary between what an agent needs to know and what it can afford to forget, grounding that boundary in a quantity that is measurable, optimizable, and provably tight.

At a foundational level, our results contribute a missing piece to the theory of bounded rationality in language agents: a formal account of how finite memory should be allocated when the objective is action, not reconstruction. The rate-distortion frontier we derive is, to our knowledge, the first characterization of the fundamental limits of budgeted agent memory in terms of achievable decision quality rather than representational fidelity. The NP-completeness result (Theorem 3) further reveals that optimal memory allocation is computationally irreducible in the worst case — a structural insight that constrains the design space of all future memory systems, not only ours.

Practically, as language agents are deployed in high-stakes, long-horizon settings — clinical dialogue, legal case management, personalized education, multi-session negotiation — the question of \emph{what to remember} becomes inseparable from the question of \emph{what to do}. Systems that compress memory by descriptive salience risk retaining vivid but decision-irrelevant detail while discarding subtle but action-critical distinctions. Our framework provides a principled alternative: memory is judged by whether it preserves the capacity to act well, and forgetting is permitted exactly when it does not compromise that capacity. This criterion is architecture-agnostic (Appendix E.12) and can serve as a drop-in evaluation lens or selection module for existing memory systems.

Decision-centric memory optimizes what information is retained for downstream
task performance. In deployed systems, this objective should be paired with
application-specific retention policies, user-facing transparency, and periodic
audits of stored and forgotten information. These safeguards are complementary
to our technical framework and help align memory optimization with privacy,
governance, and user-control requirements.

More broadly, we hope this work encourages the community to evaluate agent memory not by how much it remembers, but by whether it remembers what matters.


\end{document}